\documentclass{article} 
\usepackage{iclr2026_conference,times}

\usepackage{bbm, amsthm, amsmath,amsfonts,amssymb}
\newtheorem{theorem}{Theorem}
\newtheorem{lemma}{Lemma}

\newtheorem{remark}{Remark}
\newtheorem{assumption}{Assumption}

\newcommand{\calA}{{\mathcal{A}}}

\newcommand{\calC}{{\mathcal{C}}}

\newcommand{\calM}{{\mathcal{M}}}
\newcommand{\calN}{{\mathcal{N}}}
\newcommand{\calO}{{\mathcal{O}}}
\newcommand{\calP}{{\mathcal{P}}}
\newcommand{\calQ}{{\mathcal{Q}}}

\newcommand{\calS}{{\mathcal{S}}}

\newcommand{\calV}{{\mathcal{V}}}

\newcommand{\bbE}{\mathbb{E}}

\newcommand{\bbP}{\mathbb{P}}

\newcommand{\bbR}{\mathbb{R}}

\newcommand{\bbZ}{\mathbb{Z}}

\newcommand\AlgName{\text{Primal-Dual Policy Optimization for Linear CMDPs with Adversarial Losses}}

\newcommand\Regret{\mathrm{Regret}(K)}
\newcommand\Violation{\mathrm{Violation}(K)}

\newcommand{\piunif}{\pi_{\mathrm{unif}}}

\newcommand{\interior}{\operatorname{int}}
\newcommand{\floor}[1]{\left\lfloor #1 \right\rfloor}
\newcommand{\ceil}[1]{\left\lceil #1 \right\rceil}
\DeclareMathOperator*{\argmax}{arg\,max}
\DeclareMathOperator*{\argmin}{arg\,min}

\usepackage{xcolor}
\definecolor{darkblue}{RGB}{0,0,139}
\usepackage[colorlinks=true, citecolor=darkblue, linkcolor=darkblue, urlcolor=darkblue]{hyperref}
\usepackage{url}
\usepackage{algorithm}
\usepackage[pdftex]{graphicx}
\usepackage{mathtools,flexisym, amsthm, xcolor}
\usepackage{algpseudocode}
\usepackage{cleveref}
\usepackage{subcaption}

\title{Primal-Dual Policy Optimization for \\ Linear CMDPs with Adversarial Losses}

\iclrfinalcopy

\author{Kihyun Yu,\ \ \  Seoungbin Bae \\
Department of Industrial and Systems Engineering\\
KAIST \\
\texttt{\{khyu99,sbbae31\}@kaist.ac.kr}
\And
Dabeen Lee
\thanks{Additional Affiliations: Research Institute of Mathematics, Seoul National University; Interdisciplinary Program in Artificial Intelligence, Seoul National University; Korea Institute for Advanced Study}  
\\
Department of Mathematical Sciences\\
Seoul National University\\
\texttt{dabeenl@snu.ac.kr}
}

\begin{document}

\maketitle

\begin{abstract}
Existing work on linear constrained Markov decision processes (CMDPs) has primarily focused on stochastic settings, where the losses and costs are either fixed or drawn from fixed distributions. However, such formulations are inherently vulnerable to adversarially changing environments. To overcome this limitation, we propose a primal-dual policy optimization algorithm for online finite-horizon {adversarial} linear CMDPs, where the losses are adversarially chosen under full-information feedback and the costs are stochastic under bandit feedback. Our algorithm is the \emph{first} to achieve sublinear regret and constraint violation bounds in this setting, both bounded by $\widetilde{\mathcal{O}}(K^{3/4})$, where $K$ denotes the number of episodes. The algorithm introduces and runs with a new class of policies, which we call weighted LogSumExp softmax policies, designed to adapt to adversarially chosen loss functions. Our main result stems from the following key contributions: (i) a new covering number argument for the weighted LogSumExp softmax policies, and (ii) two novel algorithmic components---periodic policy mixing and a regularized dual update---which allow us to effectively control both the covering number and the dual variable. We also report numerical results that validate our theoretical findings on the performance of the algorithm.
\end{abstract}

\section{Introduction}\label{sec:intro}
Safe reinforcement learning (RL) studies sequential decision-making under safety constraints through interaction with an unknown environment. Many real-world applications have been explored under the safe RL framework, including autonomous driving~\citep{driving-isele2018safe}, robotics~\citep{robotics-achiam2017constrained}, and healthcare~\citep{healthcare-coronato2020reinforcement}. A common modeling framework for safe RL is the online constrained Markov decision process (CMDP) formulation, where the agent seeks a policy that minimizes (or maximizes) cumulative expected loss (or reward), while ensuring that the expected cumulative cost does not exceed a given budget~\citep{altman2021constrained}.

To better capture realistic scenarios, it is often necessary to model adversarial environments, where components may vary arbitrarily over time. For instance, in autonomous driving, the loss may reflect safety risks such as sudden braking, but it can also increase drastically due to unexpected traffic or hazardous weather. In service robotics, the loss may correspond to task failures or user dissatisfaction, which fluctuate with human preferences or rapidly changing tasks. In such applications, assuming a fixed loss signal is overly restrictive. Therefore, to model these scenarios, it is essential to consider CMDPs under adversarial settings.

Online adversarial CMDPs assume that the loss or cost functions can change arbitrarily across episodes, rather than being drawn from fixed stochastic distributions. Recently, adversarial CMDPs have been investigated under the tabular setting~\citep{qiu2020upper, stradi2024online, stradi2025learning, stradi2025policy, zhu2025optimistic}. Although these works achieved sublinear regret and constraint violation bounds, they focused only on environments where the state space is finite and relatively small. As a result, their algorithmic guarantees may not extend to settings with large state spaces. Such algorithms are often unsatisfactory in real-world applications, where the number of states is typically extremely large.

To capture settings with a large state space, safe RL with linear function approximation has been studied. \citet{ding2021provably} proposed a primal-dual policy optimization algorithm for linear mixture CMDPs, where the dynamics are expressed as a mixture of a finite set of basis kernels. For linear CMDPs, assuming linear structure in the loss and cost functions as well as in the dynamics, \citet{ghosh2022provably, ghosh2024towards} designed a primal-dual-type optimistic value iteration algorithm with a softmax policy. For the same setting, \citet{kitamura2025provably} developed an algorithm achieving zero constraint violation under the assumption of a known safe policy. Although these algorithms can handle large state spaces, they considered only stochastic settings, where the loss and cost functions are fixed or drawn from underlying distributions. That said, they fail to capture the aforementioned applications, where taking into account adversarial environments is essential when modeling safe RL algorithms.

To overcome these limitations, this paper proposes an algorithm for adversarial linear CMDPs. To handle adversarial losses with constraints, as online primal-dual mirror descent type algorithms become a standard choice, it is natural to consider primal-dual policy optimization for our setting~\citep{chen2021primal, ding2021provably}. However, when applying primal-dual policy optimization to adversarial linear CMDPs, additional challenges arise—most notably in bounding the covering number of the value function class~\citep{jin2020provably}.

To elaborate on this challenge, primal-dual policy optimization induces a more intricate policy class, necessitating a new covering number argument. In particular, slightly perturbing the primal variable before optimizing it---namely, the policy in our case---is commonly used in various settings~\citep{wei2020online, qiu2020upper, ding2021provably, stradi2025policy}. The purpose of this step is to derive a compact dual variable bound, which is essential for regret and violation analyses. However, this step breaks the recursive structure of policy optimization, so the resulting policy cannot be represented as a typical softmax policy. As a consequence, the covering number argument becomes non-trivial. In other words, while policy mixing is simple and common, it poses a critical issue for covering number arguments in linear CMDPs. Despite these challenges, we aim to answer the following question:
\begin{center}
    \emph{Can we design a primal-dual policy optimization algorithm for adversarial linear CMDPs that ensures sublinear regret and violation bounds?}
\end{center}

\paragraph{Main Contributions} We answer the question affirmatively with \Cref{alg:main-linear}, designed for finite-horizon adversarial linear CMDPs, where the losses are adversarially chosen under full-information feedback, and the costs are stochastic under bandit feedback. We summarize our main contributions.
\begin{itemize}   
    \item We present a primal-dual policy optimization algorithm (\Cref{alg:main-linear}) for adversarial linear CMDPs that achieves regret and constraint violation upper bounds of $\widetilde\calO(K^{3/4})$, where $K$ is the number of episodes. Our algorithm is the \emph{first} algorithm that achieves sublinear regret and violation in the adversarial linear CMDP setting. Moreover, the algorithm develops a new class of policies, which we refer to as weighted LogSumExp softmax policies, designed to adapt to adversarially chosen loss functions.  
    
    \item We establish a covering number argument for the novel class of {weighted LogSumExp softmax policies}, induced by primal-dual policy optimization algorithms. The main technical difficulty arises from the fact that the weight parameters across policies may differ, preventing direct application of standard properties of the LogSumExp function. Nevertheless, our analysis shows that the covering number under this policy class is bounded by $\widetilde\calO(n^2d^2)$, where $n$ is the maximum number of mixing steps, and $d$ is the feature dimension of the linear CMDP.
    
    \item Another challenge in designing a sublinear algorithm for adversarial linear CMDPs lies in the need to simultaneously control both the covering number and the dual variable. To address this, our algorithm incorporates two novel components: (i) {periodic policy mixing} and (ii) regularized dual updates. Since the covering number grows with the number of mixing steps, the purpose of periodic policy mixing is to regulate the frequency of mixing steps by applying it once in every specified mixing period, rather than in every episode. To incorporate periodic policy mixing, our dual update has to introduce an additional regularization term in order to obtain a compact bound on the dual variable. Together, these algorithmic components allow us to effectively control both the covering number and the dual variable, establishing a sublinear algorithm.
\end{itemize}

A more detailed review of related work is deferred to the appendix.

\section{Problem Setting}

\paragraph{Finite-Horizon Adversarial CMDP} A finite-horizon adversarial CMDP is defined by the tuple $\calM = (H,\calS, \calA, \{\bbP_h\}_{h=1}^H$, $\{f^k\}_{k=1}^K, \{g^k\}_{k=1}^K, s_1,b)$, where $H$ is the finite horizon, $\calS$ is the finite state space\footnote{For simplicity, we assume that the state space is finite. However, the state space may be arbitrarily large, as discussed in \citet{cassel2024near}, since the computational complexity of our algorithm---as well as the regret and constraint violation---does not scale with $|\calS|$, which will be presented in the following section.}, and $\calA$ is the finite action space.
$\{\bbP_h\}_{h=1}^H$ is a collection of transition kernels for each step $h\in[H]$, where $\bbP_h(s'\mid s,a)$ denotes the probability of transitioning from state $s$ to state $s'$ when action $a$ is taken at step $h$. $\{f^k\}_{k=1}^K$ and $\{g^k\}_{k=1}^K$ are the sequences of loss and cost functions over episodes $k\in[K]$, where $f^k = \{f_h^k\}_{h=1}^H$ and $g^k = \{g_h^k\}_{h=1}^H$ satisfy $f_h^k, g_h^k: \calS \times \calA \to [0,1]$. $s_1\in\calS$ is the fixed initial state, and $b\in [0,H]$ is the cost budget.

 We consider a setting where the loss functions are adversarial, while the cost functions are stochastic. Specifically, at the beginning of each episode $k\in [K]$, an adversary chooses the loss function $f^k$, which can be selected arbitrarily (i.e., not drawn from a distribution). In contrast, the cost function $g^k$ is sampled i.i.d. from a fixed distribution $G$, satisfying $\bbE[g_h^k(s,a)\mid s,a] = g_h(s,a)$.

 The interaction between the agent and the environment proceeds as follows. At the beginning of each episode $k\in [K]$, $f^k$ is adversarially chosen, which is not revealed to the agent. Next, the agent selects a collection of policies $\{\pi_h\}_{h=1}^H$, where $\pi_h(a\mid s)$ denotes the probability of taking action $a$ given state $s$ at step $h$. Once the episode begins, at each step $h\in[H]$, the agent samples an action $a_h\sim\pi_h(\cdot\mid s_h)$. Upon taking $a_h$, the agent observes $f_h^k$ and $g_h^k(s_h,a_h)$, which are full-information feedback for the adversarial loss and bandit feedback for the stochastic cost, respectively. Lastly, the next state is sampled as $s_{h+1}\sim \bbP_h(\cdot\mid s_h,a_h)$.

 We define the value function and the $Q$-function. Let $V_{\ell,h}^\pi(s)$ denote the value function at state $s$ and step $h$ with respect to a function $\ell = \{\ell_h\}_{h=1}^H$ and policy $\pi$, which is written as $V_{\ell,h}^\pi(s) = \bbE_{\bbP, \pi}[\sum_{j=h}^H \ell_j(s_j,a_j)\mid s_h=s]$. 
Similarly, the $Q$-function $Q_{\ell,h}^\pi(s,a)$ is defined
 as $Q_{\ell,h}^\pi(s,a) = \bbE_{\bbP, \pi}[\sum_{j=h}^H \ell_j(s_j,a_j)\mid s_h=s, a_h=a]$.

 We define the performance metrics---regret and constraint violation---as follows. Given a sequence of policies $\pi^1,\ldots,\pi^K$ generated by the agent, the regret and constraint violation for $K$ episodes are defined as $\Regret = \sum_{k=1}^K(V_{f^k,1}^{\pi^k}(s_1) - V_{f^k,1}^{\pi^*}(s_1))$ and $\Violation = \left[\sum_{k=1}^K(V_{g,1}^{\pi^k}(s_1) -b)\right]_+$, where $[\cdot]_+$ denotes $\max\{\cdot,0\}$. Here, $\pi^*$ is an optimal policy, defined as a solution to the following optimization problem over the set of all policies $\Pi$: $\pi^* \in\argmin_{\pi\in\Pi} \sum_{k=1}^K V_{f^k,1}^\pi(s_1) \ \text{s.t.}\ V_{g,1}^{\pi}(s_1) \leq b.$

\paragraph{Linear CMDP} We consider a class of CMDP instances with an underlying linear structure, referred to as linear CMDP~\citep{ghosh2022provably}. Let $\phi:\calS\times\calA\rightarrow \bbR^d$ denote the known feature mapping. With the feature $\phi$, the transition kernel is defined as $\bbP_h(s'\mid s,a) = \phi(s,a)^\top \psi_h(s')$
where $\psi_h(s')\in\bbR^d$ is an unknown signed measure. Similarly, the loss and cost functions are assumed to be linear in $\phi$ and are defined as $f_h^k(s,a) = \phi(s,a)^\top \theta_{f,h}^k$ and $g_h(s,a) = \phi(s,a)^\top \theta_{g,h}$, where $\theta_{f,h}^k, \theta_{g,h} \in \bbR^d$ are unknown parameters. Moreover, we further assume that the parameters for linear CMDPs are all bounded as follows. For all $(s,a,h,k)\in\calS\times\calA \times[H] \times[K]$, we have $\|\phi(s,a)\|_2 \leq 1$ and $\max\{\|\sum_{s'\in\calS}|\psi_h|(s')\|_2, \|\theta_{f,h}^k\|_2, \|\theta_{g,h}\|_2\}\leq \sqrt{d}$, where $|\psi_h|(s')$ denotes $(|(\psi_h(s'))_1|, |(\psi_h(s'))_2|, \ldots, |(\psi_h(s'))_d|)^\top \in \bbR^d$.

 Next, we introduce the Slater condition, which is a mild assumption commonly made in the CMDP literature~\citep{efroni2020exploration, liu2021learning, ding2021provably, ghosh2022provably}.
\begin{assumption}[Slater Condition]\label{ass:slater}
    We assume that there exists a Slater policy $\bar\pi\in\Pi$ such that $V_{g,1}^{\bar\pi}(s_1) +\gamma\leq b$ for some Slater constant $\gamma >0$. 
\end{assumption}

\section{Challenges and Novel Techniques}\label{sec:novelty}


\paragraph{Novelty 1: Analysis for Weighted LogSumExp Softmax} 
We construct a covering number argument with a new policy structure---weighted LogSumExp softmax policies---which arises from combining policy optimization with policy mixing. This policy is given by the weighted sum of exponentials of sums of $Q$-function estimates\footnote{We call this formulation the weighted LogSumExp softmax, as it is equivalent to $\widehat\pi^k\propto\exp(\log(\sum_{i}\zeta_i\exp(-\alpha \sum_{j}\widehat Q^j)))$---a softmax of weighted LogSumExp with respect to $-\alpha \sum_{j}\widehat Q^j$.}: given a step size $\alpha$, weight parameters $\zeta_i$, and $Q$-function estimates $\widehat Q^j$,
\begin{align}\label{eq:weighted LSE softmax}
    \widehat\pi^k \propto \sum_{i=1}^k\zeta_i \exp\left(-\alpha \sum_{j=k-i}^{k-1}\widehat Q^j\right).
\end{align}
Let us explain how \eqref{eq:weighted LSE softmax} arises in primal–dual policy optimization. Perturbing the primal variable before optimizing it is a simple yet effective technique for controlling the scale of the dual variable~\citep{wei2020online, qiu2020upper}. In the context of policy optimization, this technique translates into the following update: given a uniform policy $\piunif$ over $\calA$ and a mixing parameter $\theta$,
\begin{align*}
    \underbrace{\widehat\pi^{k-1} \leftarrow (1-\theta)\widehat\pi^{k-1} + \theta\piunif}_{\text{Policy Mixing}} \quad\text{and then}\quad \underbrace{\widehat\pi^k \propto \widehat\pi^{k-1}\exp(-\alpha \widehat Q^{k-1})}_{\text{Policy Optimization}}.
\end{align*}
Here, the additive relation (Policy Mixing) breaks the recursion in the proportional relation (Policy Optimization). As a consequence, the resulting policy takes the form of the weighted LogSumExp softmax. In particular, \eqref{eq:weighted LSE softmax} may assign different weights $\zeta_i$ to partial sums $\sum_{j=k-i}^{k-1}\widehat Q^j$ for $i\in[k]$. This yields a more expressive policy compared to the case without policy mixing, where the update simplifies to $\widehat\pi^k\propto \exp(-\alpha\sum_{j=0}^{k-1} \widehat Q^j)$.

Our first contribution is to provide a new covering number argument for the value function class, where the policy is given by \eqref{eq:weighted LSE softmax}. For comparison, \citet{jin2020provably} studied the greedy policy, where the policy is defined as $\argmax_{a\in\calA} \widehat Q^k$. Then the covering number can be analyzed since the $\max$ operation is a contraction mapping. Moreover, the simple softmax policy has been studied in several works, e.g., \citet{ghosh2022provably}. In that case, leveraging well-established Lipschitz properties of the softmax function is sufficient to analyze the covering number.

We note that constructing a covering number argument for \eqref{eq:weighted LSE softmax} is non-trivial. The main difficulty is that the weight parameters $\{\zeta_i\}_{i=1}^k$ depend not only on the mixing parameter $\theta$ but also on $Q$-function estimates. This means that different policies can have different weight parameters $\{\zeta_i\}_{i=1}^k$, and thus, well-known properties of LogSumExp cannot be applied. Despite these challenges, our analysis shows that the logarithm of the covering number under \eqref{eq:weighted LSE softmax}, denoted by $\log\calN_\epsilon$, grows quadratically with $n$, where $n$ is the maximum number of mixing steps during the learning process:
\begin{align} \label{eq:covering informal}
    \log \calN_\epsilon =\widetilde\calO(n^2d^2).
\end{align}

\paragraph{Novelty 2: Periodic Policy Mixing}
However, deriving an upper bound on the covering number alone is not sufficient to guarantee sublinear regret and violation. In particular, if mixing is applied in every episode, then $\log\calN_\epsilon$ grows to the order of $\widetilde\calO(K^2d^2)$, which is too large to yield a sublinear guarantee. On the other hand, if mixing is performed insufficiently, then the dual variable cannot be bounded, which is critical for violation analysis. These observations highlight an inherent trade-off between the covering number and the size of dual variables, both of which heavily depend on the frequency of mixing. This necessitates a new algorithmic component to balance the two.

The aforementioned trade-off motivates our second contribution
---periodic policy mixing---which applies the policy mixing every $K^B$ episodes\footnote{For simplicity, we assume that $K^B, K^{1-B}$ are integers to avoid additional notation such as $\floor{K^B}$.}, where $B$ is a period parameter between 0 and 1. The purpose of the periodic policy mixing is to balance the covering number and the size of dual variables. The covering number can be easily observed from \eqref{eq:covering informal}, since the number of mixing steps is at most $K^{1-B}$ (i.e., the number of episodes divided by the mixing period). However, it remains unclear whether periodic policy mixing is effective in controlling the dual variable. To address this, in the next paragraph, we show that the dual variable can indeed be bounded when periodic policy mixing is combined with a new dual update rule.

\paragraph{Novelty 3: Regularized Dual Update}
Our third contribution is another algorithmic component---a new dual update rule with additional regularization. When combined with periodic policy mixing, the dual variable $Y_k$ is bounded by $\widetilde\calO(\eta K^B)$, where $\eta$ is the step size and $K^B$ is the mixing period. For clarity, this bound omits the dependence on $\gamma, H$ to highlight the impact of the mixing period.

To elaborate on our dual update method, it takes the following form: given regularization parameters $c_1,c_2 > 0$ and the cost value function estimate $\widehat V_{g,1}^k$,
\begin{align*} 
    Y_{k+1} \leftarrow [Y_k + \eta(\widehat V_{g,1}^k(s_1) - b) +\underbrace{(-c_1Y_k - c_2)}_{\text{Regularization}}]_+.
\end{align*}
For interpretation, $Y_k + \eta(\widehat V_{g,1}^k(s_1) - b)$ corresponds to the standard online gradient ascent step for dual updates in primal-dual algorithms, while the regularization term $(-c_1Y_k - c_2)$ pulls the dual variable towards $0$, keeping it compact. 
The regularization parameters $c_1,c_2$ will be specified in the following paragraph, along with the intuition for our design.

The intuition behind our regularization is that it serves as a crucial ingredient for a drift-based analysis, a well-known method for bounding dual variables (see, e.g., \citet{yu2017online, wei2020online} in constrained online convex optimization). To enable drift analysis, these works typically incorporate an inner product term in the dual update, determined by the decision variables and the gradient, i.e., $\langle x_{t+1}-x_t, \nabla_t \rangle$. In primal-dual policy optimization, we realize that this translates to a term involving the transition kernel, i.e., $\bbE_{\bbP}[\langle \widehat\pi^{k+1} - \widehat\pi^k, \widehat Q^k\rangle]$. However, since the transition kernel is unknown, this term cannot be directly incorporated into our algorithm. Instead, we take a lower bound on this term, which becomes our regularizing component with the choice of $c_1=4\alpha\eta H^3$ and $c_2 = 4\alpha\eta H^3 + 4\theta\eta H^2$. In this way, our dual update can be viewed as a key adaptation that enables drift analysis in primal-dual policy optimization for adversarial linear CMDPs. 

\section{Algorithm}\label{sec:algorithm}

We present \AlgName (\Cref{alg:main-linear}). The algorithm consists of four main components: epoch initialization (lines 2-7), policy execution and estimation (lines 8-19), policy optimization with periodic policy mixing (lines 20-26), and updating the dual variable (line 27).
\begin{algorithm}[t]
\caption{\AlgName}
\label{alg:main-linear}
\textbf{Input:} $\delta\in (0,1),\beta_b=2\sqrt{2d\log(6KH/\delta)} + 50(K^{1/4}+1)dH\sqrt{\log(5H^2K^2|\calA|/\delta)}, \beta_w=4\beta_b\log K, \alpha=H^{-1}K^{-3/4}, \eta=H^{-2}K^{-3/4}, \theta=K^{-1}$ \\
\textbf{Initialization:} $Y_1\leftarrow 0, \ e \leftarrow 0,\ \Lambda_h^1 \leftarrow I,\ \widehat V_{\ell,H+1}^k(s)\leftarrow 0 \quad \forall (h,k,s,\ell) \in [H]\times[K]\times \calS \times \{f,g\}$ 
\begin{algorithmic}[1]
    \For{$k=1,\ldots,K$}

        \If{$k=1$ or $\exists h' \in [H]$ such that $\det(\Lambda_{h'}^k) \geq 2 \det(\Lambda_{h'}^{k_e})$} 
            \State $e \leftarrow e+1$ and $k_e \leftarrow k$
            \State $\widehat\pi_h^{k_e}(\cdot \mid s) \leftarrow \piunif(\cdot\mid s)$ $\ \forall h\in[H]$ 
            \State $Y_{k_e} \leftarrow 0$
            \State $\bar\phi_h^{k_e}(\cdot,\cdot) = \phi(\cdot,\cdot) \cdot \sigma(-\beta_w \|\phi(\cdot,\cdot)\|_{(\Lambda_h^{k_e})^{-1}} + \log K)$ $\ \forall h\in [H]$ 
            \Comment{Feature Contraction}
        \EndIf

        \For{$h=1,\ldots,H$} 
            \State Take $a_h^k \sim \widehat\pi_{h}^k(\cdot\mid s_h^k)$, and observe $\theta_{f,h}^k,\ g_h^k(s_h^k, a_h^k),\ s_{h+1}^k\sim\bbP_h(\cdot\mid s_h^k,a_h^k)$ 
        \EndFor

        \For{$h=H,\ldots,1$} 
            \State $\Lambda_h^{k+1} \leftarrow I + \sum_{\tau\in[k]} \phi(s_h^\tau, a_h^\tau) \phi(s_h^\tau, a_h^\tau)^\top$
            \State $\widehat\theta_{f,h}^k \leftarrow \theta_{f,h}^k$
            \State $\widehat \theta_{g,h}^k \leftarrow  (\Lambda_h^k)^{-1} \sum_{\tau \in [k-1]} \phi(s_h^\tau, a_h^\tau) g_h^\tau(s_h^\tau, a_h^\tau)$
            \For{$\ell \in \{f,g\}$}
                \State $\widehat\psi_h^k \widehat V_{\ell,h+1}^k = (\Lambda_h^k)^{-1} \sum_{\tau\in[k-1]}\phi(s_h^\tau, a_h^\tau) \widehat V_{\ell,h+1}^k(s_{h+1}^\tau)$
                \State $\widehat Q_{\ell,h}^k(\cdot,\cdot) \leftarrow \bar\phi_h^{k_e}(\cdot,\cdot)^\top \left[\widehat \theta_{\ell,h}^k + \widehat\psi_h^k \widehat V_{\ell,h+1}^k\right] - \beta_b \|\bar\phi_h^{k_e}(\cdot,\cdot)\|_{(\Lambda_h^{k_e})^{-1}}$
                \State $\widehat V_{\ell,h}^k(\cdot) \leftarrow \sum_{a\in\calA} \widehat\pi_h^k(a\mid \cdot) \widehat Q_{\ell,h}^k(\cdot,a)$
            \EndFor
            \If{$k-k_e\equiv 0 \mod K^{3/4}$} \Comment{Periodic Policy Mixing}
                \State $\widetilde\pi_h^k(\cdot\mid s) \leftarrow (1-\theta)\widehat\pi_h^k(\cdot\mid s) + \theta\piunif(\cdot\mid s)$
            \Else
                \State $\widetilde\pi_h^k(\cdot\mid s) \leftarrow \widehat\pi_h^k(\cdot\mid s)$
            \EndIf
            \State $\widehat\pi_h^{k+1}(\cdot\mid s) \propto \widetilde\pi_h^k(\cdot\mid s)\exp\left(-\alpha( \widehat Q_{f,h}^{k}(s,\cdot) + Y_{k} \widehat Q_{g,h}^{k}(s,\cdot))\right)$ \Comment{Policy Optimization}
        \EndFor

        \State $Y_{k+1} \leftarrow \left[(1-4\alpha\eta H^3)Y_{k}+\eta\left(\widehat V_{g,1}^{k}(s_1)- b 
        -4\alpha H^3 - 4\theta H^2\right) \right]_+$ \Comment{Dual Update}
    \EndFor
\end{algorithmic}
\end{algorithm} 

In lines 2-7, the algorithm initializes a new epoch when the determinant of the design matrix $\Lambda_{h'}^k$ increases by a multiplicative factor compared to that of $\Lambda_{h'}^{k_e}$ for some $h'$. Once the initialization procedure begins, the algorithm sets the policy to the uniform policy and initializes the dual variable to $0$. Furthermore, it defines a contracted feature $\bar\phi_h^{k_e}$ by shrinking the original feature $\phi$. The multiplicative contraction factor is determined by $\sigma(-\beta_w \|\phi(\cdot,\cdot)\|_{(\Lambda_h^{k_e})^{-1}}+ \log K)$, where $\sigma$ denotes the sigmoid function, and $\|\phi(\cdot,\cdot)\|_{(\Lambda_h^{k_e})^{-1}}$ quantifies the current uncertainty of least-squares estimators. This contracted feature is then used in the estimation of $Q$-functions.

\begin{remark}\label{remark:feature contraction}
{\rm
The feature contraction---originally proposed by \citet{cassel2024warm} for adversarial linear (unconstrained) MDPs---is necessary for the following reason. Specifically, it provides a simpler expression for the policy, which is useful in covering number arguments, via a low-dimensional representation of the sum of $Q$-function estimates. For this, they omitted a clipping operation in the definition of $Q$-function estimates, so the sum collapses into a simple inner product with an optimistic bonus. Instead of clipping, they properly contracted the feature to prevent $Q$-function estimates from expanding uncontrollably. This technique can be replaced with other approaches with the same purpose, such as \citet{sherman2024rate}, but it may lead to higher dependence on $d,H$ in regret and violation bounds.
}    
\end{remark}

In lines 8-10, the algorithm takes action $a_h^k\sim\widehat\pi_h^k(\cdot|s_h^k)$ for each step $h\in [H]$ and observes $\theta_{f,h}^k, g_h^k(s_h^k,a_h^k)$, and $s_{h+1}^k\sim \bbP_h(\cdot|s_h^k,a_h^k)$. In lines 11-14, the design matrix $\Lambda_h^{k+1}$ is updated, and the parameters for the loss and cost functions are estimated, denoted by $\widehat\theta_{f,h}^k$ and $\widehat \theta_{g,h}^k$, respectively. Based on these, in lines 15-19, for each $\ell=f,g$, it computes the $Q$-function estimates $\widehat Q_{\ell,h}^k(s,a)$ using the contracted feature and the value function estimates $\widehat V_{\ell,h}^k(s)$, which are defined by the inner product of $\widehat\pi_h^k(\cdot|s)$ and $\widehat Q_{\ell,h}^k(s,\cdot)$ for each $s\in\calS$. 

In lines 20-24, the algorithm applies the policy mixing every $K^{3/4}$ episodes. Here, the mixed policy is obtained by taking a convex combination of $\widehat\pi_h^k$ and $\piunif$ with coefficients $1-\theta$ and $\theta$, respectively. 
After this, the algorithm performs policy optimization---equivalently, an online mirror descent step with Kullback-Leibler (KL) divergence over the policy space.

In line 27, the algorithm updates the dual variable, $Y_k$. First, it scales down the dual variable by a factor of $1-4\alpha\eta H^3$, and then adds $\eta(\widehat V_{g,1}^k(s_1) - b - 4\alpha H^3 - 4\theta H^2)$. Finally, it takes $[\cdot]_+$ to ensure that the dual variable remains nonnegative. 
 
The computational complexity of \Cref{alg:main-linear} is $\calO(d^3HK + d^2 |\calA|HK^2)$, which is independent of $|\calS|$. Specifically, in lines 2-7, computing determinants simply takes $\calO(d^3 HK)$ and contracting features takes $\calO(d^2K \cdot |\calA|\cdot HK)$, since the inverse of the design matrices can be computed in $\calO(d^2K)$ by applying the Sherman-Morrison formula. In lines 8-29, the dominant step is function estimation: lines 16-18 take $\calO(d^2|\calA|HK^2)$.

\subsection{Comparison of Dual Updates}
Since our algorithm is designed for adversarial linear CMDPs, it is worth comparing our dual update with that of algorithms for (i) stochastic linear CMDPs~\citep{ghosh2022provably} and (ii) tabular CMDPs with adversarial losses~\citep{qiu2020upper}.

First, we compare with \citet{ghosh2022provably}, which focused on stochastic linear CMDPs. The key difference comes from the way the dual variable is regularized. Specifically, while both approaches adopt the standard online gradient ascent procedure, their update rule truncates the dual variable at $2H/\gamma$ to ensure that it never exceeds this threshold. In contrast, our update incorporates an extra regularization term to keep the dual variable compact. Although their update is simple and effective in the stochastic setting, it cannot be extended to the adversarial setting, since their analysis relies on the fact that the loss and cost functions are fixed over episodes. This justifies the need for our design of dual update steps in handling adversarial losses.

Second, we compare with \citet{qiu2020upper}, which proposed an occupancy measure-based algorithm for adversarial tabular CMDPs. The main difference in the dual updates arises from the choice of primal variable: policy-based mirror descent versus occupancy measure-based mirror descent. Before elaborating on this, we recall a dual update from the constrained online convex optimization literature, proposed by \citet{wei2020online} with minor modifications: given a convex cost function $\ell:\bbR^d \rightarrow \bbR$ and primal variables $x^k, x^{k+1}\in\bbR^d$,
\[
    Y_{k+1} \leftarrow \left[Y_k + \eta\left(\ell(x^{k}) - b + \langle \nabla \ell(x^k), x^{k+1} - x^{k} \rangle  \right)\right]_+.
\]
Based on this update, let us show how the dual update for occupancy measure-based algorithms can be derived. Since the occupancy measure serves as the primal variable, we take $x^k\leftarrow q^k$, where $q^k$ denotes an occupancy measure in episode $k$. Furthermore, in CMDPs, note that the expected cost is given by $\langle g, q^k \rangle$\footnote{For simplicity, we assume that $g$ is deterministic and known, and $q^k$ is induced by the true transition kernel.}, where $g\in\bbR^{|\calS|\times|\calA|\times H}$ denotes a vector representation of the cost function. Then we can take $\ell(x^k) \leftarrow \langle g,q^k \rangle$ and $\nabla \ell(x^k) \leftarrow g$. This leads to $Y_{k+1}\leftarrow [Y_k + \eta (\langle g,q^{k+1}\rangle - b)]_+$, which is the key intuition behind the dual update in \citet{qiu2020upper}. 

However, this argument does not apply to policy-based mirror descent. This is because even if we take $x^k\leftarrow \pi^k$, the expected cost is not linear in $\pi^k$, unlike in the occupancy measure case. That said, the dual updates for occupancy measure-based algorithms can be extended from the online convex optimization literature, whereas extending this to policy-based algorithms is non-trivial. This highlights the significance of our proposed design.


\subsection{Main Result}
Finally, we present upper bounds on regret and constraint violation under \Cref{alg:main-linear}.
\begin{theorem}\label{thm:main} Let $H^2\leq K$ and Assumption \ref{ass:slater} hold. Suppose that we run \Cref{alg:main-linear}. Given $\delta >0$, with probability at least $1-\delta$, then we have
    \begin{align*}
        &\Regret = \widetilde\calO\left(\sqrt{d^3 H^4}K^{3/4} + dH^3 K^{3/4} + d^3 H^4 K^{1/2} + \frac{H^6}{\gamma^2} K^{1/4} + \frac{dH^6}{\gamma^2}\right),\\
        &\Violation = \widetilde\calO\left(\frac{dH^5}{\gamma}K^{3/4} + \sqrt{d^3 H^4}K^{3/4} +  d^3 H^4 K^{1/2}\right)
    \end{align*}
    where $\widetilde\calO \left(\cdot\right)$ hides polynomial factors in $\log(dHK|\calA|/(\delta\gamma))$.
\end{theorem}

\section{Analysis}\label{sec:analysis}
In this section, we present the proof outline of \Cref{thm:main}, where the details of the proofs can be found in the appendix. As a first step, we introduce two key ingredients: (i) a high-probability good event and (ii) bounding the dual variable. We note that our covering number argument plays a central role in showing that the good event holds with high probability. Furthermore, to bound the dual variable, the key part is to consider periodic policy mixing and the regularized dual update.

\paragraph{Good Event}
We first introduce a high-probability event, denoted by $E_g$, whose formal definition is provided in the appendix. Basically, the event captures estimation errors for the loss, the cost, and the transition kernel. In addition, it guarantees the boundedness of $Q$-function estimates, reflecting the usefulness of the feature contraction. The following lemma shows that $E_g$ holds with high probability. 
\begin{lemma}\label{lem:good event main-text} Let $\beta_w \leq K$. Then for any $\delta\in(0,1)$, $\Pr[E_g] \geq 1-\delta$.
\end{lemma}
In the proof of \Cref{lem:good event main-text}, the main distinction from previous works arises in our covering number argument. Specifically, our case incorporates weighted LogSumExp softmax policies, induced by mixing the policy. Then we have to derive a Lipschitz property of the policies that have been mixed $n$ times, as the number of mixing steps is a key factor in determining the policy structure. Note that the Lipschitzness of policies is fundamentally required in most covering number arguments. 


To attain this property, we prove the following recursion in $n$, presented here in simplified form:
\[
    \|\widehat\pi_1^n - \widehat\pi_2^n\|_1 \leq c\|\widehat\pi_1^{n-1}-\widehat\pi_2^{n-1}\|_1 + \|\calP_1^n - \calP_2^n\|_2
\]
where $\widehat\pi_1^n, \widehat\pi_2^n$ denote the policies that have been mixed $n$ times, $\calP_1^n, \calP_2^n$ denote the subsets of the corresponding parameters, and $c$ is a constant. By applying this recursion repeatedly, we can bound the difference between policies using the sum of differences in their parameters, establishing the Lipschitzness. Based on this, we can show that the covering number is bounded as $\widetilde\calO(n^2d^2)$.


\paragraph{Dual Variable Bound} 
Under the good event $E_g$, we can establish another ingredient of our analysis---a drift analysis for bounding the dual variable.
\begin{lemma}\label{lem:Yk bound main-text}
    Assume that the good event $E_g$ holds. Let $H^2 \leq K$. Let $Y_k$ be the dual variable generated by \Cref{alg:main-linear} for each $k\in [K]$. 
    Then, we have $Y_k = \widetilde\calO(H^2/\gamma)$.
\end{lemma}

Let us briefly explain our proof strategy for \Cref{lem:Yk bound main-text}. Although the regularization in our dual update enables drift analysis, we cannot directly apply the previous proofs proposed by \citet{wei2020online, qiu2020upper}. This is because their analyses rely on applying policy mixing in every episode, whereas our algorithm applies it only sparsely. To exploit this sparse structure, we instead consider a subsequence of dual variables corresponding to the mixing episodes, denoted by $\{Z_n\}_{n\geq 1}$ where $Z_n = Y_{k_e + nK^B}$ for each epoch $e$. We first bound $Z_n$ for all $n$, and consequently extend the result to derive a bound on $Y_k$ for all $k$.

Next, we introduce decompositions of both $\Regret$ and $\Violation$. Let $E$ be the set of all epochs, and let $K_e$ be the set of episodes in epoch $e\in E$. We have
{\allowdisplaybreaks
\begin{align}\label{eq:decomposition}
\begin{aligned}
    \Regret 
    &\leq \underbrace{\sum_{k=1}^K \left(V_{f^k,1}^{\pi^k}(s_1) - \widehat V_{f,1}^k(s_1)\right)}_{\text{(I)}}+ \underbrace{\sum_{k=1}^K Y_k\left(b - \widehat  V_{g,1}^k(s_1)\right)}_{\text{(II)}}\\
    &\quad+
    \underbrace{
     \sum_{k=1}^K \left(\widehat V_{f,1}^k(s_1) + Y_k \widehat V_{g,1}^k(s_1) - V_{f^k,1}^{\pi^*} - Y_k V_{g,1}^{\pi^*}(s_1)\right)
    }_{\text{(III)}},
    \\
    \Violation 
    &\leq \underbrace{\sum_{k=1}^K\left( V_{g,1}^{\pi^k}(s_1)-\widehat V_{g,1}^{k}(s_1)\right)}_{\text{(IV)}}+\underbrace{\sum_{k=1}^K \left(\widehat V_{g,1}^k(s_1) - b\right)}_{\text{(V)}}.
\end{aligned}
\end{align}}

Terms (I), (IV) arise from the difference between the true value function and its optimistic estimates, which are closely related to the optimistic bonus $-\beta_b\|\bar\phi_h^{k_e}(\cdot,\cdot)\|_{(\Lambda_h^{k_e})^{-1}}$. Since our parameter for optimistic bonus is set as $\beta_b = \widetilde\calO(K^{1/4}dH)$,  where $K^{1/4}$ comes from the covering number argument, these terms are bounded by $\widetilde\calO(K^{3/4})$, as stated in the following lemma. 
\begin{lemma}\label{lem:cost of opt}
    Let $H^2\leq K$. Suppose that $E_g$ holds. For all $\ell\in\{f,g\}$,
    $$\sum_{k=1}^K(V_{\ell,1}^{\pi^k}(s_1) - \widehat V_{\ell,1}^{k}(s_1)) =\widetilde\calO\left(\sqrt{d^3H^4}K^{3/4} + d^3H^4 K^{1/2}\right).$$
\end{lemma}
Term (II) arises from the dual update in the sense that if the dual variable is not updated (i.e., $Y_k=0$ for all $k$), then this term vanishes. It can be bounded using the following lemma.
\begin{lemma}\label{lem:proj gd}
    Let $H^2\leq K$. Suppose that $E_g$ holds. Then we have
    $$ \sum_{k=1}^K Y_k(b-\widehat V_{g,1}^k(s_1)) =\widetilde\calO\left(K^{1/4} + \frac{dH^6}{\gamma^2}\right).$$
\end{lemma}

To bound term (III), we further decompose it into two parts—optimism terms and an online mirror descent term—and then bound each individually. This leads to the following lemma.
\begin{lemma}\label{lem:omd+optimism}
    Let $H^2\leq K$. Suppose that $E_g$ holds. Then we have
    $$\sum_{k=1}^K \left(\widehat V_{f,1}^k(s_1) + Y_k \widehat V_{g,1}^k(s_1) - (V_{f^k,1}^{\pi^*} + Y_k V_{g,1}^{\pi^*}(s_1))\right)=\widetilde\calO \left(dH^3 K^{3/4} + \frac{H^6}{\gamma^2}K^{1/4} + \frac{dH^5}{\gamma}\right).$$
\end{lemma}

Finally, term (V) can be bounded via the dual variable bound. This is because $\eta(\widehat V_{g,1}^k(s_1) - b)$ accumulates in the dual variable, as it is repeatedly added in the dual update. Based on this idea, we show the following lemma, which bounds term (V).
\begin{lemma}\label{lem:violation optimism}
    Let $H^2\leq K$. Suppose that $E_g$ holds. Then we have
    $$\sum_{k=1}^K (\widehat V_{g,1}^k(s_1)-b) =\widetilde\calO\left(\frac{dH^5}{\gamma}K^{3/4} + \frac{H^4}{\gamma}K^{1/4}\right).$$
\end{lemma}

\section{Numerical Experiment}\label{sec:numerical}
We evaluate \Cref{alg:main-linear} on a job-scheduling CMDP~\citep{ghosh2022provably}, modified to incorporate adversarial losses. We conduct $10$ simulations with different random seeds, each running for $K=10^5$ episodes. Additional details about the experimental setup are deferred to the appendix.
\begin{figure}[H]
    \centering
    \begin{subfigure}{0.35\textwidth}
        \centering
        \includegraphics[width=\linewidth]{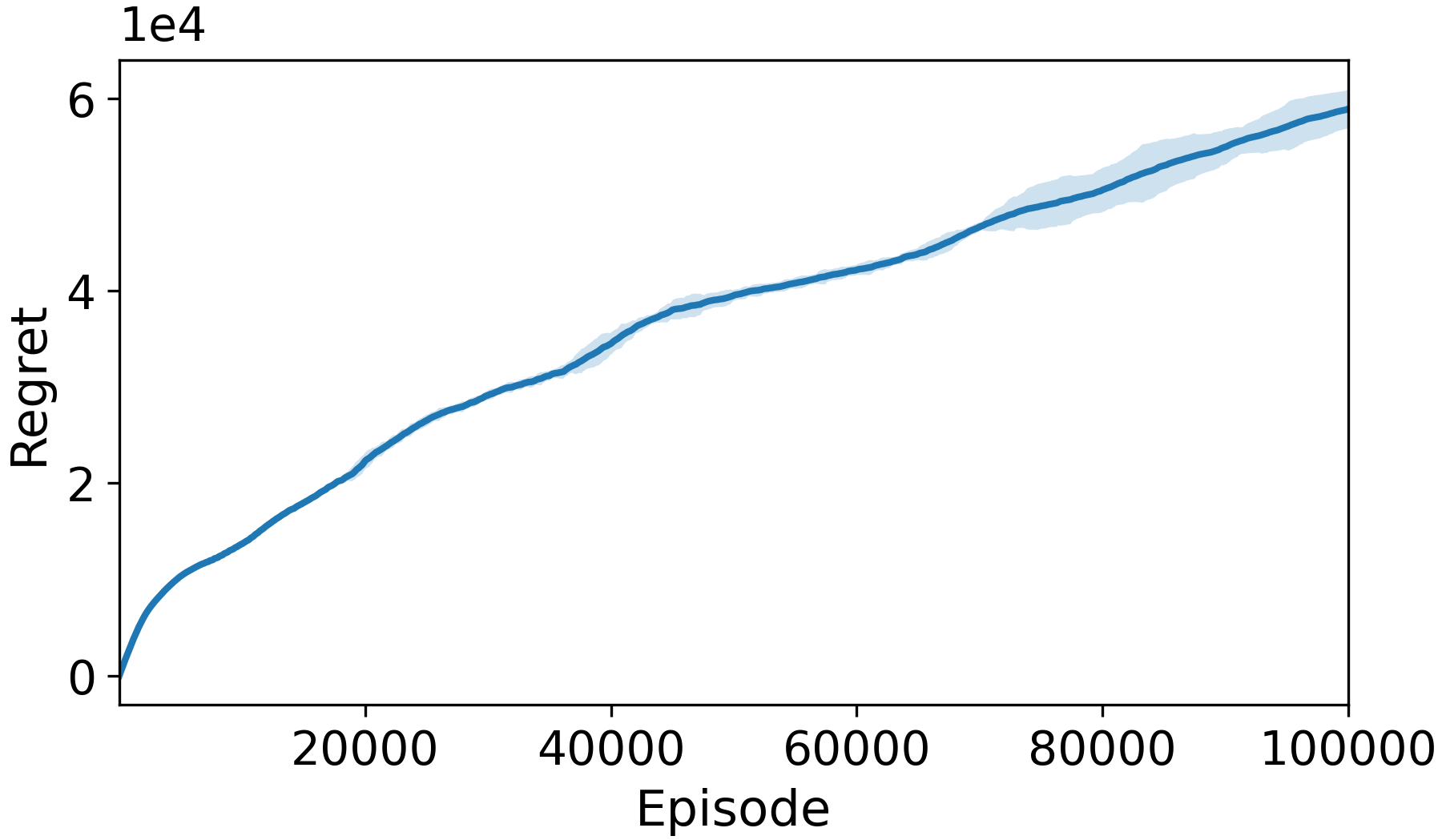}
        \caption{Regret}
        \label{fig:regret}
    \end{subfigure}%
    \hspace{0.06\textwidth}
    \begin{subfigure}{0.35\textwidth}
        \centering
        \includegraphics[width=\linewidth]{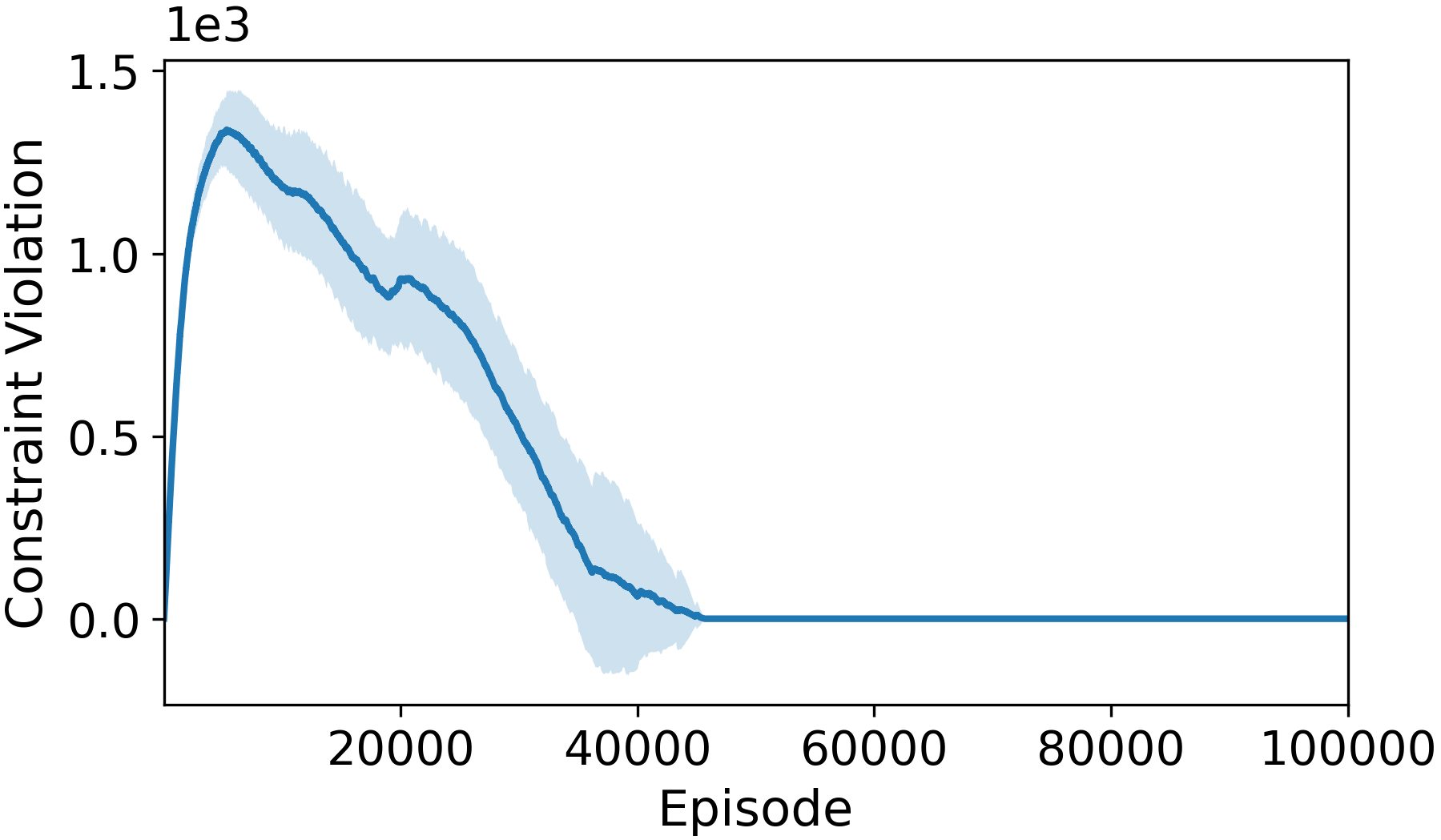}
        \caption{Constraint Violation}
        \label{fig:violation}
    \end{subfigure}
    \caption{Plots of regret and constraint violation for $K=100{,}000$ episodes. Each plot represents the average over $10$ trials with random seeds, and shaded regions indicate $95\%$ confidence intervals.}
    \label{fig:main}
\end{figure}
Figure~\ref{fig:main} summarizes the results. As shown in \Cref{fig:regret}, the regret grows sublinearly in $K$, and \Cref{fig:violation} shows that while the constraint violation grows rapidly in the early phase, it eventually converges to $0$. These results support our theoretical claims.

\section{Conclusion}\label{sec:conclusion}
This paper studies adversarial linear CMDPs, where the losses are adversarially chosen under full-information feedback and the costs are stochastic under bandit feedback. We propose a primal-dual policy optimization algorithm—the first provably efficient algorithm for safe RL with linear function approximation in adversarial settings. We establish a new covering number argument for weighted LogSumExp softmax policies, along with novel algorithmic components that jointly control the covering number and the dual variable. Building on these, we show that the proposed algorithm achieves $\widetilde\calO(K^{3/4})$ regret and violation bounds. Moreover, our numerical experiments support this. As directions for future work, it remains open to investigate the following challenges: (i) whether $\widetilde\calO(\sqrt{K})$ regret and violation bounds can be achieved in our setting, and (ii) whether a sample-efficient algorithm can be designed for linear CMDPs with adversarial losses under bandit feedback.


\subsubsection*{Acknowledgments}


We would like to thank the review team for their careful review and valuable feedback. This work was supported by the National Research Foundation of Korea (NRF) grant (No. RS-2024-00350703) and the Institute of Information \& communications Technology Planning \& evaluation (IITP) grants (No. IITP-2026-RS-2024-00437268) and (No. RS-2021-II211343, Artificial Intelligence Graduate School Program (Seoul National University)) funded by the Korea government (MSIT).

\bibliography{iclr2026_conference}
\bibliographystyle{iclr2026_conference}
\appendix
\newpage

\section{Notation}
\begin{table}[H]
    \caption{Summary of notation} 
    \label{tab:notation}
    \begin{center}
    \begin{tabular}{cl}
        \multicolumn{1}{c}{\bf NOTATION}  &\multicolumn{1}{c}{\bf DEFINITION}
        \\ \hline \\
        $H,\calS,\calA,K$ & The finite-horizon, the state and action spaces, and the number of episodes\\
        $d$ & The feature dimension \\
        $\bbP$ & The transition kernel \\
        $f,g$ & The loss and the cost functions\\
        $b$ & The budget \\
        $E$ & The set of epochs \\
        $k_e$ & The first episode of epoch $e\in E$ \\
        $\phi$ & The feature \\
        $\bar\phi_h^{k_e}$ & The contracted feature at step $h$ in epoch $e$\\
        $\widehat \theta_{\ell,h}^k$ & The estimate of $\theta_{\ell,h}^k$ \\ 
        $\Lambda_h^{k}$ & The design matrix at step $h$ in episode $k$\\
        $\psi_h V$ & $\sum_{s'\in\calS} \psi_h(s')V(s')$ for $V:\calS \to \bbR$\\
        $\widehat\psi_h^k V$ & The estimate of $\psi_h V$ at step $h$ in episode $k$\\
        $[n]$ & The set $\{1,2, \ldots, n\}$ for a positive integer $n$\\
        $\bbZ_+$ & The set $\{0,1,2,\ldots\}$\\
        $\bbR_+$ & The set $\{z\in\bbR: z\geq 0\}$\\
        $\|\cdot\|_2$ & The $\ell_2$-norm for vectors and the operator norm for matrices \\
        $\|\cdot\|_\infty$ & The $\ell_\infty$-norm\\
        $\|\cdot\|_F$ & The Frobenius norm \\
        $\|\cdot\|_\Lambda$ & $\|x\|_\Lambda = \sqrt{x^\top \Lambda x}$ for $\Lambda \succ 0$\\
        $\Delta(\calA)$ & The probability simplex over $\calA$ \\
        $I$ & The $d\times d$ identity matrix\\
        $D(\cdot||\cdot)$ & Kullback-Leibler divergence \\
        $\sigma$ & The sigmoid function \\
        $\gamma,\bar\pi$ & The Slater constant and the Slater policy\\
        $Q_{\ell,h}^{\pi}(s,a)$ & The $Q$-function\\
        $\widehat Q_{\ell,h}^k(s,a)$ & The $Q$-function estimate\\
        $V_{\ell,h}^\pi(s)$ & The value function\\
        $\widehat V_{\ell,h}^k(s)$ & the value function estimate\\
        $\bar V_{\ell,h}^\pi(s;\rho)$ & The value function with respect to a $\rho$-contracted MDP\\
        $\widehat\pi_h^k$ & The policy at step $h$ in episode $k$\\
        $\piunif$ & $\piunif(a\mid s) = 1/|\calA|$ for all $(s,a)\in \calS \times \calA$ \\
        $Y_k$ & The dual variable in episode $k$\\
        $K^B$ & The mixing period, $K^{3/4}$ \\
        $\theta$ & The mixing parameter, $\theta = K^{-1}$  \\
        $\alpha$ & The step size for the mirror descent, $\alpha = H^{-1}K^{-3/4}$ \\
        $\eta$ & The step size for the dual update, $\eta = H^{-2}K^{-3/4}$ \\
        $\beta_r$ & $2\sqrt{2d\log(6KH/\delta)}$\\
        $\beta_p$ & $50(K^{1/4}+1)dH\sqrt{\log(5H^2K^2|\calA|/\delta)}$\\
        $\beta_b$ & $\beta_r + \beta_p$\\
        $\beta_w$ & $4\beta_b\log K$\\
        $\beta_{Q,h}$ & $2(H-h+1)$\\
        $\calN_\epsilon(\widehat \calV)$ & The $\epsilon$-covering number of $\widehat\calV$ with respect to the $\ell_\infty$-norm\\
    \end{tabular}
    \end{center}
\end{table}

\section{Limitations of Prior Work and Na\"ive Extension}

In this section, we clarify why previous works---and their na\"ive extensions---fail in our setting, where the losses are adversarially chosen in each episode. In particular, we address the following two questions: (i) why the algorithm of \citet{ghosh2022provably} fails in the adversarial setting; and (ii) why simply adapting a mirror-descent update is insufficient for adversarial linear CMDPs.

\paragraph{Limitation of Prior Work} While \citet{ghosh2022provably} proposed a primal-dual algorithm for linear CMDPs, their analysis is limited to the setting with stochastic losses and constraints. The fundamental reason is that their algorithm is value-based, whose policy is determined solely by the current $Q$-function estimates in each episode; that is $\widehat\pi^{k}(\cdot\mid s)\propto \exp(-\alpha \widehat Q^{k-1}).$ Such a policy cannot adapt to time-varying environments, and in particular, cannot handle adversarially chosen losses.

Another limitation is that the dual update of \citet{ghosh2022provably} fails to handle adversarial environments. Their key technique is a dual clipping technique---cutting off the dual variable when it exceeds $2H/\gamma$---which enables them to leverage the strong duality of CMDPs. Moreover, to leverage strong duality in their analysis, they reformulate the weighted sum of regret and violation into a simple Lagrangian form (e.g., Appendix D of \citet{ghosh2022provably}), i.e., there exists a policy $\pi'$ such that
$$
    \frac{1}{K}\bigg(\sum_{k=1}^K (V_{f,1}^{\widehat \pi^k} - V_{f,1}^{\pi^\star}) + Y\sum_{k=1}^K (V_{g,1}^{\widehat \pi^k} - b)\bigg) = (V_{f,1}^{\pi'} - V_{f,1}^{\pi^\star}) + Y (V_{g,1}^{\pi'} - b).
$$
In our case, when losses are adversarially chosen in each episode, reformulating the sum into a simple Lagrangian form is not allowed, i.e.,
$$
    \frac{1}{K}\bigg(\sum_{k=1}^K (V_{f^k,1}^{\widehat \pi^k} - V_{f^k,1}^{\pi^\star}) + Y\sum_{k=1}^K (V_{g,1}^{\widehat \pi^k} - b)\bigg) \neq (V_{f',1}^{\pi'} - V_{f',1}^{\pi^\star}) + Y (V_{g,1}^{\pi'} - b).
$$
In turn, leveraging strong duality in our setting is non-trivial.

\paragraph{Na\"ive Extension} To overcome the limitations of value-based algorithms in adversarial settings, a natural approach is to adopt a mirror-descent update, whose regularizer is given by a KL divergence. In our case, this corresponds to a policy optimization, i.e.,
\[
    \widehat\pi^k = \argmin_{\pi\in \Pi}\ \langle \pi, \widehat Q^{k-1} \rangle + \frac{1}{\alpha} D(\pi|| \widehat \pi^{k-1}) \quad\Rightarrow\quad \widehat \pi^{k} \propto \widehat\pi^{k-1} \exp\left(-\alpha \widehat Q^{k-1}\right).
\]
Due to its recursive formulation, we can see that the resulting policy depends on the sum of all previous $Q$-function estimates, namely $\widehat\pi^k \propto \exp(-\alpha \sum_{j=1}^{k-1} \widehat Q^j)$. This is the key difference compared with value-based algorithms.

More technically, let us attempt to adapt the algorithm of \citet{wei2020online} to the linear CMDP setting. Their method is a mirror-descent type algorithm---in our case, policy optimization---designed for constrained online convex optimization with adversarial losses and stochastic constraints. Beyond policy optimization, there are two additional distinctions compared with \citet{ghosh2022provably}:
\begin{itemize}
    \item (Drift Analysis) As previously mentioned, since strong duality is difficult to use in the adversarial setting, the dual clipping technique may not works in the adversarial setting. To address this, \citet{wei2020online} came up with a dual update that admits a \emph{Lyapunov drift analysis}. In particular, Lyapunov drift analysis is a standard tool for bounding the dual variable. For this, we first derive an upper bound on the Lyapunov drift term defined as $\Delta(k):= (Y_{k+1}^2 - Y_k^2)/2$, and then utilize it to bound $Y_k$. The key point is to make $\Delta(k)$ small enough so that $Y_k$ stays stable, as $\Delta(k)$ captures the difference between successive dual variables.
    \item (Policy Mixing) Another key technique of \citet{wei2020online} is \emph{policy mixing}---perturbing the policy before applying the policy optimization update. The motivation behind this technique is to make $\Delta(k)$ small. In particular, a typical bound on $\Delta(k)$ in policy optimization involves KL divergence terms as follows:
    \begin{align*}
        \Delta(k) \leq - c_1 Y_k + c_2 + D(\pi||\widetilde\pi_h^k) - D(\pi||\widehat\pi_h^{k+1})
    \end{align*}
    The key issue is that $D(\pi || \widetilde\pi_h^k)$ can become arbitrarily large when $\widetilde\pi_h^k(a) \approx 0$ for some $a\in\mathcal{A}$, because the KL divergence is unbounded near the simplex boundary. In contrast, when a mixing step is applied, we ensure that $\widetilde\pi_h^k(a) \geq \theta / |\mathcal{A}|$ for all $a$, where $\theta$ denotes the level of mixing. In this case, we can easily show that $D(\pi||\widetilde\pi_h^k) \leq \log(|\mathcal{A}|/\theta)$ (\Cref{lem:KL mixing}). Hence, the mixing step guarantees that the KL term remains bounded, which in turn prevents the dual variable from blowing up.
\end{itemize}

\paragraph{Insufficiency of Na\"ive Extension} While policy mixing is essential to keep the dual variable stable, it becomes problematic in the linear CMDP setting, where the main difficulty arises in the covering number argument. In particular, as described in \Cref{sec:novelty}, we have to derive an upper bound on the covering number of the weighted LogSumExp softmax policy. Establishing such a bound is one of our key challenges and is highly non-trivial. Thus, although adapting \citet{wei2020online} to the framework of \citet{ghosh2022provably} is a natural step toward handling adversarial losses, obtaining sublinear regret and constraint violation bounds becomes unclear.

\section{Additional Discussions}
\paragraph{Intuition on Covering Number}
    We provide an intuition on why the weighted LogSumExp softmax policy yields $\widetilde\calO(n^2 d^2)$, where $n$ denotes the number of mixing steps, while the covering number under greedy or softmax policies is just $\widetilde\calO(d^2)$. Before explaining this, we note that the covering number of a function class depends on (i) how many parameters are required and (ii) how close these parameters must be for the functions to be sufficiently similar. Based on this high-level idea, the number of parameters needed to determine a weighted LogSumExp softmax policy is $\calO(nd^2)$, as each $\sum_{j}\widehat Q^j$ requires $\calO(d^2)$ parameters once the feature contraction is applied. Furthermore, we observe that the impact of parameters decreases exponentially as mixing continues, meaning that the parameters must be chosen increasingly close (see the proof of \Cref{lem:lip pol}). This leads to an additional multiplicative factor $n$, resulting in $\widetilde\calO(n^2 d^2)$.

\paragraph{Choice of Mixing Period} 
To justify the choice of $K^{3/4}$, we first clarify how $\Regret$ and $\Violation$ depend on the mixing period $K^B$. Note that the covering number is closely related to terms (I) and (IV) of the decompositions in \eqref{eq:decomposition}, and the dual variable directly affects term (V). This yields the following simplified regret and constraint violation bounds.
\begin{align*}
    &\Regret = \widetilde\calO(\sqrt{K\log \calN_\epsilon}),\quad \Violation = \widetilde\calO(\sqrt{K\log \calN_\epsilon} + Y_k / \eta).
\end{align*}
We emphasize that $\Violation$ can be bounded by $\widetilde\calO(\sqrt{K^{3-2B}} + K^B)$. This is because the log covering number is bounded by $\widetilde\calO(K^{2-2B}d^2)$, and the dual variable is bounded by $\widetilde\calO(\eta K^B)$. Thus, to minimize the dependency on $K$, we set $B = 3/4$.

\paragraph{Discussion on Lower Bound}
We note that the regret lower bound of $\Omega(\sqrt{H^3 d^2 K})$ also applies to our setting, which is for stochastic linear unconstrained MDPs~\citep{zhou2021nearly, he2022near}. This is because by taking the loss to be fixed across episodes and using a trivial constraint (i.e., taking $b = H$), our problem reduces to a stochastic linear unconstrained MDP. Therefore, we conjecture that there remains room for improving our regret bound by a factor of $\widetilde{\mathcal{O}}({K^{1/4}})$.

Additionally, we outline a promising direction toward achieving $\widetilde{\mathcal{O}}(\sqrt{K})$ regret and violation in our setting. The main challenge is analyzing the constraint violation without relying on mixing steps. In our current analysis, the mixing step is inevitable to mitigate KL divergence terms in the drift upper bound (\Cref{lem:lyap 1}); these KL divergence terms arise from mirror-descent type updates to control adversarial losses. However, mixing becomes problematic for linear CMDPs because it enlarges the covering number. If one can design an approach that controls the dual variable without mixing, then achieving optimal bounds may become possible.




\section{Related Work}\label{appendix:related work}

\paragraph{Online Tabular CMDP}
Starting from the seminal work of \citet{efroni2020exploration}, minimizing regret and constraint violation in online tabular CMDPs has been studied under various settings. Several works~\citep{liu2021learning, bura2022dope, yu2025improved} considered the case of zero constraint violation under the assumption of a known safe policy. Under the same assumption, \citet{muller2023cancellation} studied hard constraint violation---the sum of only positive constraint violations. Without this assumption, the hard constraint violation was studied by \citet{muller2024truly, stradi2025optimal}. Moreover, \citet{wei2022triple} proposed a model-free algorithm for finite-horizon CMDPs, and \citet{wei2022provably, chen2022learning} proposed algorithms for infinite-horizon average-reward CMDPs. However, these works assume stationary environments. To relax this assumption, \citet{qiu2020upper} studied adversarial losses under full-information feedback. For both adversarial losses and costs, \citet{stradi2024online, stradi2025policy} considered full-information feedback and bandit feedback, respectively. More recently, several papers proposed algorithms for adversarial CMDPs with hard constraint violation guarantees. \citet{stradi2025learning} proposed an algorithm for adversarial losses and stochastic costs under bandit feedback, and \citet{zhu2025optimistic} studied stochastic losses and adversarial costs under full-information feedback.




\paragraph{Online Linear CMDP} 
For finite-horizon linear CMDPs, \citet{ghosh2022provably,ghosh2024towards} studied cumulative and hard constraint violations, respectively. Similarly, \citet{ghosh2023achieving} developed several algorithms for the infinite-horizon average-reward setting. \citet{kitamura2025provably} studied achieving zero constraint violation under the assumption of a known safe policy, and \citet{liu2025sample} studied sample complexity under the assumption of a generative model. \citet{wei2023provably} studied non-stationary CMDPs, where components of the environment may change subject to bounded total variation. However, this setting differs from adversarial settings, in which the variation of functions is not assumed to be bounded by some factor. \citet{amani2021safe, wei2024safe, roknilamouki2025provably} investigated hard instantaneous constraints, where unsafe actions must not be taken in each step. We also note additional works that are not linear CMDPs but incorporate linear function approximation. There are several works for linear mixture CMDPs with various settings~\citep{ding2021provably, ding2023provably, shi2023near}. More generally, the $q_\pi$-realizable setting was studied by \citet{tian2024confident}, which only assumes that value functions can be represented as an inner product of a given feature. However, adversarial environments have not been considered in these settings.


\paragraph{Online Adversarial Linear MDP}
Online adversarial linear MDPs have been studied under full-information feedback~\citep{zhong2023theoretical,sherman2024rate,cassel2024warm} and bandit feedback~\citep{neu2021online,luo2021policy,dai2023refined,sherman2023improved,kong2024improved, liu2024towards}. Specifically, in the full-information feedback setting, \citet{zhong2023theoretical} proposed a multi-batched policy optimization algorithm, achieving a $\widetilde\calO(K^{3/4})$ regret bound. \citet{sherman2024rate} achieved a $\widetilde\calO(\sqrt{K})$ regret bound, adopting a warm-up phase to obtain a simple expression for policies. In addition, \citet{cassel2024warm} proposed a warm-up free policy optimization algorithm with an improved regret bound. In the bandit feedback setting, \citet{luo2021policy} introduced the notion of dilated bonus, and \citet{liu2024towards} proposed two algorithms: one achieved a $\widetilde\calO(\sqrt{K})$ regret bound but was computationally inefficient, and the other achieved $\widetilde\calO(K^{3/4})$ and was computationally efficient.

\section{Contracted MDP}\label{appendix:contracted mdp}
In this section, we explain the notion of a contracted MDP~\citep{cassel2024warm}, which is essential for deriving our main results. A contracted MDP is defined by the tuple $\bar\calM = (H,\calS, \calA, \{\bar\bbP_h\}_{h=1}^H$, $\{\bar \ell_h\}_{h=1}^H, s_1, \rho)$. Here, $\rho:\calS \times\calA\times[H]\rightarrow  [0,1]$ specifies the level of contraction. In particular, the loss function and transition kernel are defined as
\begin{align*}
    &\bar \ell_h(s,a) = \rho(s,a,h) \ell_h(s,a),\\
    &\bar \bbP_h(s'\mid s,a) = \rho(s,a,h) \bbP_h(s'\mid s,a).
\end{align*}
Since $\rho(s,a,h)\in [0,1]$, it follows that $\bar \ell_h(s,a) \in [0,1]$, meaning that the contraction preserves the boundedness of the original loss function. On the other hand, $\sum_{s'\in\calS}\bar\bbP_h(s'\mid s,a) \leq  1$, which implies that $\bar\bbP$ defines a sub-probability measure. Although this does not satisfy the definition of a probability measure, it is sufficient for our purposes, as the contracted transition kernel is only used in the analysis. Furthermore, it is often called a sub-MDP because its transition kernel is a sub-probability measure.

Accordingly, $\bar V_{\ell, h}^\pi(s;\rho)$ denotes the $\rho$-contracted value function with respect to a policy $\pi$ and a contracted MDP $\bar\calM$. Given $\bar V_{\ell, H+1}^\pi(s;\rho) = 0$ for all $s\in\calS$, $\bar V_{\ell, h}^\pi(s;\rho)$ is defined recursively as
\begin{align}\label{eq:contracted value function}
    \bar V_{\ell, h}^\pi(s;\rho) = \bbE_{\bar\bbP,\pi}\left[\sum_{j=h}^H \bar\ell_j(s_j,a_j)|s_h=s\right] = \bbE_{a\sim\pi(\cdot\mid s)}\left[\bar\ell_h(s,a) + \sum_{s'\in\calS}\bar\bbP_h(s'\mid s,a) \bar V_{\ell,h+1}^\pi(s';\rho)\right].
\end{align}
We next introduce a lemma that highlights a key property of contracted MDPs, which states a $\rho$-contracted value function is less than or equal to its original value function. 
\begin{lemma} [Lemma 2 of \citet{cassel2024warm}] \label{lem:cassel 2}
    For any $\rho:\calS \times \calA\times[H] \to [0,1]$, $\pi\in\Pi$, $h\in[H]$, $s\in\calS$, and $\ell:\calS\times\calA \times[H] \to[0,1]$, we have $\bar V_{\ell,h}^\pi(s;\rho) \leq V_{\ell,h}^\pi(s)$.
\end{lemma}
Since $\bar\bbP$ is a sub-probability, we note that $\bbE_{\bar\bbP,\pi}$ applied to a constant $c\geq 0$ could be less than $c$ itself. Although this is trivial, we state it formally below for completeness, as this relation is used frequently in our analysis.
\begin{align}\label{eq:contracted expectation}
    \bbE_{\bar\bbP,\pi}[c\mid s_1 = s] \leq \bbE_{\bbP,\pi}[c\mid s_1=s] = c \quad \forall s\in\calS
\end{align}
\begin{proof}[Proof of \eqref{eq:contracted expectation}]
First, we prove $\bbE_{\bar\bbP,\pi}[\sum_{h\in[H]}\ell_h(s,a)|s_1=s] \leq \bbE_{\bbP,\pi}[\sum_{h\in[H]}\ell_h(s,a)|s_1=s]$ for any $s\in\calS$ and $\{\ell_h\}_{h=1}^H$ where $\ell_h:\calS\times\calA \to \bbR_+$, using induction on $h$. For the base case, $h=H$, $\bbE_{\bar\bbP,\pi}[\ell_H(s_H,a_H)|s_H]= \bbE_{\bbP,\pi}[\ell_H(s_H,a_H)|s_H] = \bbE_{a_H \sim \pi(\cdot|s_H)}[\ell_H(s_H,a_H)]$. Assuming the statement is true for $h+1$, we have
\begin{align*}
    \bbE_{\bar\bbP,\pi}\left[\sum_{j=h}^H \ell_j(s_j,a_j)|s_h=s\right] &= \bbE_{a\sim\pi(\cdot|s)}\left[\ell_h(s_h,a_h) + \sum_{s'\in\calS} \bar\bbP_h(s'\mid s_h,a_h)\bbE_{\bar\bbP,\pi}\left[\sum_{j=h+1}^H \ell_j(s_j,a_j)|s'\right]\right]\\
    &\leq \bbE_{a\sim\pi(\cdot|s)}\left[\ell_h(s_h,a_h) + \sum_{s'\in\calS} \bbP_h(s'\mid s_h,a_h)\bbE_{\bbP,\pi}\left[\sum_{j=h+1}^H \ell_j(s_j,a_j)|s'\right]\right]\\
    &= \bbE_{\bbP,\pi}\left[\sum_{j=h}^H \ell_j(s_j,a_j)|s_h=s\right].
\end{align*}
This completes the induction. Furthermore, by taking $\ell_h(s,a) = c/H$ for all $(s,a,h)\in\calS\times\calA\times [H]$, we have $\sum_{h\in[H]}\ell_h(s,a) = c$, and thus $\bbE_{\bar\bbP,\pi}[c|s] \leq \bbE_{\bbP,\pi}[c|s]$ for any $s\in\calS$. Since $\bbP$ is not contracted, we know that $\bbE_{\bbP,\pi}[c|s] = c$. This completes the proof.
\end{proof}
 
The following lemma is an extension of a well-known value difference lemma to incorporate contracted MDPs. 
\begin{lemma}[Lemma 1 of \citet{shani2020optimistic} and Lemma 14 of \cite{cassel2024warm}]\label{lem:extended value diff}
Let $\pi, \widehat\pi$ be two policies, and let $\mathcal{M} = (H, \mathcal{S}, \mathcal{A}, \{\bbP_h\}_{h=1}^H, \{\ell_h\}_{h=1}^H,s_1)$ be a (possibly sub) MDP. 
For all $h\in [H]$, let $\widehat{Q}_{\ell, h}:\calS\times\calA \rightarrow \bbR$ be an arbitrary function, and let
$\widehat{V}_{\ell,h}(s) = \left\langle \widehat{Q}_{\ell,h}(s, \cdot), \widehat\pi_h(\cdot \mid s) \right\rangle$ for all $s\in\calS$.
Then,
\begin{align*}
V_{\ell,1}^{\pi}(s_1) - \widehat{V}_{\ell,1}(s_1)
&=  \bbE_{\bbP,\pi} \left[ \sum_{h=1}^H
\left\langle  \widehat{Q}_{\ell,h}(s_h, \cdot), \pi_h(\cdot \mid s_h) - \widehat\pi_h(\cdot \mid s_h) \right\rangle \,\middle|\, s_1\right] \\
&\quad+ \bbE_{\bbP,\pi} \left[ \sum_{h=1}^H
\ell_h(s_h, a_h) + \bbP_h\widehat{V}_{\ell,h+1}(s_h, a_h) - \widehat{Q}_{\ell,h}(s_h, a_h) 
\,\middle|\, s_1\right],
\end{align*}
where $V_{\ell,1}^\pi$ is the value function of $\pi$, and $\bbP_h\widehat{V}_{\ell,h+1}(s, a) = \sum_{s'\in\calS} \bbP_h(s'\mid s,a)\widehat V_{\ell,h+1}(s')$.
\end{lemma}
Since we assume that the initial state $s_1$ is fixed, we omit it when clear from the context for simplicity.

\section{Parameterizations and Function Classes}\label{appendix:parameterization}
In this section, we introduce the parameterizations of $\widehat Q$ and $\widehat \pi$. Following this, we define the function classes to which the value function estimates and policies belong. 

\paragraph{Parameterization}
For the parameterization of $\widehat Q$, given $\beta\in\bbR, w \in \bbR^d, \Lambda\in \bbR^{d\times d}$, we define $\widehat Q(\cdot,\cdot; \beta,w,\Lambda)$ as
\begin{align}\label{eq:param Q}
    \widehat Q(\cdot,\cdot; \beta,w,\Lambda) = \left(\phi(\cdot,\cdot)^\top w - \beta \|\phi(\cdot,\cdot)\|_\Lambda\right) \cdot \sigma(-\beta_w\|\phi(\cdot,\cdot)\|_\Lambda + \log K).
\end{align}
Next, we consider the parameterization of $\widehat\pi$. Let $n \in \bbZ_+$ denote the number of mixing steps. 

For the parameterization of $\widehat\pi$, given $n, \{\beta_i, w_i\}_{i=0}^{n}, \Lambda$ and a mixing parameter $\theta \in (0,1)$, we first generate policies $\{\widehat\pi_i\}_{i=0}^{n+1}$ recursively, and define the final policy as $\widehat \pi(\cdot|\cdot; \{\beta_i, w_i\}_{i=0}^{n}, \Lambda)$:
\begin{align}\label{eq:param pi}
\begin{aligned}
    \text{Generate $\widehat\pi_i$}: \quad 
    &\widehat \pi_0(\cdot\mid s) = \piunif(\cdot\mid s),\\
    &\widetilde\pi_i (\cdot\mid s) = (1-\theta)\widehat\pi_i(\cdot\mid s) + \theta \piunif(\cdot\mid s) \quad i=0,\ldots,n,\\
    &\widehat\pi_{i+1}(\cdot\mid s) \propto \widetilde\pi_i(\cdot\mid s) \exp\left(\widehat Q(s,\cdot; \beta_{i},w_{i},  \Lambda)\right) \quad i=0,\ldots,n,\\
    \text{Define}: \quad 
    &\widehat \pi(\cdot\mid s; \{\beta_i, w_i\}_{i=0}^{n}, \Lambda) = \widehat\pi_{n+1}(\cdot\mid s).
\end{aligned}
\end{align}
We keep the policy parameterization in its recursive form for the following reason. Although we can easily show that $\widehat\pi(\cdot | \cdot; \{\beta_i, w_i\}_{i=0}^{n},\Lambda)$ follows the weighted LSE softmax, i.e., $\sum_{i}\zeta_i \exp(\sum_{j}\widehat Q_j)$, the weight parameters $\zeta_i$ depend on $\{\beta_i,w_i\}_{i=0}^{n}$ and $\Lambda$, which makes analyzing this form difficult. Thus, obtaining the closed form of $\widehat\pi(\cdot|\cdot;\{\beta_i,w_i\}_{i=0}^n, \Lambda)$ is intractable, as specifying exact $\zeta_i$ is difficult.

Since $\widehat\pi_{n+1}(\cdot\mid s)$ induces a probability distribution over $\calA$ for each $s\in\calS$, \eqref{eq:param pi} indeed defines a valid policy. Furthermore, the following lemma shows that $Q$-function estimates and policies generated by \Cref{alg:main-linear} can be parameterized using \eqref{eq:param Q}, \eqref{eq:param pi}.
\begin{lemma}\label{lem:pol param}
    For any $e\in E$, consider $k\in K_e$. For some $n\geq 0$, let $k_e + nK^B$ be the last index that the mixing is applied before episode $k$, i.e., $n=\max\{0, \floor{(k-1-k_e)/K^B}\}$. Let $\widehat\pi_h^k, \{\widehat Q_{f,h}^j,\widehat Q_{g,h}^j\}_{j=k_e}^{k-1}$ be the policy and $Q$-function estimates generated by \Cref{alg:main-linear}, respectively. Let $S_i$ be the index set defined as
    \begin{align}\label{eq:Si}
        S_i &= \begin{cases}
            \{k_e + iK^B, \ldots, k_e+(i+1)K^B - 1\} & \text{for $i=0,\ldots, n-1$},\\
            \{k_e + nK^B, \ldots, k-1\} & \text{for $i=n$}.
        \end{cases}
    \end{align}
    Then there exists $\{w_{f,h}^j, w_{g,h}^j\}_{j=k_e}^{k-1}, \Lambda$ such that for $j=k_e,\ldots, k-1$,
    \begin{align*}
        \widehat Q_{f,h}^j(\cdot,\cdot) = \widehat Q(\cdot,\cdot; \beta_b, w_{f,h}^j, \Lambda),\quad \widehat Q_{g,h}^j(\cdot,\cdot) = \widehat Q(\cdot,\cdot; \beta_b, w_{g,h}^j, \Lambda).
    \end{align*}
    Furthermore, we have
    \begin{align*}
        \widehat\pi_h^k(\cdot|\cdot) = \widehat\pi(\cdot\mid \cdot; \{\beta_i,w_i\}_{i=0}^n, \Lambda)
    \end{align*}
    where $\beta_i = -\alpha\beta_b\sum_{j\in S_i} (1+Y_j), w_i = -\alpha\sum_{j\in S_i} (w_{f,h}^j + Y_j w_{g,h}^j)$ for $i=0,\ldots,n$, and $Y_j$ is the dual variable in episode $j$.
\end{lemma}
\begin{proof}
    Note that for any $j$, by algorithm
    \begin{align*}
        \widehat Q_{f,h}^j(\cdot,\cdot) = \left(\phi(\cdot,\cdot)^\top \left[\widehat\theta_{f,h}^j + \widehat\psi_{h}^j\widehat V_{f,h+1}^j\right] - \beta_b \|\phi(\cdot,\cdot)\|_{(\Lambda_h^{k_e})^{-1}}\right)\sigma(-\beta_w \|\phi(\cdot,\cdot)\|_{(\Lambda_h^{k_e})^{-1}} + \log K).
    \end{align*}
    Then we can take $w_{f,h}^j = \widehat\theta_{f,h}^j + \widehat\psi_{h}^j\widehat V_{f,h+1}^j$ and $\Lambda = (\Lambda_h^{k_e})^{-1}$. Here, it is clear that $(2K)^{-1}I\preceq (\Lambda_h^{k_e})^{-1}\preceq I$. Also, we can apply the same argument to $\widehat Q_{g,h}^j$. Then the first statement is proved. 
    
    Let us prove the second statement. By the definition of $n$, the mixing is not applied from episode $k_e + nK^B + 1$ to $k-1$. Then we have
    \begin{align*}
        \widehat\pi_h^k(\cdot\mid s) \propto \widetilde\pi_h^{k_e+nK^B}(\cdot\mid s)\exp\left(-\alpha \sum_{j\in S_n} (\widehat Q_{f,h}^j(s,\cdot) + Y_j\widehat Q_{g,h}^j(s,\cdot) )\right).
    \end{align*}
    Furthermore, since $(k_e + nK^B) - k_e \equiv 0 \mod K^B$, we know that $\widetilde\pi_h^{k_e+nK^B}$ is mixed. Then we have
    \begin{align*}
        \widetilde\pi_h^{k_e+nK^B}(\cdot\mid s) = (1-\theta)\widehat\pi_h^{k_e+nK^B}(\cdot\mid s) + \theta \piunif(\cdot\mid s).
    \end{align*}
    Similarly, we can deduce $\widehat\pi_h^{k_e+nK^B}$ using the fact that the mixing is not applied from episode $k_e + (n-1)K^B$ to $k_e + nK^B - 1$.
    \begin{align*}
        \widehat\pi_h^{k_e + nK^B}(\cdot\mid s) \propto \widetilde\pi_h^{k_e + (n-1)K^B}(\cdot\mid s)\exp\left(-\alpha \sum_{j\in S_{n-1}}(\widehat Q_{f,h}^j(s,\cdot) + Y_j\widehat Q_{g,h}^j(s,\cdot) )\right).
    \end{align*}
    We repeatedly apply these steps until episode $k_e$. Then we have for $i=0,\ldots,n-1$,
    \begin{align*}
        &\widetilde\pi_h^{k_e + iK^B}(\cdot\mid s) = (1-\theta)\widehat\pi_h^{k_e+iK^B}(\cdot\mid s) + \theta\piunif(\cdot\mid s), \\
        &\widehat\pi_h^{k_e + (i+1)K^B}(\cdot\mid s) \propto \widetilde\pi_h^{k_e + iK^B}(\cdot\mid s)\exp\left(-\alpha \sum_{j\in S_i}(\widehat Q_{f,h}^j(s,\cdot) + Y_j\widehat Q_{g,h}^j(s,\cdot) )\right).
    \end{align*}
    Note that $\widehat\pi_h^{k_e}=\piunif$. Then we have $\widehat\pi_h^k = \widehat\pi_{n+1}$, where $\widehat\pi_{n+1}$ is recursively defined as
    \begin{align*}
        &\widehat\pi_0(\cdot\mid s) = \piunif(\cdot\mid s)\\
        &\widetilde\pi_i(\cdot\mid s) = (1-\theta)\widehat\pi_i(\cdot\mid s) + \theta\piunif(\cdot\mid s) \quad \forall i=0,\ldots,n\\
        &\widehat\pi_{i+1}(\cdot\mid s) \propto \widetilde\pi_i(\cdot\mid s)\exp\left(-\alpha \sum_{j\in S_i } (\widehat Q_{f,h}^j(s,\cdot) + Y_j \widehat Q_{g,h}^j(s,\cdot))\right) \quad \forall i=0,\ldots,n.
    \end{align*}
    Note that 
    \begin{align*}
        &-\alpha \sum_{j\in S_i } (\widehat Q_{f,h}^j(s,a) + Y_j \widehat Q_{g,h}^j(s,a))\\ 
        &=-\alpha \sum_{j\in S_i} \left(\phi(s,a)^\top(w_{f,h}^j + Y_j w_{g,h}^j) - \beta_b(1+Y_j)\|\phi(s,a)\|_{\Lambda}\right)\sigma(-\beta_w \|\phi(s,a)\|_{\Lambda} + \log K) \\
        &= (\phi(s,a)^\top w_i - \beta_i\|\phi(s,a)\|_{\Lambda})\sigma(-\beta_w\|\phi(s,a)\|_{\Lambda} + \log K)\\
        &=\widehat Q(s,a;\beta_i, w_i, \Lambda)
    \end{align*}
    where $\beta_i = -\alpha\beta_b\sum_{j\in S_i} (1+Y_j), w_i = -\alpha\sum_{j\in S_i} (w_{f,h}^j + Y_j w_{g,h}^j)$. This completes the proof for the second statement.
\end{proof}

\paragraph{Function Class}
Now, we define the function classes as follows. Given some boundedness constants $C_\beta, C_w, C_Q \geq 0$,
\begin{align*}
    &\widehat\calQ(C_\beta, C_w, C_Q) = \left\{\widehat Q(\cdot, \cdot; \beta, w, \Lambda): |\beta|\leq C_\beta, \|w\|_2\leq C_w, (2K)^{-1}I\preceq \Lambda \preceq I, \|\widehat Q(\cdot,\cdot; \beta,w, \Lambda)\|_\infty \leq C_Q\right\}.
\end{align*}
Unlike $\widehat\calQ$, the class of policies has to be defined based on the number of mixing steps, since the formulation is determined by this number. Let $\widehat \Pi_n$ denote the set of policies that involve exactly $n$ mixing operations. Furthermore, we consider a boundedness constant $C_Y$ to incorporate the scale of the dual variable. Given $n\in\bbZ_+$ and some boundedness constants $C_\beta, C_w, C_Q, C_Y \geq 0$,
\begin{align*}
    &\widehat \Pi_n(C_\beta, C_w, C_Q, C_Y) \\
    &\quad= \left\{
        \widehat \pi(\cdot\mid s; \{\beta_i, w_i\}_{i=0}^{n}, \Lambda):
        \begin{aligned}
        & |\beta_i| \leq (1+C_Y)C_\beta, \
          \|w_i\|_2 \leq (1+C_Y)C_w, \\
        & (2K)^{-1}I \preceq \Lambda \preceq I, \
          \|\widehat Q(\cdot,\cdot; \beta_i,w_i, \Lambda)\|_\infty 
           \leq (1+C_Y)C_Q,
        \end{aligned}\ i=0,\ldots,n
        \right\}.
\end{align*}
Similar to $\widehat \Pi_n$, since $\widehat V$ is defined by $\widehat Q$ and $\widehat \pi$, we define the function class of $\widehat V$ for each $n$.
\begin{align*}
    &\widehat \calV_n(C_\beta, C_w, C_Q, C_Y)= \left\{\widehat V(\cdot): \widehat V(\cdot) = \sum_{a\in\calA}\widehat\pi(a\mid \cdot) \widehat Q(\cdot,a),\quad
    \begin{aligned}
        &\widehat Q \in \widehat\calQ(C_\beta,C_w, C_Q),\\
        &\widehat\pi \in \widehat \Pi_n(KC_\beta, K C_w, KC_Q, C_Y)
    \end{aligned} \right\}.
\end{align*}
Note that if we apply the policy mixing every $K^B$ episodes, then the number of mixing steps is at most $K^{1-B} = K^L$, where $L=1-B$. Thus, we define $\widehat \calV(C_\beta, C_w, C_Q, C_Y), \widehat \Pi(C_\beta, C_w, C_Q, C_Y)$ as the unions over $n = 0,\ldots, K^L$ as follows.
\begin{align*}
    &\widehat \calV(C_\beta, C_w, C_Q, C_Y) = \bigcup_{n=0}^{K^L}\widehat \calV_n(C_\beta, C_w, C_Q, C_Y),\\
    &\widehat \Pi(C_\beta, C_w, C_Q, C_Y) = \bigcup_{n=0}^{K^L} \widehat \Pi_n(C_\beta, C_w, C_Q, C_Y).
\end{align*}

\section{Covering Number}\label{appendix:covering number}
In this section, we show an upper bound on the covering number of  $\widehat\calV(C_\beta, C_w, C_Q, C_Y)$, which is crucial to analyze linear CMDPs. As a first step, we show that any $\widehat Q \in \widehat\calQ(C_\beta, C_w, C_Q)$ is Lipschitz, i.e., the $\ell_\infty$-norm between $Q$-function estimates is bounded by the $\ell_2$-norm between their parameters. We closely follow the proof of Lemma 10 of \citet{cassel2024warm}. 
\begin{lemma}[Lipschitz $\widehat Q$] \label{lem:lip Q}
Let $1\leq \beta_w, C_w, C_\beta$. For any $\widehat Q(\cdot,\cdot;\beta^1, w^1, \Lambda^1),\widehat Q(\cdot,\cdot;\beta^2, w^2, \Lambda^2) \in \widehat\calQ(C_\beta,C_w, C_Q)$, we have
\[
    \left\|\widehat Q(\cdot,\cdot; \beta^1,w^1,\Lambda^1) - \widehat Q(\cdot,\cdot;\beta^2,w^2,\Lambda^2)\right\|_\infty \leq 4\sqrt{K}\beta_w\max\{C_w, C_\beta\} \left\|(\beta^1,w^1,\Lambda^1) - (\beta^2, w^2, \Lambda^2)\right\|_2
\]
where $\left\|(\beta^1,w^1,\Lambda^1) - (\beta^2, w^2, \Lambda^2)\right\|_2$ is defined in \eqref{eq:lem:lip Q:l2}.
\end{lemma}
\begin{proof}
    Consider 
    \begin{align*}
        &|\widehat Q(s,a; \beta^1,w^1, \Lambda^1) - \widehat Q(s,a; \beta^2,w^2, \Lambda^2)| \\
        &\leq \underbrace{|\widehat Q(s,a; \beta^1,w^1, \Lambda^1) - \widehat Q(s,a; \beta^2, w^1, \Lambda^1)|}_{\text{(I)}}
         + \underbrace{|\widehat Q(s,a; \beta^2,w^1, \Lambda^1) - \widehat Q(s,a; \beta^2,w^2, \Lambda^1)|}_{\text{(II)}}\\
        &\quad+ \underbrace{|\widehat Q(s,a; \beta^2,w^2, \Lambda^1) - \widehat Q(s,a; \beta^2,w^2, \Lambda^2)|}_{\text{(III)}}.
    \end{align*}
    We bound each term individually. Note that $(2K)^{-1}I\preceq\Lambda^1 \preceq I$. For (I), since $\|\phi(s,a)\|_{\Lambda^1}\leq \|(\Lambda^1)^{1/2}\|_2 \|\phi(s,a)\|_2\leq 1$ and $|\sigma(z)|\leq 1$ for any $z\in\bbR$,
    \begin{align*}
        \text{(I)}
        &= |\beta^1 - \beta^2| \cdot \|\phi(s,a)\|_{\Lambda^1} \sigma(-\beta_w \|\phi(s,a)\|_{\Lambda^1} + \log K)\\
        &\leq |\beta^1 - \beta^2|.
    \end{align*}
    For (II), by the Cauchy-Schwarz inequality,
    \begin{align*}
        \text{(II)}
        &= |\phi(s,a)^\top(w^1-w^2)| \sigma(-\beta_w\|\phi(s,a)\|_{\Lambda^1} + \log K)\\
        &\leq \|w^1 - w^2\|_2
    \end{align*}
    For (III), by the triangle inequality,
    \begin{align*}
        \text{(III)}&=|\widehat Q(s,a; \beta^2,w^2, \Lambda^1) - \widehat Q(s,a; \beta^2,w^2, \Lambda^2)| \\
        &\leq |\phi(s,a)^\top w^2|\cdot \left|\sigma(-\beta_w\|\phi(s,a)\|_{\Lambda^1} + \log K) - \sigma(-\beta_w\|\phi(s,a)\|_{\Lambda^2} + \log K)\right| \\
        &\quad + \beta^2 \left|\|\phi(s,a)\|_{\Lambda^1} - \|\phi(s,a)\|_{\Lambda^2}\right| \cdot \sigma(-\beta_w\|\phi(s,a)\|_{\Lambda^1} + \log K) \\
        &\quad + \beta^2 \|\phi(s,a)\|_{\Lambda^2}\cdot \left|\sigma(-\beta_w\|\phi(s,a)\|_{\Lambda^1} + \log K) - \sigma(-\beta_w\|\phi(s,a)\|_{\Lambda^2} + \log K)\right|.
    \end{align*}
    Note that the sigmoid function is $1$-Lipschitz on $\bbR$, and the triangle inequality implies that $\left|\|\phi(s,a)\|_{\Lambda^1} - \|\phi(s,a)\|_{\Lambda^2}\right| \leq \|(\Lambda^1)^{1/2} - (\Lambda^2)^{1/2}\|_2$.
    Then we can deduce that
    \begin{align*}
        &|\widehat Q(s,a; \beta^2,w^2, \Lambda^1) - \widehat Q(s,a; \beta^2,w^2, \Lambda^2)| \\
        &\leq C_w\beta_w {\|(\Lambda^1)^{1/2} - (\Lambda^2)^{1/2}\|_2} + C_\beta \|(\Lambda^1)^{1/2}- (\Lambda^2)^{1/2}\|_2 + C_\beta \beta_w {\|(\Lambda^1)^{1/2} - (\Lambda^2)^{1/2}\|_2}\\
        &\leq 3\max\{C_w \beta_w, C_\beta, C_\beta\beta_w\}{\|(\Lambda^1)^{1/2} - (\Lambda^2)^{1/2}\|_2}\\
        &\leq 3\max\{C_w\beta_w, C_\beta, C_\beta\beta_w\} \cdot \frac{1}{2\sqrt{1/(2K)}} {\|\Lambda^1 - \Lambda^2\|_2}\\
        &\leq (3/\sqrt{2})\sqrt{K}\beta_w\max\{C_w, C_\beta\} \|\Lambda^1 - \Lambda^2\|_2
    \end{align*}
    where the third inequality is due to \Cref{lem:cassel 17}, and the last inequality is because we assumed that $\beta_w \geq 1$. Note that we know $\|\Lambda^1 - \Lambda^2\|_2\leq \|\Lambda^1 - \Lambda^2\|_F$ where $\|\cdot\|_F$ denotes the Frobenius norm. Finally, we show that for any $(s,a) \in \calS \times \calA$,
    \begin{align*}
        &|\widehat Q(s,a; \beta^1,w^1, \Lambda^1) - \widehat Q(s,a; \beta^2,w^2, \Lambda^2)|\\
        &\leq \|w^1 -w^2\|_2 + |\beta^1 - \beta^2| + (3/\sqrt{2})\sqrt{K}\beta_w\max\{C_w, C_\beta\} \|\Lambda^1 - \Lambda^2\|_F\\
        &\leq \sqrt{3\left(\|w^1 -w^2\|_2^2 + |\beta^1 - \beta^2|^2 + ((3/\sqrt{2})\sqrt{K}\beta_w\max\{C_w, C_\beta\})^2 \|\Lambda^1 - \Lambda^2\|_F^2\right)}\\
        &\leq 4\sqrt{K}\beta_w\max\{C_w, C_\beta\}\left\|(\beta^1,w^1,\Lambda^1) - (\beta^2, w^2, \Lambda^2)\right\|_2
    \end{align*}
    where the second inequality follows from the Cauchy-Schwarz inequality, and the last inequality is due to $1 \leq (3/\sqrt{2})\sqrt{K}\beta_w\max\{C_w,C_\beta\}$, as we assumed that $1\leq \beta_w, C_w, C_\beta$. Here, $\left\|(\beta^1,w^1,\Lambda^1) - (\beta^2, w^2, \Lambda^2)\right\|_2$ denotes
    \begin{align}\label{eq:lem:lip Q:l2}
        \left\|(\beta^1,w^1,\Lambda^1) - (\beta^2, w^2, \Lambda^2)\right\|_2 = \sqrt{|\beta^1 - \beta^2|^2 + \|w^1 - w^2\|_2^2 + \|\Lambda^1 - \Lambda^2\|_F^2}.
    \end{align}
\end{proof}

Now, we have to show that the policy parameterization given in \eqref{eq:param pi} satisfies a Lipschitz property, i.e., $\ell_1$-norm between any two policies is bounded by the $\ell_2$-norm between their parameters. 
\begin{lemma}[Lipschitz $\widehat\pi$]\label{lem:lip pol}
    Let $1\leq \beta_w, C_w, C_\beta$ and $0\leq C_Y$. Suppose that two policies $\widehat\pi^1,\widehat\pi^2 \in \widehat\Pi_n(C_\beta, C_w, C_Q, C_Y)$ are parameterized by $\{\beta_i^1, w_i^1\}_{i=0}^{n}, \Lambda^1$ and $\{\beta_i^2, w_i^2\}_{i=0}^{n}, \Lambda^2$, respectively. Then the following holds for any $s\in\calS$.
    \begin{align*}
        &\|\widehat \pi^1(\cdot\mid s) - \widehat\pi^2(\cdot\mid s)\|_1 \\
        &\leq 32\sqrt{(n+1)K}\beta_w(1+C_Y)\max\{C_w,C_\beta\}\left(\frac{8|\calA|}{\theta}\right)^n \sqrt{\sum_{i=0}^n \|(\beta_{i}^1,w_{i}^1,\Lambda^1) - (\beta_{i}^2, w_{i}^2, \Lambda^2)\|_2^2}.
    \end{align*}
\end{lemma}
\begin{proof}
    Fix $s\in\calS$. Let $\{\widehat\pi_i^1, \widetilde\pi_i^1\}_{i=0}^{n+1}, \{\widehat\pi_i^2,\widetilde\pi_i^2\}_{i=0}^{n+1}$ be the sequences of policies recursively generated by \eqref{eq:param pi} to define $\widehat\pi^1, \widehat\pi^2$, respectively. Then it follows that
    \begin{align*}
        \widehat\pi^1(\cdot\mid s) =\widehat\pi_{n+1}^1(\cdot\mid s) \propto \widetilde\pi_{n}^1(\cdot\mid s)\exp( \widehat Q(s,\cdot; \beta_{n}^1, w_{n}^1, \Lambda^1)),\\
        \widehat\pi^2(\cdot\mid s) = \widehat\pi_{n+1}^2(\cdot\mid s) \propto \widetilde\pi_{n}^2(\cdot\mid s)\exp( \widehat Q(s,\cdot; \beta_{n}^2, w_{n}^2, \Lambda^2)).
    \end{align*}
    Note that $\widetilde\pi_n^1(a\mid s), \widetilde\pi_n^2(a\mid s) > 0$ for all $a\in\calA$, since they are perturbed. Then we can define $\log\widetilde\pi_n^1(a\mid s), \log\widetilde\pi_n^2(a\mid s)$, and it leads to
    \begin{align}\label{eq:lem:lip pol:1}
    \begin{aligned}
        \widehat\pi_{n+1}^1(\cdot\mid s) \propto \exp\left(\log\widetilde\pi_{n}^1(\cdot\mid s) + \widehat Q(s,\cdot; \beta_{n}^1, w_{n}^1, \Lambda^1)\right),\\
        \widehat\pi_{n+1}^2(\cdot\mid s) \propto \exp\left(\log\widetilde\pi_{n}^2(\cdot\mid s) + \widehat Q(s,\cdot; \beta_{n}^2, w_{n}^2, \Lambda^2)\right).
    \end{aligned}
    \end{align}
    By \Cref{lem:softmax lipschitz},
    \begin{align*}
        &\|\widehat\pi_{n+1}^1(\cdot\mid s) - \widehat\pi_{n+1}^2(\cdot\mid s)\|_1 \\
        &\leq 8\left\|\log\widetilde\pi_{n}^1(\cdot\mid s) + \widehat Q(s,\cdot; \beta_{n}^1, w_{n}^1, \Lambda^1)- \log\widetilde\pi_{n}^2(\cdot\mid s)- \widehat Q(s,\cdot; \beta_{n}^2, w_{n}^2, \Lambda^2)\right\|_\infty\\
        &\leq 8\left\|\log\widetilde\pi_{n}^1(\cdot\mid s)- \log\widetilde\pi_{n}^2(\cdot\mid s)\right\|_\infty
        + 8\left\|\widehat Q(s,\cdot; \beta_{n}^1, w_{n}^1, \Lambda^1)- \widehat Q(s,\cdot; \beta_{n}^2, w_{n}^2, \Lambda^2)\right\|_\infty.
    \end{align*}
    Note that $\widetilde\pi_n^1(a\mid s), \widetilde\pi_n^2(a\mid s) \geq \theta / |\calA|$ for all $a\in\calA$ due to the definition. Then we can utilize the Lipschitzness of $\log$ function in $\left[\theta/|\calA|, \infty\right)^{|\calA|}$. Thus, by \Cref{lem:lip log}, 
    \begin{align}\label{eq:lem:lip pol:2}
    \begin{aligned}
        \left\|\log\widetilde\pi_{n}^1(\cdot\mid s)- \log\widetilde\pi_{n}^2(\cdot\mid s)\right\|_\infty 
        &\leq \frac{|\calA|}{\theta}\|\widetilde\pi_n^1(\cdot\mid s) - \widetilde\pi_n^2(\cdot\mid s)\|_1\\
        &= \frac{|\calA|}{\theta}(1-\theta)\|\widehat\pi_n^1(\cdot\mid s) - \widehat\pi_n^2(\cdot\mid s)\|_1\\
        &\leq \frac{|\calA|}{\theta}\|\widehat\pi_n^1(\cdot\mid s) - \widehat\pi_n^2(\cdot\mid s)\|_1\\
    \end{aligned}
    \end{align}
    where the equality is due to $\widetilde\pi_n^1(\cdot\mid s) = (1-\theta)\widehat\pi_n^1(\cdot\mid s) + \theta\piunif(\cdot\mid s)$ and $\widetilde\pi_n^2(\cdot\mid s) = (1-\theta)\widehat\pi_n^2(\cdot\mid s) + \theta\piunif(\cdot\mid s)$. Plugging \eqref{eq:lem:lip pol:2} into \eqref{eq:lem:lip pol:1}, we have a recursive relation, and it leads to
    \begin{align*}
        &\|\widehat\pi_{n+1}^1(\cdot\mid s) - \widehat\pi_{n+1}^2(\cdot\mid s)\|_1 \\
        &\leq \frac{8|\calA|}{\theta}\left\|\widehat\pi_{n}^1(\cdot\mid s)- \widehat\pi_{n}^2(\cdot\mid s)\right\|_1
        + 8\left\|\widehat Q(s,\cdot; \beta_{n}^1, w_{n}^1, \Lambda^1)- \widehat Q(s,\cdot; \beta_{n}^2, w_{n}^2, \Lambda^2)\right\|_\infty\\
        &\leq \left(\frac{8|\calA|}{\theta}\right)^2\left\|\widehat\pi_{n-1}^1(\cdot\mid s)- \widehat\pi_{n-1}^2(\cdot\mid s)\right\|_1 \\
        &\quad+ 8 \sum_{i=0}^1 \left(\frac{8|\calA|}{\theta}\right)^i \left\|\widehat Q(s,\cdot; \beta_{n-i}^1, w_{n-i}^1, \Lambda^1)- \widehat Q(s,\cdot; \beta_{n-i}^2, w_{n-i}^2, \Lambda^2)\right\|_\infty\\
        &\qquad \vdots\\
        &\leq \left(\frac{8|\calA|}{\theta}\right)^{n+1}\left\|\widehat\pi_{0}^1(\cdot\mid s)- \widehat\pi_{0}^2(\cdot\mid s)\right\|_1 \\
        &\quad+ 8 \sum_{i=0}^n \left(\frac{8|\calA|}{\theta}\right)^i \left\|\widehat Q(s,\cdot; \beta_{n-i}^1, w_{n-i}^1, \Lambda^1)- \widehat Q(s,\cdot; \beta_{n-i}^2, w_{n-i}^2, \Lambda^2)\right\|_\infty\\
        &= 8 \sum_{i=0}^n \left(\frac{8|\calA|}{\theta}\right)^i \left\|\widehat Q(s,\cdot; \beta_{n-i}^1, w_{n-i}^1, \Lambda^1)- \widehat Q(s,\cdot; \beta_{n-i}^2, w_{n-i}^2, \Lambda^2)\right\|_\infty
    \end{align*}
    where the equality is due to $\widehat\pi_0^1(\cdot\mid s) = \widehat\pi_0^2(\cdot\mid s) = \piunif(\cdot\mid s)$. Furthermore, by the Cauchy-Schwarz inequality, 
    \begin{align*}
        &\|\widehat\pi_{n+1}^1(\cdot\mid s) - \widehat\pi_{n+1}^2(\cdot\mid s)\|_1 \\
        &\leq 8 \sqrt{\sum_{i=0}^n \left(\frac{8|\calA|}{\theta}\right)^{2i}}\sqrt{\sum_{i=0}^n \left\|\widehat Q(s,\cdot; \beta_{i}^1, w_{i}^1, \Lambda^1)- \widehat Q(s,\cdot; \beta_{i}^2, w_{i}^2, \Lambda^2)\right\|_\infty^2}\\
        &\leq 8 \sqrt{(n+1)\left(\frac{8|\calA|}{\theta}\right)^{2n}}\sqrt{\sum_{i=0}^n \left\|\widehat Q(s,\cdot; \beta_{i}^1, w_{i}^1, \Lambda^1)- \widehat Q(s,\cdot; \beta_{i}^2, w_{i}^2, \Lambda^2)\right\|_\infty^2}.
    \end{align*}
    Note that $\widehat Q(s,\cdot; \beta_{i}^1, w_{i}^1, \Lambda^1), \widehat Q(s,\cdot; \beta_{i}^2, w_{i}^2, \Lambda^2) \in \widehat\calQ((1+C_Y)C_\beta, (1+C_Y)C_w, (1+C_Y)C_Q)$. By the Lipschitzness of $\widehat Q$ (\Cref{lem:lip Q}),
    \begin{align*}
        &\|\widehat\pi_{n+1}^1(\cdot\mid s) - \widehat\pi_{n+1}^2(\cdot\mid s)\|_1 \\
        &\leq 8 \sqrt{(n+1)\left(\frac{8|\calA|}{\theta}\right)^{2n}} \sqrt{\sum_{i=0}^n \left(4\sqrt{K}\beta_w (1+C_Y)\max\{C_w,C_\beta\}\right)^2\|(\beta_{i}^1,w_{i}^1,\Lambda^1) - (\beta_{i}^2, w_{i}^2, \Lambda^2)\|_2^2}\\
        &=32\sqrt{(n+1)K}\beta_w(1+C_Y)\max\{C_w,C_\beta\}\left(\frac{8|\calA|}{\theta}\right)^n \sqrt{\sum_{i=0}^n \|(\beta_{i}^1,w_{i}^1,\Lambda^1) - (\beta_{i}^2, w_{i}^2, \Lambda^2)\|_2^2}
    \end{align*} as desired.
\end{proof}

Based on the Lipschitz properties that we have shown, we show a Lipschitz property of $\widehat V$, and it leads to an upper bound on the covering number of $\widehat \calV_n(C_\beta, C_w, C_Q, C_Y)$.
\begin{lemma} \label{lem:cov Vn}
Let $1\leq \beta_w, C_\beta, C_w, C_Q$ and $0\leq C_Y$. Given $\epsilon>0$, let $\calN_{\epsilon}(\widehat \calV_n(C_\beta, C_w, C_Q, C_Y))$ denote the $\epsilon$-covering number of $\widehat \calV_n(C_\beta, C_w, C_Q, C_Y)$ with respect to the $\ell_\infty$-norm. Then we have
\begin{align*}
        &\log\calN_{\epsilon}(\widehat \calV_n(C_\beta, C_w, C_Q, C_Y))\leq 3(n+2)^2d^2\log ((8|\calA|/\theta)(1+2C_1C_2/\epsilon))
    \end{align*}
    where
    \begin{align}\label{eq:C1 C2}
    \begin{aligned}
        &C_1 = 33\sqrt{(n+1)K^3}\beta_w\max\{C_w,C_\beta\} C_Q(1+C_Y),\\
        &C_2 = (n+2)(1+C_Y)K(C_\beta + C_w) + (n+2)\sqrt{d}.
    \end{aligned}
    \end{align}
\end{lemma}
\begin{proof}
    Recall that $\widehat V\in \widehat\calV_n(C_\beta, C_w, C_Q, C_Y)$ can be expressed as $\widehat V(\cdot) = \sum_{a\in\calA} \widehat\pi(a\mid \cdot) \widehat Q(\cdot,a)$, where $\widehat Q \in \widehat\calQ(C_\beta, C_w, C_Q),\ \widehat\pi \in \widehat \Pi_n(KC_\beta, K C_w, KC_Q, C_Y)$.
    Then consider $\widehat V_1, \widehat V_2 \in \widehat\calV_n(C_\beta, C_w, C_Q, C_Y)$ such that $\widehat V^1(\cdot) = \sum_{a\in\calA} \widehat\pi^1(a\mid \cdot) \widehat Q^1(\cdot,a)$ and $\widehat V^2(\cdot) = \sum_{a\in\calA} \widehat\pi^2(a\mid \cdot) \widehat Q^2(\cdot,a)$. Suppose that each of those are parameterized as follows.
    \begin{align*}
        &\widehat \pi^1 = \widehat \pi(\cdot|\cdot; \{\beta_i^{\pi,1}, w_i^{\pi,1}\}_{i=0}^{n}, \Lambda^{\pi,1})\in \widehat \Pi_n(KC_\beta, K C_w, KC_Q, C_Y),\\
        &\widehat \pi^2 = \widehat \pi(\cdot|\cdot; \{\beta_i^{\pi,2}, w_i^{\pi,2}\}_{i=0}^{n}, \Lambda^{\pi,2})\in \widehat \Pi_n(KC_\beta, K C_w, KC_Q, C_Y),\\
        &\widehat Q^1 = \widehat Q(\cdot,\cdot; \beta^{Q,1}, w^{Q,1}, \Lambda^{Q,1})\in \widehat\calQ(C_\beta, C_w, C_Q),\\
        &\widehat Q^2 = \widehat Q(\cdot,\cdot; \beta^{Q,2}, w^{Q,2}, \Lambda^{Q,2})\in \widehat\calQ(C_\beta, C_w, C_Q).
    \end{align*}
    Then for any $s\in\calS$,
    \begin{align*}
        &\left|\widehat V^1(s) - \widehat V^2(s)\right| \\
        &= \left|\sum_{a\in\calA} \widehat\pi^1(a\mid s) \widehat Q^1(s,a) - \sum_{a\in\calA} \widehat\pi^2(a\mid s) \widehat Q^2(s,a)\right|\\
        &\leq \left|\sum_{a\in\calA} \widehat\pi^1(a\mid s) \widehat Q^1(s,a) - \sum_{a\in\calA} \widehat\pi^1(a\mid s) \widehat Q^2(s,a)\right| + \left|\sum_{a\in\calA} \widehat\pi^1(a\mid s) \widehat Q^2(s,a) - \sum_{a\in\calA} \widehat\pi^2(a\mid s) \widehat Q^2(s,a)\right|\\
        &\leq \|\widehat\pi^1(\cdot\mid s)\|_1\|\widehat Q^1(s,\cdot) - \widehat Q^2(s,\cdot)\|_\infty + \|\widehat \pi^1(\cdot\mid s) - \widehat \pi^2(\cdot\mid s)\|_1\|\widehat Q^2(s,\cdot)\|_\infty\\
        &=\|\widehat Q^1(s,\cdot) - \widehat Q^2(s,\cdot)\|_\infty + \|\widehat \pi^1(\cdot\mid s) - \widehat \pi^2(\cdot\mid s)\|_1\|\widehat Q^2(s,\cdot)\|_\infty
    \end{align*}
    where the first inequality is due to the triangle inequality, and the second inequality is due to H\"older's inequality. 
    By \Cref{lem:lip Q}, 
    \begin{align*}
        \|\widehat Q^1(s,\cdot) - \widehat Q^2(s,\cdot)\|_\infty \leq 4\sqrt{K}\beta_w\max\{C_w, C_\beta\} \left\|(\beta^{Q,1},w^{Q,1},\Lambda^{Q,1}) - (\beta^{Q,2}, w^{Q,2}, \Lambda^{Q,2})\right\|_2.
    \end{align*}
    Furthermore, for the second term,
    \begin{align*}
        &\|\widehat \pi^1(\cdot\mid s) - \widehat \pi^2(\cdot\mid s)\|_1\|\widehat Q^2(s,\cdot)\|_\infty \\
        &\leq C_Q \|\widehat \pi^1(\cdot\mid s) - \widehat \pi^2(\cdot\mid s)\|_1\\
        &\leq C_Q \cdot 32\sqrt{(n+1)K}\beta_w(1+C_Y)\max\{KC_w,KC_\beta\}\left(\frac{8|\calA|}{\theta}\right)^n \\
        &\quad\times\sqrt{\sum_{i=0}^{n} \|(\beta_{i}^{\pi,1},w_{i}^{\pi,1},\Lambda^{\pi,1}) - (\beta_{i}^{\pi,2}, w_i^{\pi,2}, \Lambda^{\pi,2})\|_2^2}
    \end{align*}
    where the first inequality is due to $\|\widehat Q^2\|_\infty \leq C_Q$ for any $\widehat Q^2 \in \widehat \calQ(C_\beta,C_w, C_Q)$, and the second inequality is due to \Cref{lem:lip pol}. Then we deduce that
    \begin{align*}
        &\max_{s\in\calS} \left|\widehat V^1(s) - \widehat V^2(s)\right|\\
        &\leq 4\sqrt{K}\beta_w\max\{C_w, C_\beta\} \left\|(\beta^{Q,1},w^{Q,1},\Lambda^{Q,1}) - (\beta^{Q,2}, w^{Q,2}, \Lambda^{Q,2})\right\|_2 \\
        &\quad+ 32C_Q\sqrt{(n+1)K^3}\beta_w(1+C_Y)\max\{C_w, C_\beta\}\left(\frac{8|\calA|}{\theta}\right)^n\sqrt{\sum_{i=0}^{n} \|(\beta_{i}^{\pi,1},w_{i}^{\pi,1},\Lambda^{\pi,1}) - (\beta_{i}^{\pi,2}, w_{i}^{\pi,2}, \Lambda^{\pi,2})\|_2^2}\\
        &\leq \sqrt{(4\sqrt{K}\beta_w\max\{C_w, C_\beta\})^2 + (32C_Q\sqrt{(n+1)K^3}\beta_w(1+C_Y)\max\{C_w, C_\beta\}(8|\calA|/\theta)^n)^2}\\
        &\quad\times \sqrt{\left\|(\beta^{Q,1},w^{Q,1},\Lambda^{Q,1}) - (\beta^{Q,2}, w^{Q,2}, \Lambda^{Q,2})\right\|_2^2 + \sum_{i=0}^{n} \|(\beta_{i}^{\pi,1},w_{i}^{\pi,1},\Lambda^{\pi,1}) - (\beta_{i}^{\pi,2}, w_{i}^{\pi,2}, \Lambda^{\pi,2})\|_2^2}\\
        &\leq 33\sqrt{(n+1)K^3}\beta_w(1+C_Y)\max\{C_w,C_\beta\} C_Q\left(\frac{8|\calA|}{\theta}\right)^n\\
        &\quad\times \sqrt{\left\|(\beta^{Q,1},w^{Q,1},\Lambda^{Q,1}) - (\beta^{Q,2}, w^{Q,2}, \Lambda^{Q,2})\right\|_2^2 + \sum_{i=0}^{n} \|(\beta_{i}^{\pi,1},w_{i}^{\pi,1},\Lambda^{\pi,1}) - (\beta_{i}^{\pi,2}, w_{i}^{\pi,2}, \Lambda^{\pi,2})\|_2^2}
    \end{align*}
    where the second inequality is due to the Cauchy-Schwarz inequality, and the last inequality is due to the assumption that $1\leq C_Q$ and $\theta \in (0,1)$.

    Note that 
    \begin{align*}
        \|\beta^{Q,1},w^{Q,1},\Lambda^{Q,1},\{\beta_i^{\pi,1},w_i^{\pi,1},\Lambda^{\pi,1}\}_{i=0}^{n}\|_2 
        &\leq |\beta^{Q,1}| + \|w^{Q,1}\|_2 + \|\Lambda^{Q,1}\|_F+ \sum_{i=0}^{n}(|\beta_i^{\pi,1}| + \|w_i^{\pi,1}\|_2 + \|\Lambda^{\pi,1}\|_F)\\
        &\leq C_\beta + C_w + \sqrt{d} + (n+1)(1+C_Y)K(C_\beta + C_w) + (n+1)\sqrt{d}\\
        &\leq (n+2)(1+C_Y)K(C_\beta + C_w) + (n+2)\sqrt{d}.
    \end{align*}
    Note that $(\beta^{Q,1},w^{Q,1},\Lambda^{Q,1},\{\beta_i^{\pi,1},w_i^{\pi,1},\Lambda^{\pi,1}\}_{i=0}^{n})$ can be viewed as a $(n+2)(1+d+d^2)$-dimensional vector. Since $1+d+d^2 \leq 3d^2$, by \Cref{lem:cassel 24},
    \begin{align*} 
        &\log\calN_{\epsilon}(\widehat \calV_n(C_\beta, C_w, C_Q, C_Y))\leq 3(n+2)d^2\log\left(1+2(8|\calA|/\theta)^n C_1 C_2 / \epsilon\right)
    \end{align*}
    where
    \begin{align*}
        &C_1 = 33\sqrt{(n+1)K^3}\beta_wC_Q(1+C_Y)\max\{C_w,C_\beta\},\\
        &C_2 = (n+2)(1+C_Y)K(C_\beta + C_w) + (n+2)\sqrt{d}.
    \end{align*}
    Furthermore, the $\log$ term contains an exponential term in $n$, we further deduce as follows.
    \begin{align*}
        \log\left(1+2(8|\calA|/\theta)^n C_1 C_2 / \epsilon\right) 
        &\leq n\log (8|\calA|/\theta) + \log (1+2C_1C_2/\epsilon)\\
        &\leq (n+1)\log ((8|\calA|/\theta)(1+2C_1C_2/\epsilon)).
    \end{align*}
    Finally, we have
    \begin{align*}
        &\log\calN_{\epsilon}(\widehat \calV_n(C_\beta, C_w, C_Q, C_Y))\leq 3(n+2)^2d^2\log ((8|\calA|/\theta)(1+2C_1C_2/\epsilon)).
    \end{align*}
\end{proof}

Finally, we show an upper bound on the covering number of $\widehat \calV(C_\beta, C_w, C_Q, C_Y)$.
\begin{lemma} \label{lem:cov V}
Let $1\leq \beta_w, C_\beta, C_w, C_Q$ and $0\leq C_Y$. Given $\epsilon>0$, let $\calN_{\epsilon}(\widehat \calV(C_\beta, C_w, C_Q, C_Y))$ denote the $\epsilon$-covering number of $\widehat \calV(C_\beta, C_w, C_Q, C_Y)$ with respect to the $\ell_\infty$-norm, where $\widehat \calV(C_\beta, C_w, C_Q, C_Y) = \bigcup_{n=0}^{K^L} \widehat \calV_n(C_\beta, C_w, C_Q, C_Y)$. Then we have
\begin{align*}
        &\log\calN_\epsilon(\widehat\calV(C_\beta, C_w, C_Q, C_Y)) \leq 3(K^L+2)^2d^2 \log \left((K^L+1)\frac{8|\calA|}{\theta}(1+\frac{2C_1 C_2}{\epsilon})\right)
    \end{align*}
    where $C_1,C_2$ are defined in \eqref{eq:C1 C2} with $n= K^L$.
\end{lemma}
\begin{proof}
    For each $n=0,\ldots, K^L$, let $\calC_n \subseteq \widehat\calV_n(C_\beta, C_w, C_Q, C_Y)$ be an $\epsilon$-cover of $\widehat\calV_n(C_\beta, C_w, C_Q, C_Y)$ with respect to the $\ell_\infty$-norm. By \Cref{lem:cov Vn}, suppose that the covers satisfy
    \begin{align*}
        \log|\calC_n| \leq 
        3(n+2)^2d^2\log ((8|\calA|/\theta)(1+2C_1C_2/\epsilon))
        \quad \forall n = 0,\ldots, K^L
    \end{align*}
    where $C_1, C_2$ are defined in \eqref{eq:C1 C2} with $n=K^L$. Furthermore, let
    \begin{align*}
        \calC = \bigcup_{n=0}^{K^L} \calC_n.    
    \end{align*}
    Then we claim that $\calC$ is an $\epsilon$-cover of $\widehat\calV(C_\beta,C_w, C_Q, C_Y)$. For any $\widehat V\in \widehat \calV(C_\beta, C_w, C_Q, C_Y)$, since $\widehat\calV(C_\beta, C_w, C_Q, C_Y)$ is defined as the union, there exists $m\in\{0,\ldots,K^L\}$ such that
    \begin{align*}
        \widehat V \in \widehat\calV_m(C_\beta, C_w, C_Q, C_Y).
    \end{align*}
    Since $\calC_m$ is an $\epsilon$-cover of $\widehat\calV_m(C_\beta, C_w, C_Q, C_Y)$, there exists $\widehat V_m \in \calC_m \subseteq \calC$ such that
    \begin{align*}
        \|\widehat V - \widehat V_m\|_\infty \leq \epsilon.
    \end{align*}
    This implies that $\calC$ is an $\epsilon$-cover of $\widehat\calV(C_\beta, C_w, C_Q, C_Y)$ with respect to the $\ell_\infty$-norm. Furthermore, we have
    \begin{align*}
        \calN_\epsilon(\widehat\calV(C_\beta, C_w, C_Q, C_Y)) \leq |\calC| \leq \sum_{n=0}^{K^L} |\calC_n| 
        &\leq  (K^L+1)\left(\frac{8|\calA|}{\theta}(1+\frac{2C_1 C_2}{\epsilon})\right)^{3(K^L+2)^2d^2}\\
        &\leq  \left((K^L+1)\frac{8|\calA|}{\theta}(1+\frac{2C_1 C_2}{\epsilon})\right)^{3(K^L+2)^2d^2}
    \end{align*}
    where the second inequality is true because $z \leq z^y$ for any $z,y \geq 1$. Finally, by taking $\log$ on both sides, we have
    \begin{align*}
        \log\calN_\epsilon(\widehat\calV(C_\beta, C_w, C_Q, C_Y)) 
        &\leq 3(K^L+2)^2d^2 \log \left((K^L+1)\frac{8|\calA|}{\theta}(1+\frac{2C_1 C_2}{\epsilon})\right)
    \end{align*} as desired.
\end{proof}


\section{Good Event}
In this section, we introduce a high probability good event, denoted by $E_g$, which simplifies our analysis. We begin by presenting the formal definition of $E_g$. We define $E_g$ as
\begin{equation}\label{def:good event}
    E_g = E_1 \cap E_2 \cap E_3.
\end{equation}
$E_1, E_2, E_3$ are defined as
\begin{align}
    &E_1 = \left\{\forall (k,h) \in [K] \times [H]: \|\theta_{f,h}^k  - \widehat \theta_{f,h}^k\|_{\Lambda_h^k} \leq \beta_r,\ \|\theta_{g,h}  - \widehat \theta_{g,h}^k\|_{\Lambda_h^k} \leq \beta_r\right\},\label{def:E1}\\
    &E_2 = \left\{\forall (k,h,\ell)\in [K]\times [H] \times\{f,g\}: \|(\psi_h - \widehat\psi_h^k) \widehat V_{\ell,h+1}^k\|_{\Lambda_h^k} \leq \beta_p, \|\widehat Q_{\ell,h}^k\|_\infty \leq \beta_Q, Y_k\leq 11\eta H^3 K\right\},\label{def:E2}\\
    &E_3 = \left\{\sum_{k\in[K]} \bbE_{\bbP,\widehat\pi^k}[W_k] \leq 2\sum_{k\in[K]} W_k + 4H(3\beta_b + 8\beta_Q\beta_w^2) \log\frac{6K}{\delta}\right\},\label{def:E3}
\end{align}
where
\begin{align*}
    &\beta_r = 2\sqrt{2d\log(6KH/\delta)},\\
    &\beta_p = 50(K^{1/4} + 1)dH \sqrt{\log(5H^2K^2|\calA|/\delta)},\\
    &\beta_Q = 2H,\\
    &\beta_b = \beta_r + \beta_p,\\
    &\beta_w = 4\beta_b \log K,\\
    &W_k = \sum_{h\in[H]} \left(3\beta_b\|\phi(s_h^k,a_h^k)\|_{(\Lambda_h^k)^{-1}} + 8\beta_Q \beta_w^2 \|\phi(s_h^k,a_h^k)\|_{(\Lambda_h^k)^{-1}}^2\right).
\end{align*}
We note that one of the key differences from \citet{cassel2024warm} is that $E_2$ involves an upper bound of $Y_k$. This is because a (possibly polynomial in $d,H,K$) upper bound of $Y_k$ is required to prove that $E_g$ holds with high probability. In contrast, since we do not truncate $\widehat Q_{g,h}^k$, its trivial upper bound cannot be obtained. Thus, to avoid circular logic, we include it in $E_2$, and use induction to show (Step 3-2 of \Cref{lem:good event}).

However, since directly proving $E_g$ holds with high probability is difficult, we instead consider a proxy good event and then show that it implies $E_g$. Here, we define the proxy good event $\bar E_g$ as
\begin{equation}\label{def:proxy good event}
    \bar E_g = E_1 \cap \bar E_2 \cap E_3,
\end{equation}
where $\bar E_2$ is defined as
\begin{align*}
    \bar E_2 = \left\{\forall (k,h, \widehat V)\in [K]\times [H] \times \widehat \calV(\beta_b, 2K \beta_Q , \beta_{Q}, 11\eta H^3 K): \|(\psi_h - \widehat\psi_h^k) \widehat V\|_{\Lambda_h^k} \leq \beta_p \right\}.
 \end{align*}

Based on the upper bound of the covering number of $\widehat \calV(C_\beta, C_w, C_Q, C_Y)$,  we prove that $\bar E_g$ holds with high probability.
\begin{lemma}[Proxy Good Event, $\bar E_g$]\label{lem:proxy good event}
   Let $1 \leq \beta_w \leq K$, and let $\eta,\alpha\leq 1$. Then $\Pr[\bar E_g] \geq 1-\delta$ for any $\delta\in(0,1)$.
\end{lemma}
\begin{proof}
    We prove the statement by showing $\Pr[E_1] \geq 1-\delta/3$, $\Pr[\bar E_2] \geq 1-\delta/3$, and $\Pr[E_3] \geq 1-\delta/3$. For $E_1,$ by \Cref{lem:cassel 21}, we have for all $h\in[H],k\geq 1$ with probability at least $1-\delta/3$,
    \[
        \|\theta_{g, h} - \widehat \theta_{g, h}^k\|_{\Lambda_h^k} \leq 2\sqrt{2d\log(6KH/\delta)}:= \beta_r.
    \]
    Furthermore, it is clear that $\|\theta_{f,h}^k - \widehat\theta_{f,h}^k\|_{\Lambda_h^k} = 0 \leq \beta_r$ because we take $\widehat\theta_{f,h}^k = \theta_{f,h}^k$. Then for any $h,k,\ell\in\{f,g\}$, $E_1$ holds with probability at least $1-\delta/3$. 
    
    Now we consider $\bar E_2$.  By \Cref{lem:cassel 22}, for all $\widehat V \in \widehat \calV(\beta_b, 2K \beta_Q , \beta_{Q}, 11\eta H^3 K)$, with probability at least $1-\delta/3$, 
    \begin{align}\label{eq:lem:proxy:E2}
        \|(\psi_h - \widehat\psi_h^k) \widehat V\|_{\Lambda_h^k} \leq 4\beta_{Q,h} \sqrt{d\log(K+1) + 2\log(3H^2 / \delta) + 2\log \calN_\epsilon(\widehat \calV(\beta_b, 2K \beta_Q , \beta_{Q}, 11\eta H^3 K))}.
    \end{align}
    The parameters in \Cref{alg:main-linear} satisfy $1\leq \beta_b, 2K\beta_Q, \beta_{Q}$. Then \Cref{lem:cov V} can be applied to deduce the covering number. It follows that
    \begin{align*}
        \log \calN_\epsilon(\widehat \calV(\beta_b, 2K \beta_Q , \beta_{Q}, 11\eta H^3 K)) \leq 3(K^L+2)^2d^2 \log \left((K^L+1)\frac{8|\calA|}{\theta}(1+\frac{2C_1 C_2}{\epsilon})\right).
    \end{align*}
    Since we assume $1 \leq \beta_w \leq K, \ \beta_Q \leq 2H,\ \eta,\alpha \leq 1$, and $K^L \leq K$, we have the following bounds on $C_1,\ C_2$ with $n=K^L$.
    \begin{align*}
        C_1 
        &= 33\sqrt{(K^L+1)K^3}\beta_w\max\{2K\beta_Q, \beta_b\} \beta_{Q}(1+11\eta H^3K)\\
        &\leq 4481 H^5 K^6,\\
        C_2
        &=(K^L+2)(1+11\eta H^3K)K(\beta_b + 2K\beta_Q) + (K^L+2)\sqrt{d}\\
        &\leq 242\sqrt{d}H^4 K^4.
    \end{align*}
    By \Cref{lem:cassel 22}, we can take $\epsilon = \sqrt{d}/(2K)$, and thus the covering number is bounded as
    \begin{align*}
        \log \calN_\epsilon(\widehat \calV(\beta_b, 2K \beta_Q , \beta_{Q}, 11\eta H^3 K)) \leq 36(K^L+2)^2d^2\log(5HK|\calA|/\theta).
    \end{align*}
    Applying this to \eqref{eq:lem:proxy:E2}, since $\theta = K^{-1}$,
    \begin{align*}
        \|(\psi_h - \widehat\psi_h^k) \widehat V\|_{\Lambda_h^k} 
        &\leq 8H\sqrt{d\log(K+1) + 2\log(3H^2/\delta) + 36(K^L+2)^2d^2\log(5HK|\calA|/\theta)}\\
        &\leq 50(K^L+1)dH\sqrt{\log(5H^2K^2|\calA|/\delta)}\\
        &:= \beta_p.
    \end{align*}
    Thus, we showed that $\Pr[\bar E_2] \geq 1-\delta/3$ holds.
    For $E_3$, note that for any $(s,a)\in\calS \times\calA$,
    \begin{align*}
        \sum_{h\in [H]} 3\beta_b\|\phi(s,a)\|_{(\Lambda_h^{k})^{-1}} + 8\beta_Q \beta_w^2 \|\phi(s,a)\|_{(\Lambda_h^k)^{-1}}^2 \leq H\left(3\beta_b + 8\beta_Q\beta_w^2\right).
    \end{align*}
    Furthermore, $s_h^k,a_h^k$ are generated under $\bbP,\widehat\pi^k$. Then, by \Cref{lem:rosen d4} with probability at least $1-\delta/3$,
    \begin{align*}
        \sum_{k\in[K]} \bbE_{\bbP,\widehat\pi^k}[W_k] \leq 2 \sum_{k\in[K]} W_k + 4H(3\beta_b + 8\beta_Q\beta_w^2)\log\frac{6K}{\delta}.
    \end{align*}
    Consequently, by union bound, we have $\Pr[\bar E_g] \geq 1-\delta$.
\end{proof}

Before proving that $E_g$ holds with high probability, we show the following lemma, which is a modification of Lemma 12 of \citet{cassel2024warm} to our CMDP setting. This lemma plays a crucial role in establishing the connection between $\bar E_g$ and $E_g$.
\begin{lemma}\label{lem:cassel 12}
    Suppose that $\bar E_g$ holds. Given $k\in [K]$, if $\widehat\pi_h^k \in \widehat \Pi(K\beta_b, 2K^2 \beta_Q , K\beta_{Q,h}, 11\eta H^3 K)$ for all $h \in [H]$, then $\widehat Q_{f,h}^{k}, \widehat Q_{g,h}^{k} \in \widehat \calQ(\beta_b, 2K \beta_Q , \beta_{Q,h})$, and $\widehat V_{f,h}^{k}, \widehat V_{g,h}^{k} \in \widehat \calV(\beta_b, 2K \beta_Q , \beta_{Q,h}, 11\eta H^3 K)$ for all $h\in[H+1]$.
\end{lemma}
\begin{proof}
To show the statement, we apply induction on $h$ for fixed $k$. For the base case, consider $h=H+1$. As we initialize as $\widehat Q_{f,H+1}^{k}(s,a)=\widehat Q_{g,H+1}^{k}(s,a) = \widehat V_{f,H+1}^{k}(s) = \widehat V_{g,H+1}^{k}(s) = 0$ for all $(s,a)$, it is clear that 
    $\widehat Q_{f,H+1}^{k}, \widehat Q_{g,H+1}^{k} \in \widehat\calQ(\beta_b, 2K \beta_Q , \beta_{Q,H+1})$ and $\widehat V_{f,H+1}^{k}, \widehat V_{g,H+1}^{k} \in \widehat\calV(\beta_b, 2K \beta_Q , \beta_{Q,H+1}, 11\eta H^3 K)$.
    Next, we assume that the statement is true for $h+1$, i.e., $\widehat Q_{f,h+1}^{k}, \widehat Q_{g,h+1}^{k} \in \widehat\calQ(\beta_b, 2K \beta_Q , \beta_{Q,h+1})$, $\widehat V_{f,h+1}^{k}, \widehat V_{g,h+1}^{k} \in \widehat\calV(\beta_b, 2K \beta_Q , \beta_{Q,h+1}, 11\eta H^3 K)$. It follows that for any $\ell\in\{f,g\}$,
    \begin{align*}
        |\widehat Q_{\ell,h}^{k}(s,a)| 
        &= \left|\bar\phi_h^{k_e}(s,a)^\top \left(\widehat \theta_{\ell,h}^{k} + \widehat\psi_h^k \widehat V_{\ell,h+1}^{k} \right) - \beta_b \left\|\bar \phi_h^{k_e}(s,a)\right\|_{(\Lambda_h^{k_e})^{-1}}\right|\\
        &\leq \left|\bar\phi_h^{k_e}(s,a)^\top \left(\theta_{\ell,h}^k + \psi_h \widehat V_{\ell,h+1}^{k} \right)\right| \\
        &\quad + \left(\beta_b + \left\|\widehat\theta_{\ell,h}^{k} - \theta_{\ell,h}^k\right\|_{\Lambda_h^{k_e}} + \left\|(\widehat\psi_h^k - \psi_h) \widehat V_{\ell,h+1}^{k}\right\|_{\Lambda_h^{k_e}}\right)\left\| \bar\phi_h^{k_e}(s,a)\right\|_{(\Lambda_h^{k_e})^{-1}}\\
        &\leq \left|\bar\phi_h^{k_e}(s,a)^\top \left(\theta_{\ell,h}^k + \psi_h \widehat V_{\ell,h+1}^{k} \right)\right| \\
        &\quad + \left(\beta_b + \left\|\widehat\theta_{\ell,h}^{k} - \theta_{\ell,h}^k\right\|_{\Lambda_h^{k}} + \left\|(\widehat\psi_h^k - \psi_h) \widehat V_{\ell,h+1}^{k}\right\|_{\Lambda_h^{k}}\right)\left\| \bar\phi_h^{k_e}(s,a)\right\|_{(\Lambda_h^{k_e})^{-1}}
    \end{align*}
    where the first inequality is due to the triangle inequality and the Cauchy-Schwarz inequality, and the second inequality is due to the fact that $\Lambda_h^{k_e} \preceq \Lambda_h^k$. 
    
    We bound each term individually. For the first term, for all $\ell\in\{f,g\}$,
    \begin{align*}
        \left|\bar\phi_h^{k_e}(s,a)^\top \left(\theta_{\ell,h}^k + \psi_h \widehat V_{\ell,h+1}^{k}\right)\right| &= \sigma\left(-\beta_w \|\phi(s,a)\|_{(\Lambda_h^{k_e})^{-1}} + \log K\right) \left| \phi(s,a)^\top (\theta_{\ell,h}^k + \psi_h \widehat V_{\ell,h+1}^{k}) \right|\\
        &\leq \left| \phi(s,a)^\top (\theta_{\ell,h}^k + \psi_h \widehat V_{\ell,h+1}^{k}) \right|\\
        &=\left| \ell_h^k(s,a) + \sum_{s'\in\calS} \bbP_h(s'\mid s,a) \widehat V_{\ell,h+1}^{k}(s') \right|\\
        &\leq 1 + \|\widehat V_{\ell,h+1}^{k}\|_\infty
    \end{align*}
    where the first and second equality are due to the definition of $\bar\phi_h^{k_e}$ and linear MDPs, respectively. Next, we can bound the second term, since the proxy good event $\bar E_g$ is assumed.
    \begin{align*}
         \left(\beta_b + \left\|\widehat\theta_{\ell,h}^{k} - \theta_{\ell,h}^k\right\|_{\Lambda_h^{k}} + \left\|(\widehat\psi_h^k - \psi_h) \widehat V_{\ell,h+1}^{k}\right\|_{\Lambda_h^{k}}\right) \left\| \bar\phi_h^{k_e}(s,a)\right\|_{(\Lambda_h^{k_e})^{-1}} \leq (\beta_b + \beta_r + \beta_p) \left\| \bar\phi_h^{k_e}(s,a)\right\|_{(\Lambda_h^{k_e})^{-1}}.
    \end{align*}
    Recall that $\|\bar\phi_h^{k_e}(s,a)\|_{(\Lambda_h^{k_e})^{-1}} = \|\phi(s,a)\|_{(\Lambda_h^{k_e})^{-1}}\sigma(-\beta_w \|\phi(s,a)\|_{(\Lambda_h^{k_e})^{-1}} + \log K) \leq \max_{y\geq 0} y\cdot\sigma(-\beta_w y + \log K)$. It follows that
    \begin{align*}
        |\widehat Q_{\ell,h}^{k}(s,a)| 
        &\leq 1 + \|\widehat V_{\ell,h+1}^{k}\|_\infty + (\beta_b + \beta_r + \beta_p) \left\| \bar\phi_h^{k_e}(s,a)\right\|_{(\Lambda_h^{k_e})^{-1}} \\
        &\leq 1 + \|\widehat V_{\ell,h+1}^{k}\|_\infty + (\beta_b + \beta_r + \beta_p) \max_{y\geq 0}\left[y\cdot\sigma(-\beta_w y + \log K)\right]\\
        &\leq 1 + \|\widehat V_{\ell,h+1}^{k}\|_\infty + \frac{2\log K}{\beta_w} (\beta_r + \beta_p + \beta_b)\\
        &= 2 + \|\widehat V_{\ell,h+1}^{k}\|_\infty\\
        &\leq 2 + \beta_{Q,h+1}\\
        &= \beta_{Q,h}
    \end{align*}
    where the third inequality follows from \Cref{lem:cassel 18}, the equality is due to $\beta_w = 2(\beta_r + \beta_p + \beta_b)\log K$, and the last inequality holds because of the induction hypothesis. 
    
    So far, we have shown that $\|\widehat Q_{\ell,h}^k\|_\infty \leq \beta_{Q,h}$. To show $\widehat Q_{\ell,h}^{k} \in \widehat\calQ (\beta_b, 2K \beta_Q , \beta_{Q,h})$, it remains to show that the corresponding parameters are upper bounded. Recall that $\widehat Q_{\ell,h}^{k}$ is defined as
    \[
        \widehat Q_{\ell,h}^{k}(s,a) 
        = \bar\phi_h^{k_e}(s,a)^\top w_{\ell,h}^{k} - \beta_b \left\|\bar \phi_h^{k_e}(s,a)\right\|_{(\Lambda_h^{k_e})^{-1}}
    \]
    where $w_{\ell,h}^{k} = \widehat\theta_{\ell,h}^{k} + \widehat\psi_{h}^{k} \widehat V_{\ell,h+1}^{k}$. Note that 
    \begin{align*}
        \|\widehat \theta_{\ell, h}^{k}\|_2 \leq \left\| (\Lambda_h^k)^{-1} \sum_{\tau \in [k-1]} \phi(s_h^\tau,a_h^\tau) \ell_h^\tau(s_h^\tau, a_h^\tau) \right\|_2 \leq \left\| (\Lambda_h^k)^{-1} \right\|_2\left\| \sum_{\tau \in [k-1]} \phi(s_h^\tau,a_h^\tau) \ell_h^\tau(s_h^\tau, a_h^\tau) \right\|_2 \leq K.
    \end{align*}
    Furthermore, the induction hypothesis implies that
    \[
        \left\|\widehat\psi_{h}^{k} \widehat V_{\ell,h+1}^{k}\right\|_2 = \left\|(\Lambda_h^k)^{-1} \sum_{\tau \in [k-1]} \phi(s_h^\tau,a_h^\tau) \widehat V_{\ell,h+1}^{k}(s_{h+1}^\tau)\right\|_2 \leq \beta_{Q,h+1}K.
    \]
    It follows that
    \[
        \|w_{\ell,h}^{k}\|_2 \leq \|\widehat \theta_{\ell,h}^{k}\|_2 + \left\|\widehat\psi_{h}^{k} \widehat V_{\ell,h+1}^{k}\right\|_2 \leq 2\beta_{Q} K.
    \]
    Furthermore, we have $(2K)^{-1}I\preceq(\Lambda_h^{k_e})^{-1} \preceq I$. Thus, we have
    \[
        \widehat Q_{\ell,h}^{k} \in \widehat\calQ(\beta_b, 2K \beta_Q , \beta_{Q,h}).
    \]
    By definition, since we have $\widehat V_{\ell,h}^{k}(s) = \sum_{a\in\calA} \widehat\pi_h^k(a\mid s) \widehat Q_{\ell,h}^{k}(s,a)$ and the assumption $\widehat\pi_h^k \in \widehat \Pi(K\beta_b, 2K^2\beta_Q, K\beta_{Q,h}, 11\eta H^3 K)$, it follows that
    \[
        \widehat V_{\ell,h}^{k} \in \widehat\calV(\beta_b, 2K \beta_Q , \beta_{Q,h}, 11\eta H^3 K).
    \]
    This completes the proof.
\end{proof}

Finally, we prove that $E_g$ holds with high probability. The proof closely follows Lemma 6 of \citet{cassel2024warm}, with modifications for the CMDP setting.
\begin{lemma}[Restatement of \Cref{lem:good event main-text}]\label{lem:good event} Let $1 \leq \beta_w \leq K$, let $\eta,\alpha \leq 1$ and $4\alpha \eta H^3 \leq 1$. Then $\Pr[E_g] \geq 1-\delta$ for any $\delta\in(0,1)$.
\end{lemma}
\begin{proof}
    We assume $\bar E_g$, which holds with probability at least $1-\delta$ by \Cref{lem:proxy good event}. Next, under $\bar E_g$, we focus on showing $E_2$. As a first step, we show that $\widehat\pi_h^k \in \widehat \Pi(K\beta_b, 2K^2\beta_Q , K\beta_{Q,h}, 11\eta H^3 K)$ and $Y_k \in [0, 11\eta H^3 k]$ for all $k,h$ using induction on $k\in K_e$ for each epoch $e\in E$. Finally, based on this induction, we prove that $E_2$ holds.

    \paragraph{Step 1: Base Case}
    First, let us fix $e\in E$. For the base case, consider $k=k_e$. Since $\widehat\pi_h^{k_e} = \piunif$ for all $h$, it follows that $\widehat \pi_h^{k_e} \in \widehat \Pi(K\beta_b, 2K^2 \beta_Q , K\beta_{Q,h}, 11\eta H^3 K)$, as $\piunif$ can be viewed as $\pi(a\mid s; 0, 0, I)$ with $n=0$. Furthermore, we initialize $Y_{k_e} = 0$. Thus, the base case holds.

    \paragraph{Step 2: Induction Hypothesis}
    For $k\in K_e$, we assume that $\widehat\pi_h^{k'} \in \widehat \Pi(K\beta_b, 2K^2 \beta_Q , K\beta_{Q,h}, 11\eta H^3 K)$ and $Y_{k'}\in [0, 11\eta H^3 k']$ for all $h$ and $k_e \leq k' < k$. Then, by \Cref{lem:cassel 12}, it follows that for all $(h,k',\ell) \in [H]\times \{k_e,\ldots,k-1\} \times\{f,g\}$,
    \begin{align}\label{eq:lem:good event:lemma12}
    \begin{aligned}
        &\widehat Q_{\ell, h}^{k'} \in \widehat \calQ(\beta_b, 2K \beta_Q , \beta_{Q,h}),\quad \widehat V_{\ell, h}^{k'} \in \widehat\calV(\beta_b, 2K \beta_Q , \beta_{Q,h}, 11\eta H^3 K).
    \end{aligned}
    \end{align}
    Furthermore, let $\beta_b, w_{\ell,h}^{k'},\Lambda$ denote the parameters that specify $\widehat Q_{\ell,h}^{k'}$, i.e., for all $h,k'<k, \ell\in \{f,g\}$, $\widehat Q_{\ell,h}^{k'}(\cdot,\cdot) = \widehat Q(\cdot,\cdot; \beta_b, w_{\ell,h}^{k'}, \Lambda)$.
    
    \paragraph{Step 3-1: Induction Step ($\widehat\pi_h^k$)}
    Next, we show that $\widehat\pi_h^k \in \widehat\Pi(K\beta_b, 2K^2\beta_Q, K\beta_{Q,h}, 11\eta H^3 K)$ for all $h$. By \Cref{lem:pol param}, 
    \begin{align*}
        \widehat \pi_h^k(\cdot\mid \cdot) = \widehat\pi(\cdot\mid \cdot; \{\beta_i, w_i\}_{i=0}^n, \Lambda)
    \end{align*}
    where $w_i = -\alpha \sum_{j\in S_i} (w_{f,h}^{j} + Y_{j} w_{g,h}^{j}),\ \beta_i = -\alpha \beta_b \sum_{j\in S_i} (1+Y_{j}).$ Note that $j \in S_i$ satisfies $j < k$ for each $i$, thus we can use \eqref{eq:lem:good event:lemma12} to bound the parameters. For $w_i$,
    \begin{align}\label{eq:lem:good event:w}
    \begin{aligned}
        \|w_i\|_2 
        &\leq \sum_{j\in S_i} \|w_{f,h}^{j} + Y_{j} w_{g,h}^{j}\|_2\\
        &\leq \sum_{j\in S_i} \| w_{f,h}^{j}\|_2 + Y_{j} \|w_{g,h}^{j}\|_2\\
        &\leq \sum_{j\in S_i}(1+ 11\eta H^3 K)2K\beta_Q \\
        &\leq (1+ 11\eta H^3 K)2K^2\beta_Q
    \end{aligned}
    \end{align}
    where the first inequality is due to $|\alpha|\leq1$, and the third inequality is due to the induction hypothesis ($Y_{k'}<11\eta H^3 k'$ for all $k'<k$) and that \eqref{eq:lem:good event:lemma12} implies $\|w_{\ell,h}^{j}\|_2 \leq 2K\beta_Q$. Similarly,
    \begin{align}\label{eq:lem:good event:beta}
        |\beta_i| \leq (1+11\eta H^3 K)K\beta_b.
    \end{align}
    Again, by \eqref{eq:lem:good event:lemma12}, we have $\|\widehat Q_{f,h}^j\|_\infty, \|\widehat Q_{g,h}^j\|_\infty \leq \beta_Q$. Then, for any $(s,a)\in\calS\times\calA$ and $i=0,\ldots,n$,
    \begin{align}\label{eq:lem:good event:Q}
        \left|-\alpha \sum_{j\in S_i} (\widehat Q_{f,h}^{j}(s,a) + Y_{j}\widehat Q_{g,h}^{j}(s,a))\right| \leq (1+11\eta H^3 K)K\beta_{Q,h}.
    \end{align}
    Note that $n = \max\{0,\floor{(k-1-k_e)/K^B}\} \leq \floor{K/K^B} \leq K^{1-B}=K^L$. Furthermore, its parameters are bounded by \eqref{eq:lem:good event:w}, \eqref{eq:lem:good event:beta}, and \eqref{eq:lem:good event:Q}, and the same argument can be applied for all $h\in [H]$. Thus, for all $h\in [H]$,
    \begin{align*}
        \widehat\pi_h^k \in \widehat \Pi(K\beta_b, 2K^2 \beta_Q , K\beta_{Q,h}, 11\eta H^3 K).
    \end{align*}
    \paragraph{Step 3-2: Induction Step ($Y_k$)} To bound $Y_k$, 
    \begin{align*}
        Y_k &= \left[(1-4\alpha\eta H^3)Y_{k-1}+\eta\left(\widehat V_{g,1}^{k-1}(s_1)- b -4\alpha H^3 - 4\theta H^2\right) \right]_+\\
        &\leq \left|(1-4\alpha\eta H^3)Y_{k-1}+\eta\left(\widehat V_{g,1}^{k-1}(s_1)- b -4\alpha H^3 - 4\theta H^2\right) \right|\\
        &\leq (1-4\alpha\eta H^3) |Y_{k-1}| + \eta |\widehat V_{g,1}^{k-1}(s_1)- b -4\alpha H^3 - 4\theta H^2|\\
        &\leq |Y_{k-1}| + \eta |\widehat V_{g,1}^{k-1}(s_1)- b -4\alpha H^3 - 4\theta H^2|\\
        &\leq 11\eta H^3(k-1) + 11\eta H^3\\
        &\leq 11\eta H^3k
    \end{align*}
    where the first inequality is due to the fact that $\max\{0,z\}\leq |z|$ for all $z\in\bbR$, the second and third inequality follows from the triangle inequality and $0\leq1-4\alpha\eta H^3 \leq 1$, and the fourth inequality is due the induction hypothesis, i.e., $Y_{k'}\leq 11\eta H^3 k'$ and $\|\widehat V_{g,1}^{k-1}
    \|_\infty \leq 2H$ for all $k'<k$.
    
    These complete the induction, i.e., $\widehat\pi_h^k \in \widehat \Pi(K\beta_b, 2K^2\beta_Q , K\beta_{Q,h}, 11\eta H^3 K)$ and $Y_k \in [0, 11\eta H^3 k]$ for all $(h,k) \in [H] \times K_e$. Furthermore, we can apply the same argument for all $e\in E$. Thus, it holds for all $(h,k) \in [H]\times [K]$.

    \paragraph{Step 4: Showing $E_g$}
    By \Cref{lem:cassel 12}, we have for all $h,k$,
    \begin{align*}
        &\widehat Q_{f,h}^k, \widehat Q_{g,h}^k \in \widehat\calQ(\beta_b, 2K\beta_Q, \beta_{Q,h}),\quad \widehat V_{f,h}^{k}, \widehat V_{g,h}^{k} \in \widehat \calV(\beta_b, 2K \beta_Q , \beta_{Q,h}, 11\eta H^3 K).
    \end{align*}
    As a result, since $\bar E_2$ is assumed, we have $\|(\psi_h - \widehat\psi_h^k) \widehat V_{\ell,h+1}^k\|_{\Lambda_h^k} \leq \beta_p$. Thus, $E_2$ holds. Furthermore, $E_1,E_3$ hold by $\bar E_g$. This completes the proof.
\end{proof}

\section{Lyapunov Drift Analysis}\label{appendix:drift}
In this section, we upper bound the dual variable based on a Lyapunov drift analysis. As a first step, we bound the Lyapunov drift $(Y_{k+1}^2 - Y_k^2)/2$.
\begin{lemma}\label{lem:lyap 1}
    Assume that the good event $E_g$ holds. For all $e \in E$ and $k, k+1\in K_e$, the Lyapunov drift is bounded as
    \begin{align*}
    \begin{aligned}
        \frac{Y_{k+1}^2 - Y_k^2}{2} 
        &\leq -\eta\gamma Y_k + \frac{\eta}{\alpha} \bbE_{\bar\bbP^{k_e},\bar\pi}\left[\sum_{h\in [H]}D(\bar\pi_h(\cdot|s_h) || \widetilde\pi_h^k(\cdot|s_h)) - D(\bar\pi_h(\cdot|s_h)||\widehat\pi_h^{k+1}(\cdot|s_h))\right]\\
        &\quad+\eta(2\alpha H^3 + 4H^2 \theta + 4H^2) + 2\eta^2(9H^2 + 16\alpha^2 H^6 + 1936\alpha^2 \eta^2 H^{12}K^2  + 16\theta^2 H^4).
    \end{aligned}
    \end{align*}
\end{lemma}
\begin{proof}
    Recall that the dual variable follows $Y_{k+1} = [(1-4\alpha\eta H^3)Y_{k}+\eta(\widehat V_{g,1}^{k}(s_1)- b -4\alpha H^3 - 4\theta H^2)]_+$. It can be rewritten as
    $$Y_{k+1} = \left[Y_{k} + \eta \left(\widehat V_{g,1}^{k}(s_1) - b - 4\alpha H^3(1+Y_k) - 4\theta H^2\right)\right]_+.$$
    Note that $\max\{0,z\}^2 \leq z^2$ for any $z\in\bbR$. Then, if we square both sides, we have
    \begin{align*}
        Y_{k+1}^2 \leq Y_k^2 + 2 Y_k \eta\left(\widehat V_{g,1}^{k}(s_1) - b - 4\alpha H^3(1+Y_k) - 4\theta H^2\right) + \eta^2\left(\widehat V_{g,1}^{k}(s_1) - b - 4\alpha H^3(1+Y_k) - 4\theta H^2\right)^2.
    \end{align*}
    It can be rewritten as
    \begin{align}\label{eq:lyap 1}
    \begin{aligned}
        &\frac{Y_{k+1}^2 - Y_{k}^2}{2} \\ 
        &\leq \underbrace{Y_{k}\eta\left( \widehat V_{g,1}^{k}(s_1)-b - 4\alpha H^3(1+Y_k) - 4\theta H^2\right)}_{\text{(I)}} + \underbrace{\frac{\eta^2}{2}\left(\widehat V_{g,1}^{k}(s_1) - b - 4\alpha H^3(1+Y_k) - 4\theta H^2\right)^2}_{\text{(II)}}.
    \end{aligned}
    \end{align}
    (II) can be bounded as follows.
    \begin{align*}
        \text{(II)} 
        &\leq \frac{\eta^2}{2}\cdot 4((\widehat V_{g,1}^k(s_1) - b)^2 + (4\alpha H^3)^2 + (4\alpha H^3Y_k)^2 + (4\theta H^2)^2) \\
        &\leq \frac{\eta^2}{2}\cdot 4(9H^2 + 16\alpha^2 H^6 + 16\alpha^2 H^6 Y_k^2 + 16\theta^2 H^4)\\
        &\leq \frac{\eta^2}{2}\cdot 4(9H^2 + 16\alpha^2 H^6 + 1936\alpha^2 \eta^2 H^{12}K^2  + 16\theta^2 H^4)
    \end{align*}
    where the first inequality follows from the Cauchy-Schwarz inequality, the second and third inequalities follow from that $E_g$ implies $|\widehat V_{g,1}^k(s_1)| \leq 2H$ and $0\leq Y_k \leq 11\eta H^3 k$ for all $k$.

    Next, we bound (I). To obtain a negative drift, we first deduce a bound on $Y_k\langle \widehat\pi^{k+1}_h(\cdot|s_h) - \widehat\pi_h^{k}(\cdot|s_h), \widehat Q_{g,h}^k(s_h,\cdot)\rangle$. Since we assumed $k,k+1 \in K_e$, $\widehat\pi_h^{k+1}\neq \piunif$. Then $\widehat\pi_h^{k+1}$ satisfies
    \begin{equation*}
        \widehat\pi_h^{k+1}(\cdot\mid s) = \argmin_{\pi(\cdot\mid s)\in\Delta(\calA)}\ \langle \pi, \widehat Q_{f,h}^{k} + Y_{k} \widehat Q_{g,h}^{k} \rangle + \frac{1}{\alpha}D(\pi || \widetilde\pi_h^{k}).
    \end{equation*}
    Applying \Cref{lem:pushback} and letting $z=\bar\pi_h$, we have for any $s_h \in \calS$,
    \begin{align*}
        &\left\langle \widehat\pi_h^{k+1}(\cdot|s_h), \widehat Q_{f,h}^{k}(s_h,\cdot) + Y_{k} \widehat Q_{g,h}^{k}(s_h,\cdot)\right\rangle + \frac{1}{\alpha}D(\widehat\pi_h^{k+1}(\cdot|s_h)|| \widetilde\pi_h^{k}(\cdot|s_h))\\
        &\leq \left\langle \bar \pi_h(\cdot|s_h), \widehat Q_{f,h}^{k}(s_h,\cdot) + Y_{k}\widehat Q_{g,h}^{k}(s_h,\cdot)\right\rangle + \frac{1}{\alpha} D(\bar \pi_h(\cdot|s_h) || \widetilde\pi_h^{k}(\cdot|s_h)) - \frac{1}{\alpha} D(\bar \pi_h(\cdot|s_h) || \widehat\pi_h^{k+1}(\cdot|s_h))
    \end{align*}
    where $\bar\pi$ is the Slater policy satisfying $V_{g,1}^{\bar \pi}(s_1) \leq b - \gamma$ for some $\gamma > 0$. Next, summing over $h$ and rearranging terms yield
    \begin{align}\label{eq:lyap 2}
    \begin{aligned}
        &Y_k\sum_{h\in [H]} \langle \widehat\pi^{k+1}_h(\cdot|s_h) - \widehat\pi_h^{k}(\cdot|s_h), \widehat Q_{g,h}^k(s_h,\cdot)\rangle \\
        &\leq \frac{1}{\alpha} \sum_{h\in [H]}D(\bar\pi_h(\cdot|s_h) || \widetilde\pi_h^k(\cdot|s_h)) - \frac{1}{\alpha} \sum_{h\in [H]}D(\bar\pi_h(\cdot|s_h)||\widehat\pi_h^{k+1}(\cdot|s_h)) \\
        &\quad+ \sum_{h\in [H]}\langle \widehat\pi_h^k(\cdot|s_h) - \widehat\pi_h^{k+1}(\cdot|s_h), \widehat Q_{f,h}^{k}(s_h,\cdot) \rangle - \frac{1}{\alpha}\sum_{h\in [H]}D(\widehat\pi_h^{k+1}(\cdot|s_h)||\widetilde\pi_h^k(\cdot|s_h)) \\
        &\quad+ \sum_{h\in [H]}\langle \bar\pi_h(\cdot|s_h) - \widehat\pi_h^k(\cdot|s_h), \widehat Q_{f,h}^k(s_h,\cdot) \rangle + Y_k\sum_{h\in [H]}\langle \bar\pi_h(\cdot|s_h) - \widehat\pi_h^k(\cdot|s_h), \widehat Q_{g,h}^k(s_h,\cdot) \rangle.
    \end{aligned}
    \end{align}
    Now, we take $\bbE_{\bar\bbP^{k_e}, \bar\pi}$, which is taken over $\{s_h\}_{h=1}^H$ under $\bar\bbP^{k_e}, \bar\pi$ for a fixed $s_1$. Note that since $\bar\bbP^{k_e}$ is a transition kernel of a contracted MDP, it could be $\sum_{s'}\bar\bbP^{k_e}(s'|s,a) \leq 1$. However, taking $\bbE_{\bar\bbP^{k_e},\bar\pi}$ can be viewed as a linear combination, where its coefficients is in the form of a sub-probability measure defined as $\Pr[s_1=s, \ldots, s_H=s' \mid s_1, \bar\bbP^{k_e},\bar\pi] \in [0,1]$. This implies that taking $\bbE_{\bar\bbP^{k_e}, \bar\pi}$ guarantees monotonicity, i.e.,
    \begin{align}\label{eq:lyap 2.5}
    \begin{aligned}
        &Y_k\bbE_{\bar\bbP^{k_e},\bar\pi}\left[\sum_{h\in [H]} \langle \widehat\pi^{k+1}_h(\cdot|s_h) - \widehat\pi_h^{k}(\cdot|s_h), \widehat Q_{g,h}^k(s_h,\cdot)\rangle\right] \\
        &\leq \frac{1}{\alpha} \bbE_{\bar\bbP^{k_e},\bar\pi}\left[\sum_{h\in [H]}D(\bar\pi_h(\cdot|s_h) || \widetilde\pi_h^k(\cdot|s_h)) - D(\bar\pi_h(\cdot|s_h)||\widehat\pi_h^{k+1}(\cdot|s_h))\right] \\
        &\quad+ \underbrace{\bbE_{\bar\bbP^{k_e},\bar\pi}\left[\sum_{h\in [H]}\langle \widehat\pi_h^k(\cdot|s_h) - \widehat\pi_h^{k+1}(\cdot|s_h), \widehat Q_{f,h}^{k}(s_h,\cdot) \rangle - \frac{1}{\alpha}D(\widehat\pi_h^{k+1}(\cdot|s_h)||\widetilde\pi_h^k(\cdot|s_h))\right]}_{\text{(III)}} \\
        &\quad+ \underbrace{\bbE_{\bar\bbP^{k_e},\bar\pi}\left[\sum_{h\in [H]}\langle \bar\pi_h(\cdot|s_h) - \widehat\pi_h^k(\cdot|s_h), \widehat Q_{f,h}^k(s_h,\cdot) \rangle\right]}_{\text{(IV)}} \\
        &\quad+ \underbrace{Y_k\bbE_{\bar\bbP^{k_e},\bar\pi}\left[\sum_{h\in [H]}\langle \bar\pi_h(\cdot|s_h) - \widehat\pi_h^k(\cdot|s_h), \widehat Q_{g,h}^k(s_h,\cdot) \rangle\right]}_{\text{(V)}}.
    \end{aligned}
    \end{align}
    We bound (III). By the third statement of \Cref{lem:policy inequality}, for any $s_h \in \calS$,
    we have 
    \begin{align*}
        \langle \widehat\pi_h^k(\cdot|s_h) - \widehat\pi_h^{k+1}(\cdot|s_h), \widehat Q_{f,h}^{k}(s_h,\cdot) \rangle - \frac{1}{\alpha}D(\widehat\pi_h^{k+1}(\cdot|s_h)||\widetilde\pi_h^k(\cdot|s_h)) \leq 2\alpha H^2 + 4H \theta.
    \end{align*}
    Thus, summing over $h\in [H]$ and taking $\bbE_{\bar\bbP^{k_e},\bar\pi}$, we have
    \begin{align*}
        \text{(III)} \leq \bbE_{\bar\bbP^{k_e},\bar\pi}\left[2\alpha H^3 + 4H^2 \theta\right] \leq 2\alpha H^3 + 4H^2 \theta
    \end{align*}
    where the second inequality is due to \eqref{eq:contracted expectation}. To bound (IV),
    \begin{align*}
        \text{(IV)}
        &\leq \bbE_{\bar\bbP^{k_e},\bar\pi}\left[\sum_{h\in [H]}\|\bar\pi_h(\cdot|s_h) - \widehat\pi_h^k(\cdot|s_h)\|_1  \|\widehat Q_{f,h}^k(s_h,\cdot)\|_\infty\right] \\
        &\leq \bbE_{\bar\bbP^{k_e},\bar\pi}\left[4H^2\right] \\
        &\leq 4H^2
    \end{align*}
    where the first inequality is due to H\"older's inequality, and the last inequality is due to \eqref{eq:contracted expectation}. 
    
    To bound (V), we observe the following. Let $\bar V_{g,1}^{\bar\pi}(s)$ denote the $\rho$-contracted value function, where $\rho(s,a,h) = \sigma(-\beta_w\|\phi(s,a)\|_{(\Lambda_h^{k_e})^{-1}} + \log K)$. By Lemmas \ref{lem:cassel 2} and \ref{lem:extended value diff}, we have
    \begin{align*}
        V_{g,1}^{\bar\pi}(s_1) - \widehat V_{g,1}^k(s_1) 
        &\geq \bar V_{g,1}^{\bar\pi}(s_1) - \widehat V_{g,1}^k(s_1) \\
        &= \bbE_{\bar\bbP^{k_e},\bar\pi}\left[\sum_{h\in[H]}\langle \bar\pi_h(\cdot\mid s_h) - \widehat\pi_h^k(\cdot\mid s_h) ,\widehat Q_{g,h}^k(\cdot\mid s_h) \rangle\right] \\
        &\quad+ \bbE_{\bar\bbP^{k_e},\bar\pi}\left[\sum_{h\in[H]} \bar g_h(s_h,a_h) + \sum_{s'\in\calS} \bar \bbP_h^{k_e}(s'\mid s_h, a_h)\widehat V_{g,h+1}^k(s') - \widehat Q_{g,h}^k(s_h,a_h)\right].
    \end{align*}
    To bound the latter term, we have for any $(s_h,a_h)\in\calS \times \calA$,
    \begin{align*}
        &\bar g_h(s_h,a_h) + \sum_{s'\in\calS} \bar \bbP_h(s'\mid s_h, a_h)\widehat V_{g,h+1}^k(s') - \widehat Q_{g,h}^k(s_h,a_h) \\
        &= \bar\phi_h^{k_e}(s_h,a_h)^\top \theta_{g,h} + \bar\phi_h^{k_e}(s_h,a_h)^\top\psi_h \widehat V_{g,h+1}^k - \bar\phi_h^{k_e}(s_h,a_h)^\top \left[\widehat\theta_{g,h}^k + \widehat\psi_h^k\widehat V_{g,h+1}^k\right] + \beta_b\|\bar\phi_h^{k_e}(s_h,a_h)\|_{(\Lambda_h^{k_e})^{-1}}\\
        &\geq - \|\theta_{g,h} - \widehat \theta_{g,h}^k\|_{\Lambda_h^{k_e}} \|\bar\phi_h^{k_e}(s_h,a_h)\|_{(\Lambda_h^{k_e})^{-1}} - \|(\psi_h- \widehat\psi_h^k)\widehat V_{g,h+1}^k\|_{\Lambda_h^{k_e}} \|\bar\phi_h^{k_e}(s_h,a_h)\|_{(\Lambda_h^{k_e})^{-1}} + \beta_b \|\bar\phi_h^{k_e}(s_h,a_h)\|_{(\Lambda_h^{k_e})^{-1}}\\
        &\geq - \|\theta_{g,h} - \widehat \theta_{g,h}^k\|_{\Lambda_h^{k}} \|\bar\phi_h^{k_e}(s_h,a_h)\|_{(\Lambda_h^{k_e})^{-1}} - \|(\psi_h- \widehat\psi_h^k)\widehat V_{g,h+1}^k\|_{\Lambda_h^{k}} \|\bar\phi_h^{k_e}(s_h,a_h)\|_{(\Lambda_h^{k_e})^{-1}} + \beta_b \|\bar\phi_h^{k_e}(s_h,a_h)\|_{(\Lambda_h^{k_e})^{-1}}\\
        &\geq -\beta_r \|\bar\phi_h^{k_e}(s_h,a_h)\|_{(\Lambda_h^{k_e})^{-1}} - \beta_p \|\bar\phi_h^{k_e}(s_h,a_h)\|_{(\Lambda_h^{k_e})^{-1}} + \beta_b \|\bar\phi_h^{k_e}(s_h,a_h)\|_{(\Lambda_h^{k_e})^{-1}}\\
        &= 0
    \end{align*}    
    where the first equality is due to the definition of contracted MDP, the first inequality is due to the Cauchy-Schwarz  inequality, the second inequality is due to $\Lambda_h^{k_e} \preceq \Lambda_h^k$, the third inequality is due to $E_g$, and the last equality is due to $\beta_b = \beta_r + \beta_p$. This implies that the latter term is nonnegative, as $\bbE_{\bar\bbP^{k_e},\pi}[0] = 0$. Thus, it follows that
    \begin{align*}
        \text{(V)} &= Y_k\bbE_{\bar\bbP^{k_e},\bar\pi}\left[\sum_{h\in[H]}\langle \bar\pi_h(\cdot\mid s_h) - \widehat\pi_h^k(\cdot\mid s_h) ,\widehat Q_{g,h}^k(\cdot\mid s_h) \rangle\right] \\
        &\leq Y_k(V_{g,1}^{\bar\pi}(s_1) - \widehat V_{g,1}^k(s_1)) \\
        &\leq Y_k(b-\gamma - \widehat V_{g,1}^k(s_1))
    \end{align*}
    where the last inequality is due to the Slater condition and $Y_k \geq0$. Finally, plugging the bounds on (III),(IV), and (V) into \eqref{eq:lyap 2.5}, we have
    \begin{align*}
        &Y_k\bbE_{\bar\bbP^{k_e},\bar\pi}\left[\sum_{h\in [H]} \langle \widehat\pi^{k+1}_h(\cdot|s_h) - \widehat\pi_h^{k}(\cdot|s_h), \widehat Q_{g,h}^k(s_h,\cdot)\rangle\right] \\
        &\leq \frac{1}{\alpha} \bbE_{\bar\bbP^{k_e},\bar\pi}\left[\sum_{h\in [H]}D(\bar\pi_h(\cdot|s_h) || \widetilde\pi_h^k(\cdot|s_h)) - D(\bar\pi_h(\cdot|s_h)||\widehat\pi_h^{k+1}(\cdot|s_h))\right]\\
        &\quad+ 2\alpha H^3 + 4H^2 \theta + 4H^2 + Y_k(b-\gamma - \widehat V_{g,1}^k(s_1)).
    \end{align*}
    Then we can bound (I) as follows.
    \begin{align*}
        \text{(I)} 
        &= Y_{k}\eta\left( \widehat V_{g,1}^{k}(s_1)-b - 4\alpha H^3(1+Y_k) - 4\theta H^2\right) \\
        &\leq Y_{k}\eta\left( \widehat V_{g,1}^{k}(s_1)-b + \bbE_{\bar\bbP^{k_e},\bar\pi}\left[\sum_{h\in [H]} \langle \widehat\pi^{k+1}_h(\cdot|s_h) - \widehat\pi_h^{k}(\cdot|s_h), \widehat Q_{g,h}^k(s_h,\cdot)\rangle\right]\right)\\
        &\leq Y_k\eta(\widehat V_{g,1}^{k}(s_1)-b) + \frac{\eta}{\alpha} \bbE_{\bar\bbP^{k_e},\bar\pi}\left[\sum_{h\in [H]}D(\bar\pi_h(\cdot|s_h) || \widetilde\pi_h^k(\cdot|s_h)) - D(\bar\pi_h(\cdot|s_h)||\widehat\pi_h^{k+1}(\cdot|s_h))\right]\\
        &\quad+ \eta(2\alpha H^3 + 4H^2 \theta + 4H^2) + \eta Y_k(b-\gamma - \widehat V_{g,1}^k(s_1))\\
        &=-\eta\gamma Y_k + \frac{\eta}{\alpha} \bbE_{\bar\bbP^{k_e},\bar\pi}\left[\sum_{h\in [H]}D(\bar\pi_h(\cdot|s_h) || \widetilde\pi_h^k(\cdot|s_h)) - D(\bar\pi_h(\cdot|s_h)||\widehat\pi_h^{k+1}(\cdot|s_h))\right]\\
        &\quad+\eta(2\alpha H^3 + 4H^2 \theta + 4H^2).
    \end{align*}
    Here, the first inequality is true as follows. For any $s_1,\ldots, s_H\in\calS$,
    \begin{align*}
        \left|\sum_{h\in [H]} \langle \widehat\pi_h^{k+1}(\cdot|s_h) - \widehat\pi_h^{k}(\cdot|s_h), \widehat Q_{g,h}^k(s_h,\cdot) \rangle\right| 
        &\leq \sum_{h\in [H]} \left|\langle \widehat\pi_h^{k+1}(\cdot|s_h) - \widehat\pi_h^{k}(\cdot|s_h), \widehat Q_{g,h}^k(s_h,\cdot) \rangle\right|\\
        &\leq \sum_{h\in [H]} (4\alpha H^2(1+Y_k) + 4\theta H)\\
        &= 4\alpha H^3(1+Y_k) + 4\theta H^2
    \end{align*}
    where the first inequality is due to the triangle inequality, and the second inequality is due to the second statement of \Cref{lem:policy inequality}. Note that the second inequality holds regardless of whether $\widetilde\pi_h^k$ is perturbed, because when $\widetilde\pi_h^k$ is not perturbed, it can be viewed as $\theta=0$. It follows that
    \begin{align*}
        \bbE_{\bar\bbP^{k_e}, \bar\pi}\left[\sum_{h\in [H]} \langle \widehat\pi_h^{k+1}(\cdot|s_h) - \widehat\pi_h^{k}(\cdot|s_h), \widehat Q_{g,h}^k(s_h,\cdot) \rangle\right] 
        &\geq -\bbE_{\bar\bbP^{k_e}, \bar\pi}[4\alpha H^3(1+Y_k) + 4\theta H^2]\\
        &\geq -(4\alpha H^3(1+Y_k) + 4\theta H^2)
    \end{align*}
    where the second inequality is due to \eqref{eq:contracted expectation}. 
    
    Consequently, plugging the bounds on (I) and (II) into \eqref{eq:lyap 1}, the Lyapunov drift is bounded as
    \begin{align*}
        \frac{Y_{k+1}^2 - Y_k^2}{2} 
        &\leq -\eta\gamma Y_k + \frac{\eta}{\alpha} \bbE_{\bar\bbP^{k_e},\bar\pi}\left[\sum_{h\in [H]}D(\bar\pi_h(\cdot|s_h) || \widetilde\pi_h^k(\cdot|s_h)) - D(\bar\pi_h(\cdot|s_h)||\widehat\pi_h^{k+1}(\cdot|s_h))\right]\\
        &\quad+\eta(2\alpha H^3 + 4H^2 \theta + 4H^2) + 2\eta^2(9H^2 + 16\alpha^2 H^6 + 1936\alpha^2 \eta^2 H^{12}K^2  + 16\theta^2 H^4).
    \end{align*}
\end{proof}
\begin{lemma}[Restatement of \Cref{lem:Yk bound main-text}]\label{lem:Yk bound}
    Assume that the good event $E_g$ holds. Let $H^{2}\leq K$. For $k\in [K]$, we have
    \begin{align*}
        Y_k \leq \frac{2C_4}{K^B\eta\gamma}+2K^B \delta_\max.
    \end{align*}
    where $\delta_\max$ and $C_4$ are given in \eqref{eq:lem:lyap 3:delta max} and \eqref{eq:lem:lyap 3:C4}, respectively. Furthermore, under the parameter choice of \Cref{alg:main-linear}, we have
    \begin{align*}
        Y_{k}&=\widetilde\calO(H^2/\gamma).
    \end{align*} 
\end{lemma}
\begin{proof}
    For ease of notation, let $N_e = \max\{n\in\bbZ_+: k_e+nK^B \in K_e\}$, and let
    \[
        Z_n = Y_{k_e + nK^B}.
    \]
    Fix $e \in E$. We first upper bound $Z_n$ for $0 \leq  n \leq  N_e$. Note that $|\max\{z_1,0\} - z_2| \leq |z_1-z_2|$ for any $z_1 \in \bbR$ and $z_2 \in \bbR_+$. Then it follows that
    \begin{align}\label{eq:lem:lyap 3:delta max}
    \begin{aligned}
        |Y_{k+1} - Y_k|
        &\leq\left|-4\alpha \eta H^3 Y_k + \eta(\widehat V_{g,1}^k(s_1) - b - 4\alpha H^3 - 4\theta H^2)\right|\\
        &\leq  4\eta\alpha H^3(11\eta H^3 K) + 3\eta H + 4\eta\alpha H^3 + 4\eta\theta H^2\\
        &:= \delta_\max.
    \end{aligned}
    \end{align}
    where the second inequality is due to the triangle inequality and the fact that $Y_k \leq 11\eta H^3 K$ and $\|\widehat V_{g,1}^k\|_\infty \leq 2H$ under $E_g$. Thus, by the triangle inequality,
    \begin{equation}\label{eq:lem:lyap 3:KB delta max}
        |Z_{n+1} - Z_n|= \left|\sum_{\tau = k_e+nK^B}^{k_e+(n+1)K^B-1} (Y_{\tau+1} - Y_\tau)\right| \leq K^B\delta_\max.
    \end{equation}
    By \Cref{lem:lyap 1}, we deduce the Lyapunov drift of $Z_n$ as
    \begin{align}\label{eq:lem:lyap 3:drift n0 decomp}
    \begin{aligned}
        \frac{Z_{n+1}^2 - Z_{n}^2}{2}
        &=\sum_{\tau=k_e+nK^B}^{k_e+(n+1)K^B-1}\frac{Y_{\tau+1}^2 - Y_\tau^2}{2}\\
        &\leq \sum_{\tau = k_e+nK^B}^{k_e+(n+1)K^B-1} -\eta\gamma Y_\tau\\
        &\quad+ \underbrace{\sum_{\tau = k_e+nK^B}^{k_e+(n+1)K^B-1}\frac{\eta}{\alpha} \bbE_{\bar\bbP^{k_e},\bar\pi}\left[\sum_{h\in [H]}D(\bar\pi_h(\cdot|s_h) || \widetilde\pi_h^\tau(\cdot|s_h)) - D(\bar\pi_h(\cdot|s_h)||\widehat\pi_h^{\tau+1}(\cdot|s_h))\right]}_{\text{(I)}}\\
        &\quad+ K^B C_3
    \end{aligned}
    \end{align}
    where
    \begin{align}\label{eq:C3}
        C_3 = \eta(2\alpha H^3 + 4H^2 \theta + 4H^2) + 2\eta^2(9H^2 + 16\alpha^2 H^6 + 1936\alpha^2 \eta^2 H^{12}K^2  + 16\theta^2 H^4)
    \end{align}
    To bound (I),
    \begin{align*}
        \text{(I)} 
        &= \frac{\eta}{\alpha}\sum_{h\in [H]} \bbE_{\bar\bbP^{k_e},\bar\pi}\left[\sum_{\tau=k_e+nK^B}^{k_e+(n+1)K^B-1} D(\bar\pi_h(\cdot|s_h) || \widetilde\pi_h^\tau(\cdot|s_h)) - D(\bar\pi_h(\cdot|s_h)||\widehat\pi_h^{\tau+1}(\cdot|s_h))\right]\\
        &= \frac{\eta}{\alpha}\sum_{h\in [H]} \bbE_{\bar\bbP^{k_e},\bar\pi}\left[D(\bar\pi(\cdot|s_h)||\widetilde\pi_h^{k_e+nK^B}(\cdot|s_h))-D(\bar\pi(\cdot|s_h)||\widehat\pi_h^{k_e+(n+1)K^B}(\cdot|s_h))\right]\\
        &\quad+ \frac{\eta}{\alpha}\sum_{h\in [H]} \bbE_{\bar\bbP^{k_e},\bar\pi}\left[\sum_{\tau=k_e+nK^B+1}^{k_e+(n+1)K^B-1} D(\bar\pi_h(\cdot|s_h) || \widetilde\pi_h^\tau(\cdot|s_h)) - D(\bar\pi_h(\cdot|s_h)||\widehat\pi_h^{\tau}(\cdot|s_h))\right]\\
        &= \frac{\eta}{\alpha}\sum_{h\in [H]} \bbE_{\bar\bbP^{k_e},\bar\pi}\left[D(\bar\pi(\cdot|s_h)||\widetilde\pi_h^{k_e+nK^B}(\cdot|s_h))-D(\bar\pi(\cdot|s_h)||\widehat\pi_h^{k_e+(n+1)K^B}(\cdot|s_h))\right]\\
        &\leq \frac{\eta}{\alpha}\sum_{h\in [H]} \bbE_{\bar\bbP^{k_e},\bar\pi}\left[\log(|\calA|/ \theta)\right]\\
        &\leq \frac{\eta}{\alpha}H\log(|\calA|/ \theta)
    \end{align*}
    where the last equality is because $\widetilde \pi_h^\tau = \widehat\pi_h^\tau$ for all $\tau$ such that $\tau - k_e \not\equiv 0 \mod K^B$ by algorithm. The first inequality is because the KL divergence is nonnegative, and we apply \Cref{lem:KL mixing}, as we know $\widetilde\pi_h^{k_e+nK^B} = (1-\theta)\widehat\pi_h^{k_e+nK^B} + \theta\piunif$ by algorithm. The last inequality is due to \eqref{eq:contracted expectation}. Then we have
    \begin{align}\label{eq:lem:lyap 3:drift n0}
        \frac{Z_{n+1}^2 - Z_{n}^2}{2} 
        &\leq -\eta\gamma\sum_{\tau = k_e+nK^B}^{k_e+(n+1)K^B-1}  Y_\tau + C_4
    \end{align}
    where
    \begin{equation}\label{eq:lem:lyap 3:C4}
        C_4 = \frac{\eta}{\alpha}H\log(|\calA|/ \theta) + K^B C_3
    \end{equation}
    Suppose that there exists $n \in \{0,\ldots,N_e\}$ such that $Z_n > \frac{2C_4}{K^B\eta\gamma} + K^B \delta_\max$. Then we can define the first $n$ that exceeds the threshold $\frac{2C_4}{K^B\eta\gamma}+ K^B \delta_\max$ as
    \[
        n_{\text{hit}} = \min\left\{n \in \{0,\ldots,N_e\}:  Z_n > \frac{2C_4}{K^B\eta\gamma}+K^B \delta_\max\right\}.
    \]
    Note that $n_{\text{hit}}\neq 0$, as we set $Y_{k_e} = 0$. Since $n_{\text{hit}}$ is the first time, we have $Z_{n_{\text{hit}}-1} \leq \frac{2C_4}{K^B\eta\gamma}+ K^B \delta_\max$. Moreover, we have $Y_{k_e + (n_{\text{hit}}-1)K^B}, \ldots , Y_{k_e + n_{\text{hit}}K^B -1} > \frac{2C_4}{K^B\eta\gamma}$. If not, $Z_{n_{\text{hit}}} = Y_{k_e + n_{\text{hit}}K^B}$ cannot be larger than $\frac{2C_4}{K^B\eta\gamma}+K^B \delta_\max$. By \eqref{eq:lem:lyap 3:drift n0}, this implies that
    \begin{align*}
        \frac{Z_{n_{\text{hit}}}^2 - Z_{n_{\text{hit}}-1}^2}{2}\leq -\eta\gamma\sum_{\tau = k_e+(n_{\text{hit}}-1)K^B}^{k_e+n_{\text{hit}}K^B-1}  Y_\tau + C_4 \leq -\eta\gamma K^B \frac{2C_4}{K^B\eta\gamma} + C_4 = -C_4 <0.
    \end{align*}
    This implies that $Z_{n_{\text{hit}}} < Z_{n_{\text{hit}}-1} < \frac{2C_4}{K^B\eta\gamma}+ K^B \delta_\max$. This contradicts the definition of $n_{\text{hit}}$. Therefore, we have for all $e\in[E]$ and $n \in \{0,\ldots,N_e\}$,
    \begin{align*}
        Y_{k_e + nK^B} \leq \frac{2C_4}{K^B\eta\gamma}+K^B \delta_\max.
    \end{align*}
    Moreover, by \eqref{eq:lem:lyap 3:delta max}, we have for all $k\in [K]$,
    \begin{align*}
        Y_k \leq \frac{2C_4}{K^B\eta\gamma}+2K^B \delta_\max.
    \end{align*}
    This completes the first statement. Next, we carefully plug our parameter choice into the upper bound on $Y_k$. Recall that the definitions of $C_3,C_4,\delta_\max$, and our parameter choice such that
    \begin{align*}
        B=\frac{3}{4}, \eta = H^{-2}K^{-B}, \alpha = H^{-1}K^{-B}, \theta = K^{-1}.
    \end{align*}
    $\delta_\max$ is bounded as
    \begin{align*}
        \delta_\max 
        &= 4\eta\alpha H^3(11\eta H^3 K) + 3\eta H + 4\eta\alpha H^3 + 4\eta\theta H^2 \\
        &= 44\eta^2\alpha H^6 K + 3\eta H + 4\eta\alpha H^3 + 4\eta \theta H^2\\
        &=44HK^{-3B+1} + 3H^{-1}K^{-B} + 4K^{-2B}+4K^{-1-B}\\
        &=\widetilde\calO\left(HK^{-5/4} + H^{-1}K^{-3/4}\right)
    \end{align*}
    Since we assumed $H^2\leq K$, it follows that $HK^{-1/2} \leq 1$. Then we have
    \begin{align*}
        \delta_\max =\widetilde\calO\left(K^{-3/4}\right)
    \end{align*}
    $C_3$ is bounded as
    \begin{align*}
        C_3 
        &= \eta(2\alpha H^3 + 4H^2 \theta + 4H^2) + 2\eta^2(9H^2 + 16\alpha^2 H^6 + 1936\alpha^2 \eta^2 H^{12}K^2  + 16\theta^2 H^4)\\
        &=2K^{-2B} + 4K^{-1-B} + 4K^{-B} + 18H^{-2}K^{-2B} + 32K^{-4B} + 3872H^2K^{2-6B} + 32K^{-2-2B}.
    \end{align*}
    $2C_4/(\eta K^B \gamma)$ is bounded as
    \begin{align*}
        \frac{2C_4}{\eta K^B \gamma} 
        &\leq \frac{2H\log(|\calA|/\theta)}{K^B\gamma \alpha} + \frac{2C_3}{\eta\gamma}\\
        &=\frac{2}{\gamma}H^2\log(|\calA|K)\\
        &\quad+\frac{2}{\gamma}(2H^2K^{-B} + 4H^2K^{-1} + 4H^2 + 18K^{-B} + 32H^2K^{-3B} + 3872H^4K^{2-5B} + 32H^2K^{-2-B})\\
        &=\widetilde\calO\left(H^2/\gamma + H^4K^{-7/4}/\gamma\right).
    \end{align*}
    Since we assumed $H^{2}\leq K$, we have $H^4K^{-7/4} \leq H^2$. Then we can drop $\widetilde\calO(H^4K^{-7/4}/\gamma)$.
    $2K^B\delta_\max$ is bounded as
    \begin{align*}
        2K^B\delta_\max &= 88HK^{-2B+1} + 6H^{-1} + 8K^{-2B}+8K^{-1}\\
        &=\widetilde\calO\left(HK^{-1/2}\right)
    \end{align*}
    Finally, $Y_k$ is bounded as
    \begin{align*}
        Y_k = \widetilde\calO\left(H^2/\gamma\right).
    \end{align*}
\end{proof}

\section{Detailed Proofs for the Analysis}\label{appendix:main resutls}
In this section, we first introduce lemmas, which bound an online mirror descent term and optimism terms, and these are useful to prove \Cref{lem:omd+optimism}. Then we present the proofs of Lemmas \ref{lem:cost of opt}, \ref{lem:proj gd}, \ref{lem:omd+optimism}, \ref{lem:violation optimism}.  Then we conclude the section by providing the proof of \Cref{thm:main}.

The following lemma is to bound the regret due to online mirror descent. Here, the main difference with the standard online mirror descent lemma (e.g., \citet{hazan2016introduction, lattimore2020bandit}) comes from the periodic policy mixing, which requires a modified analysis.
\begin{lemma}\label{lem:omd with mixing}
Let $H^2\leq K$. Suppose that $E_g$ holds. Then we have
    \begin{align*}
        &\sum_{e\in E}\sum_{k\in K_e} \bbE_{\bar\bbP^{k_e},\pi^*}\left[\sum_{h\in [H]}\langle\widehat Q_{f,h}^{k}(s_h,\cdot) + Y_k \widehat Q_{g,h}^k(s_h,\cdot), \widehat\pi_h^{k}(\cdot\mid s_h) - \pi_h^*(\cdot\mid s_h) \rangle\right] \\
         &=\widetilde\calO \left(dH^3 K^{3/4} + \frac{H^6}{\gamma^2}K^{1/4} + \frac{dH^5}{\gamma}\right).
    \end{align*}
\end{lemma}
\begin{proof}
    Consider $e\in E$. For any $s\in\calS$ and $k\in K_e$ such that $\widehat\pi_h^{k+1}\neq \piunif$ (i.e., $k=k_e,\ldots,k_{e+1}-2$), we have
    \begin{align*}
        \widehat\pi_h^{k+1}(\cdot\mid s) = \argmin_{\pi(\cdot\mid s)\in\Delta(\calA)}\ \langle\widehat Q_{f,h}^{k}(s,\cdot) + Y_k \widehat Q_{g,h}^k(s,\cdot), \pi(\cdot\mid s) \rangle + \frac{1}{\alpha} D(\pi(\cdot\mid s)||\widetilde\pi_h^k(\cdot\mid s)).
    \end{align*}
    For ease of notation, we omit $(s,\cdot)$ and $(\cdot\mid s)$. By \Cref{lem:pushback}, for any policy $\pi$,
    \begin{align*}
        \langle\widehat Q_{f,h}^{k} + Y_k \widehat Q_{g,h}^k, \widehat\pi_h^{k+1} - \pi \rangle \leq \frac{1}{\alpha}D(\pi||\widetilde\pi_h^k) - \frac{1}{\alpha}D(\pi||\widehat\pi_h^{k+1}) - \frac{1}{\alpha}D(\widehat\pi_h^{k+1}||\widetilde\pi_h^k).
    \end{align*}
    By adding $\langle\widehat Q_{f,h}^{k} + Y_k \widehat Q_{g,h}^k, \widehat\pi_h^k - \widehat\pi_h^{k+1} \rangle$ on both sides, we have for $k=k_e,\ldots, k_{e+1}-2$,
    \begin{align}\label{eq:omd eq 1}
    \begin{aligned}
        &\langle\widehat Q_{f,h}^{k} + Y_k \widehat Q_{g,h}^k, \widehat\pi_h^{k} - \pi \rangle\\
        &\leq \frac{1}{\alpha}D(\pi||\widetilde\pi_h^k) - \frac{1}{\alpha}D(\pi||\widehat\pi_h^{k+1}) - \frac{1}{\alpha}D(\widehat\pi_h^{k+1}||\widetilde\pi_h^k) + \langle\widehat Q_{f,h}^{k} + Y_k \widehat Q_{g,h}^k, \widehat\pi_h^k - \widehat\pi_h^{k+1} \rangle\\
        &\leq \frac{1}{\alpha}D(\pi||\widetilde\pi_h^k) - \frac{1}{\alpha}D(\pi||\widehat\pi_h^{k+1}) - \frac{1}{2\alpha}\|\widehat\pi_h^{k+1}-\widetilde\pi_h^k\|_1^2 + \|\widehat Q_{f,h}^{k} + Y_k \widehat Q_{g,h}^k\|_\infty \|\widehat\pi_h^k - \widehat\pi_h^{k+1}\|_1\\
        &\leq \frac{1}{\alpha}D(\pi||\widetilde\pi_h^k) - \frac{1}{\alpha}D(\pi||\widehat\pi_h^{k+1}) - \frac{1}{2\alpha}\|\widehat\pi_h^{k+1}-\widetilde\pi_h^k\|_1^2 + \|\widehat Q_{f,h}^{k} + Y_k \widehat Q_{g,h}^k\|_\infty \|\widehat\pi_h^{k+1} - \widetilde\pi_h^k \|_1 \\
        &\quad+\|\widehat Q_{f,h}^{k} + Y_k \widehat Q_{g,h}^k\|_\infty \|\widetilde\pi_h^k - \widehat \pi_h^k\|_1\\
        &= \frac{1}{\alpha}D(\pi||\widetilde\pi_h^k) - \frac{1}{\alpha}D(\pi||\widehat\pi_h^k) + \frac{1}{\alpha}D(\pi||\widehat\pi_h^k) - \frac{1}{\alpha}D(\pi||\widehat\pi_h^{k+1})\\
        &\quad- \frac{1}{2\alpha}\|\widehat\pi_h^{k+1}-\widetilde\pi_h^k\|_1^2 + \|\widehat Q_{f,h}^{k} + Y_k \widehat Q_{g,h}^k\|_\infty \|\widehat\pi_h^{k+1} - \widetilde\pi_h^k \|_1 + \|\widehat Q_{f,h}^{k} + Y_k \widehat Q_{g,h}^k\|_\infty \|\widetilde\pi_h^k - \widehat \pi_h^k\|_1
    \end{aligned}
    \end{align}
    where the second inequality follows from Pinsker's inequality and H\"older's inequality, and the last inequality is due to the triangle inequality. By taking $\sum_{k\in K_e}$ on both sides, we have
    \begin{align*}
        &\sum_{k\in K_e}\langle\widehat Q_{f,h}^{k} + Y_k \widehat Q_{g,h}^k, \widehat\pi_h^{k} - \pi \rangle\\
        &\leq \underbrace{\frac{1}{\alpha} \sum_{k=k_e}^{k_{e+1-2}}\left(D(\pi||\widetilde\pi_h^k) - D(\pi||\widehat\pi_h^k)\right)}_{\text{(I)}} + \underbrace{\frac{1}{\alpha}\sum_{k=k_e}^{k_{e+1-2}}\left(D(\pi||\widehat\pi_h^k) - D(\pi||\widehat\pi_h^{k+1})\right)}_{\text{(II)}}\\
        &\quad+ \underbrace{\sum_{k=k_e}^{k_{e+1-2}}\left( - \frac{1}{2\alpha}\|\widehat\pi_h^{k+1}-\widetilde\pi_h^k\|_1^2 + \|\widehat Q_{f,h}^{k} + Y_k \widehat Q_{g,h}^k\|_\infty \|\widehat\pi_h^{k+1} - \widetilde\pi_h^k \|_1\right)}_{\text{(III)}} + \underbrace{\sum_{k=k_e}^{k_{e+1-2}}\|\widehat Q_{f,h}^{k} + Y_k \widehat Q_{g,h}^k\|_\infty \|\widetilde\pi_h^k - \widehat \pi_h^k\|_1}_{\text{(IV)}}\\
        &\quad+ \langle \widehat Q_{f,h}^{k_{e+1}-1} + Y_{k_{e+1}-1} \widehat Q_{g,h}^{k_{e+1}-1}, \widehat \pi_h^{k_{e+1}-1} - \pi\rangle.
    \end{align*}
    Note that $\langle \widehat Q_{f,h}^{k_{e+1}-1} + Y_{k_{e+1}-1} \widehat Q_{g,h}^{k_{e+1}-1}, \widehat \pi_h^{k_{e+1}-1} - \pi\rangle$ is added, as \eqref{eq:omd eq 1} does not holds for $k_{e+1}-1$, i.e., the last episode in epoch $e$. Furthermore, this term can be bounded by $2\|\widehat Q_{f,h}^{k_{e+1}-1} + Y_{k_{e+1}-1} \widehat Q_{g,h}^{k_{e+1}-1}\|_\infty$ using H\"older's inequality.
    
    To bound (I), we observe the following. If $k - k _e\not\equiv 0 \mod K^B$, then $\widetilde \pi_h^k = \widehat\pi_h^k$. Thus, $D(\pi||\widetilde\pi_h^k) - D(\pi||\widehat\pi_h^k) = 0$. Otherwise, since $\widetilde \pi_h^k = (1-\theta)\widehat\pi_h^k + \theta\piunif$, we can apply \Cref{lem:KL mixing}, and thus $D(\pi||\widetilde\pi_h^k) - D(\pi||\widehat\pi_h^k) \leq \theta\log |\calA|$, i.e.,
    \begin{align*}
        D(\pi||\widetilde\pi_h^k) - D(\pi||\widehat\pi_h^k) \leq \begin{cases}
            \theta \log |\calA| & \text{if $k-k_e \equiv 0 \mod K^B$},\\
            0 &\text{otherwise}.
        \end{cases}
    \end{align*}
    It follows that 
    \begin{align*}
        \text{(I)} \leq \frac{1}{\alpha}\sum_{k\in K_e: k-k_e\equiv 0 \text{ mod} K^B}  \theta \log |\calA| \leq \frac{\theta|K_e|\log|\calA|}{\alpha K^B} + \frac{\theta\log|\calA|}{\alpha}
    \end{align*}
    where the second inequality is due to $|\{k\in K_e: k-k_e \equiv 0\mod K^B\}| \leq \ceil{|K_e|/K^B} \leq |K_e|/K^B + 1$. Furthermore, since (II) is in the form of a telescoping sum and $\widehat \pi_h^{k_e} = \piunif$, we have
    \begin{align*}
        \text{(II)} \leq \frac{1}{\alpha}D(\pi||\widehat\pi_h^{k_e}) \leq \frac{\log|\calA|}{\alpha}.
    \end{align*}
    To bound (III), since $-ax^2 + bx \leq b^2/(4a)$ for $a,b,x \geq 0$, we have
    \begin{align*}
        \text{(III)} \leq \frac{\alpha|K_e| C_5^2}{2}.
    \end{align*}
    where $C_5$ is a constant such that $\|\widehat Q_{f,h}^k + Y_k \widehat Q_{g,h}^k\|_\infty \leq C_5$ for all $h,k$. Again, by definition of $\widetilde\pi_h^k$, we have
    \begin{align*}
        \text{(IV)} 
        &= \sum_{k\in K_e: k-k_e\equiv 0 \text{ mod} K^B} \|\widehat Q_{f,h}^{k} + Y_k \widehat Q_{g,h}^k\|_\infty \|\widetilde\pi_h^k - \widehat \pi_h^k\|_1\\
        &= \sum_{k\in K_e: k-k_e\equiv 0 \text{ mod} K^B} \|\widehat Q_{f,h}^{k} + Y_k \widehat Q_{g,h}^k\|_\infty \theta\|\piunif - \widehat \pi_h^k\|_1\\
        &\leq \frac{2\theta |K_e| C_5}{ K^B} + 2\theta C_5.
    \end{align*} 
    Finally, we have for any policy $\pi$ and $s\in\calS$,
    \begin{align*}
        &\sum_{k\in K_e}\langle\widehat Q_{f,h}^{k}(s,\cdot) + Y_k \widehat Q_{g,h}^k(s,\cdot), \widehat\pi_h^{k}(\cdot\mid s) - \pi(\cdot\mid s) \rangle \\
        &\leq \frac{\theta|K_e|\log|\calA|}{\alpha K^B} + \frac{(1+\theta)\log|\calA|}{\alpha}+\frac{\alpha|K_e| C_5^2}{2}+\frac{2\theta |K_e| C_5}{\alpha K^B} + 2\theta C_5 + 2C_5.
    \end{align*}
     Let us take $\pi = \pi_h^*$ for each $h\in [H]$. Then, by taking $\sum_{h\in [H]}$ and $\bbE_{\bar\bbP^{k_e},\pi^*}$, it follows that
     \begin{align*}
         &\bbE_{\bar\bbP^{k_e},\pi^*}\left[\sum_{h\in [H]}\sum_{k\in K_e}\langle\widehat Q_{f,h}^{k}(s_h,\cdot) + Y_k \widehat Q_{g,h}^k(s_h,\cdot), \widehat\pi_h^{k}(\cdot\mid s_h) - \pi_h^*(\cdot\mid s_h) \rangle\right] \\
        &\leq \bbE_{\bar\bbP^{k_e},\pi^*}\left[H\left(\frac{\theta|K_e|\log|\calA|}{\alpha K^B} + \frac{(1+\theta)\log|\calA|}{\alpha}+\frac{\alpha|K_e| C_5^2}{2}+\frac{2\theta |K_e| C_5}{K^B} + 2\theta C_5 + 2C_5\right)\right]\\
        &\leq H\left(\frac{\theta|K_e|\log|\calA|}{\alpha K^B} + \frac{(1+\theta)\log|\calA|}{\alpha}+\frac{\alpha|K_e| C_5^2}{2}+\frac{2\theta |K_e| C_5}{K^B} + 2\theta C_5 + 2C_5\right).
     \end{align*}
     Finally, by $E_g$ and \Cref{lem:Yk bound}, we have $C_5=\widetilde\calO(H^3/\gamma)$. Furthermore, by \Cref{lem:cassel 8 num of epoch}, the number of epochs is at most $\widetilde\calO(dH)$. Then, by taking $\sum_{e\in E}$ to the above inequality, it follows that
     \begin{align*}
         &\sum_{e\in E}\sum_{k\in K_e} \bbE_{\bar\bbP^{k_e},\pi^*}\left[\sum_{h\in [H]}\langle\widehat Q_{f,h}^{k}(s_h,\cdot) + Y_k \widehat Q_{g,h}^k(s_h,\cdot), \widehat\pi_h^{k}(\cdot\mid s_h) - \pi_h^*(\cdot\mid s_h) \rangle\right] \\
         &=\widetilde\calO \left(dH^3 K^{3/4} + \frac{H^6}{\gamma^2}K^{1/4} + \frac{dH^5}{\gamma}\right).
     \end{align*}
\end{proof}

The next lemma claims that the regret terms associated with optimism are nonpositive, highlighting the effectiveness of our bonus terms. We closely follow the proof of Lemma 4 of \citet{cassel2024warm}.
\begin{lemma}\label{lem:optimism} 
Let $H^2\leq K$. Suppose that $E_g$ holds. For all  $(s,a,h,k, \ell)\in\calS\times\calA\times[H]\times [K] \times \{f,g\}$, 
    \[
        \widehat Q_{\ell,h}^k(s,a) - \bar \phi_h^{k_e}(s,a)^\top(\theta_{\ell,h}^k + \psi_h \widehat V_{\ell,h+1}^k)\leq 0
    \]
\end{lemma}
\begin{proof}
    By definition, we have
    \begin{align*}
        &\widehat Q_{\ell,h}^k(s,a) - \bar \phi_h^{k_e}(s,a)^\top(\theta_{\ell,h}^k + \psi_h \widehat V_{\ell,h+1}^k)\\ 
        &= \bar \phi_h^{k_e}(s,a)^\top(\widehat \theta_{\ell,h}^k - \theta_{\ell,h}^k) + \bar \phi_h^{k_e}(s,a)^\top(\widehat\psi_h^k -\psi_h) \widehat V_{\ell,h+1}^k - \beta_b \|\bar\phi_h^{k_e}(s,a)\|_{(\Lambda_h^{k_e})^{-1}}\\
        &\leq \|\widehat \theta_{\ell,h}^k - \theta_{\ell,h}^k\|_{\Lambda_h^{k_e}}\|\bar \phi_h^{k_e}(s,a)\|_{(\Lambda_h^{k_e})^{-1}} + \|(\widehat\psi_h^k -\psi_h) \widehat V_{\ell,h+1}^k\|_{\Lambda_h^{k_e}}\|\bar \phi_h^{k_e}(s,a)\|_{(\Lambda_h^{k_e})^{-1}} - \beta_b \|\bar\phi_h^{k_e}(s,a)\|_{(\Lambda_h^{k_e})^{-1}}\\
        &\leq \|\widehat \theta_{\ell,h}^k - \theta_{\ell,h}^k\|_{\Lambda_h^{k}}\|\bar \phi_h^{k_e}(s,a)\|_{(\Lambda_h^{k_e})^{-1}} + \|(\widehat\psi_h^k -\psi_h) \widehat V_{\ell,h+1}^k\|_{\Lambda_h^{k}}\|\bar \phi_h^{k_e}(s,a)\|_{(\Lambda_h^{k_e})^{-1}} - \beta_b \|\bar\phi_h^{k_e}(s,a)\|_{(\Lambda_h^{k_e})^{-1}}\\
        &\leq \beta_r \|\bar \phi_h^{k_e}(s,a)\|_{(\Lambda_h^{k_e})^{-1}} + \beta_p \|\bar \phi_h^{k_e}(s,a)\|_{(\Lambda_h^{k_e})^{-1}} - \beta_b \|\bar\phi_h^{k_e}(s,a)\|_{(\Lambda_h^{k_e})^{-1}}\\
        &=0
    \end{align*}
    where the first inequality is due to the Cauchy-Schwarz inequality, the second inequality follows from $\Lambda_h^{k_e} \preceq \Lambda_h^{k}$, and the last equality is because $\beta_b = \beta_r + \beta_p$.
\end{proof}

\paragraph{Proof of \Cref{lem:cost of opt}}
    By \Cref{lem:extended value diff},
    \begin{align*}
        V_{\ell,1}^{\pi^k}(s_1) - \widehat V_{\ell,1}^{k}(s_1) 
        &= \bbE_{\bbP,\widehat\pi^k}\left[\sum_{h\in[H]} \phi(s_h^k, a_h^k)^\top (\theta_{\ell,h}^k + \psi_h \widehat V_{\ell, h+1}^k) - \widehat Q_{\ell,h}^k(s_h^k, a_h^k)\right]\\
        &= \underbrace{\bbE_{\bbP,\widehat\pi^k}\left[\sum_{h\in[H]} \bar\phi_h^{k_e}(s_h^k, a_h^k)^\top (\theta_{\ell,h}^k + \psi_h \widehat V_{\ell, h+1}^k) - \widehat Q_{\ell,h}^k(s_h^k, a_h^k)\right]}_{\text{(I)}}\\
        &\quad+ \underbrace{\bbE_{\bbP,\widehat\pi^k}\left[\sum_{h\in[H]} (\phi(s_h^k,a_h^k)-\bar\phi_h^{k_e}(s_h^k, a_h^k))^\top (\theta_{\ell,h}^k + \psi_h \widehat V_{\ell, h+1}^k)\right]}_{\text{(II)}}.
    \end{align*}
    To bound (I), we have
    \begin{align*}
        &\bar\phi_h^{k_e}(s_h^k, a_h^k)^\top (\theta_{\ell,h}^k + \psi_h \widehat V_{\ell, h+1}^k) - \widehat Q_{\ell,h}^k(s_h^k, a_h^k)\\ 
        &= \bar\phi_h^{k_e}(s_h^k, a_h^k)^\top (\theta_{\ell,h}^k - \widehat \theta_{\ell,h}^k) +\bar\phi_h^{k_e}(s_h^k, a_h^k)^\top(\psi_h - \widehat \psi_h^k)\widehat V_{\ell,h+1}^k + \beta_b \|\bar \phi_h^{k_e}(s_h^k, a_h^k)\|_{(\Lambda_h^{k_e})^{-1}}\\
        &\leq \beta_b \|\bar \phi_h^{k_e}(s_h^k,a_h^k)\|_{(\Lambda_h^k)^{-1}} + \beta_b \|\bar \phi_h^{k_e}(s_h^k,a_h^k)\|_{(\Lambda_h^{k_e})^{-1}}
    \end{align*}
    where the inequality is due to the Cauchy-Schwarz inequality. Furthermore, since $k\in K_e$, it must hold $\det (\Lambda_h^k) \leq 2\det (\Lambda_h^{k_e})$, otherwise $k$ would belong to epoch $e+1$. As it is obvious that $(\Lambda_h^k)^{-1} \preceq (\Lambda_h^{k_e})^{-1}$, we can apply \Cref{lem:abbasi 12} for nonzero $\bar\phi_h^{k_e}(s_h^k, a_h^k)$ as follows.
    \begin{align*}
        \frac{\|\bar\phi_h^{k_e}(s_h^k, a_h^k)\|^2_{(\Lambda_h^{k_e})^{-1}}}{\|\bar\phi_h^{k_e}(s_h^k, a_h^k)\|^2_{(\Lambda_h^k)^{-1}}} \leq \frac{\det ((\Lambda_h^{k_e})^{-1})}{\det ((\Lambda_h^k)^{-1})}\leq 2.
    \end{align*}
    This implies that $\beta_b \|\bar \phi_h^{k_e}(s_h^k,a_h^k)\|_{(\Lambda_h^{k_e})^{-1}} \leq 2\beta_b \|\bar \phi_h^{k_e}(s_h^k,a_h^k)\|_{(\Lambda_h^{k})^{-1}}$ for nonzero $\bar\phi_h^{k_e}(s_h^k, a_h^k)$. If $\bar\phi_h^{k_e}(s_h^k, a_h^k)=0$, then the inequality is trivial. Then it follows that
    \begin{align*}
        \text{(I)}\leq \bbE_{\bbP,\widehat\pi^k}\left[\sum_{h\in[H]}3\beta_b \|\bar \phi_h^{k_e}(s_h^k,a_h^k)\|_{(\Lambda_h^k)^{-1}}\right]\leq \bbE_{\bbP,\widehat\pi^k}\left[\sum_{h\in[H]}3\beta_b \|\phi(s_h^k,a_h^k)\|_{(\Lambda_h^k)^{-1}}\right].
    \end{align*}
    where the second inequality is due to $\|\bar \phi_h^{k_e}(s_h^k,a_h^k)\|_{(\Lambda_h^k)^{-1}} \leq \|\phi(s_h^k,a_h^k)\|_{(\Lambda_h^k)^{-1}}$. To bound (II), by \Cref{lem:cassel 3}, we have
    \begin{align*}
        &(\phi(s_h^k,a_h^k)-\bar\phi_h^{k_e}(s_h^k, a_h^k))^\top (\theta_{\ell,h}^k + \psi_h \widehat V_{\ell, h+1}^k)\\ 
        &\leq (4\beta_w^2 \|\phi(s_h^k, a_h^k)\|_{(\Lambda_h^k)^{-1}}^2 + 2K^{-1}) |\phi(s_h^k, a_h^k)^\top (\theta_{\ell,h}^k + \psi_h \widehat V_{\ell, h+1}^k)| \\
        &\leq 16H\beta_w^2 \|\phi(s_h^k, a_h^k)\|_{(\Lambda_h^k)^{-1}}^2 + 8H K^{-1}
    \end{align*}
    where the second inequality follows from the fact that for any $\ell\in\{f,g\}$,
    $$\phi(s_h^k,a_h^k)^\top (\theta_{\ell,h}^k + \psi_h \widehat V_{\ell, h+1}^k) = \ell_h^k(s_h^k,a_h^k) + \sum_{s'\in\calS} \bbP_h(s'\mid s_h^k,a_h^k) \widehat V_{\ell,h+1}^k(s') \leq 1 + \|\widehat V_{\ell, h+1}^k\|_\infty \leq 4H.$$
    This implies that (II)$\leq 16H\beta_w^2 \|\phi(s_h^k, a_h^k)\|_{(\Lambda_h^k)^{-1}}^2 + 8H K^{-1}$. Finally, we have
    \begin{align*}
        V_{\ell,1}^{\pi^k}(s_1) - \widehat V_{\ell,1}^{k}(s_1) 
        &\leq \bbE_{\bbP,\widehat\pi^k}\left[\sum_{h\in[H]}3\beta_b \|\phi(s_h^k,a_h^k)\|_{(\Lambda_h^k)^{-1}} + 16H\beta_w^2 \|\phi(s_h^k, a_h^k)\|_{(\Lambda_h^k)^{-1}}^2\right] + 8H^2 K^{-1}.
    \end{align*}    
    Taking $\sum_{k=1}^K$ on both sides,
    \begin{align*}
        \sum_{k=1}^K(V_{\ell,1}^{\pi^k}(s_1) - \widehat V_{\ell,1}^{k}(s_1))
        &\leq \sum_{k=1}^K\bbE_{\bbP,\widehat\pi^k}\left[\sum_{h\in[H]}3\beta_b \|\phi(s_h^k,a_h^k)\|_{(\Lambda_h^k)^{-1}} + 16H\beta_w^2 \|\phi(s_h^k, a_h^k)\|_{(\Lambda_h^k)^{-1}}^2\right] + 8H^2\\
        &\leq \sum_{k=1}^K \sum_{h\in[H]} \left(6\beta_b \|\phi(s_h^k,a_h^k)\|_{(\Lambda_h^k)^{-1}} + 32H\beta_w^2 \|\phi(s_h^k, a_h^k)\|_{(\Lambda_h^k)^{-1}}^2\right)\\
        &\quad+ 8H(3\beta_b + 16H\beta_w^2)\log\frac{6K}{\delta} + 8H^2
    \end{align*}
    where the second inequality follows from $E_g$.
    By \Cref{lem: jin d2}, we have
    \begin{align*}
        \sum_{k\in[K]} \|\phi(s_h^k,a_h^k)\|_{(\Lambda_h^k)^{-1}}^2 \leq 2\log\frac{\det(\Lambda_h^{K+1})}{\det(\Lambda_h^1)} \leq 2d\log (K+1)
    \end{align*}
    where the second inequality follows from $\|\Lambda_h^{k+1}\|_2 = \|I + \sum_{\tau\in[k]}\phi(s_h^\tau, a_h^\tau)\phi(s_h^\tau, a_h^\tau)^\top\|_2 \leq 1 + k$, and thus $\det(\Lambda_h^{K+1}) \leq (K+1)^d$. Furthermore, the Cauchy-Schwarz inequality implies that $\sum_{k\in[K]} \|\phi(s_h^k,a_h^k)\|_{(\Lambda_h^k)^{-1}} \leq \sqrt{2dK\log(K+1)}$. Then we deduce that
    \begin{align*}
    \begin{aligned}
        \sum_{k=1}^K(V_{\ell,1}^{\pi^k}(s_1) - \widehat V_{\ell,1}^{k}(s_1)) &\leq 6 \beta_b H\sqrt{2dK\log(K+1)} + 64dH^2\beta_w^2\log(K+1)\\
        &\quad+ 8H(3\beta_b + 16H\beta_w^2)\log\frac{6K}{\delta} + 8H^2\\
        &=\widetilde\calO\left(\sqrt{d^3H^4}K^{3/4} + d^3H^4 K^{1/2}\right)
    \end{aligned}
    \end{align*} 
    where the last equality follows from $\beta_b,\beta_w = \widetilde\calO\left(K^{1/4}dH\right).$
    \qed

\paragraph{Proof of \Cref{lem:proj gd}}
Given $e\in E$, for any $k\in K_e$, the dual variable $Y_k$ is updated as
$$
    Y_{k+1} =\begin{cases}
    0 & \text{if $k+1 = k_e$}, \\
    \left[(1-4\alpha\eta H^3)Y_{k} + \eta\left(\widehat V_{g,1}^{k}(s_1)-b -4\alpha H^3- 4\theta H^2\right)\right]_{+} & \text{otherwise}.
\end{cases}
$$
Then it follows that
\begin{align*}
\begin{aligned}
    0 &\leq Y_{k_{e+1}-1}^2 \\
    &= \sum_{k=k_e}^{k_{e+1}-2} \left(Y_{k+1}^2 - Y_{k}^2\right)\\
    &= \sum_{k=k_e}^{k_{e+1}-2} \left(\left[(1-4\alpha\eta H^3)Y_{k} + \eta\left(\widehat V_{g,1}^{k}(s_1)-b -4\alpha H^3- 4\theta H^2\right)\right]_{+}^2  - Y_{k}^2\right)\\
    &= \sum_{k=k_e}^{k_{e+1}-2} \left(\left[Y_{k} + \eta\left(\widehat V_{g,1}^{k}(s_1)-b -4\alpha H^3- 4\theta H^2 - 4\alpha H^3Y_k\right)\right]_{+}^2  - Y_{k}^2\right)\\
    &\leq \sum_{k=k_e}^{k_{e+1}-2} \left(2Y_k\eta\left(\widehat V_{g,1}^{k}(s_1)-b -4\alpha H^3- 4\theta H^2 - 4\alpha H^3Y_k\right) + \eta^2\left(\widehat V_{g,1}^{k}(s_1)-b -4\alpha H^3- 4\theta H^2 - 4\alpha H^3Y_k\right)^2 \right)
\end{aligned}
\end{align*}
where the first equality is due to $Y_{k_e}=0$, and the last inequality is due to the fact that $\max\{0,z\}^2 \leq z^2$ for any $z\in\bbR$. This can be rewritten as
\begin{align}\label{eq:lem:proj gd:0}
\begin{aligned}
    &\sum_{k=k_e}^{k_{e+1}-2} Y_k(b-\widehat V_{g,1}^k(s_1))\\
    &\leq \sum_{k=k_e}^{k_{e+1}-2} Y_k(-4\alpha H^3 - 4\theta H^2 - 4\alpha H^3 Y_k) + \frac{\eta}{2}\sum_{k=k_e}^{k_{e+1}-2}\left(\widehat V_{g,1}^{k}(s_1)-b -4\alpha H^3- 4\theta H^2 - 4\alpha H^3Y_k\right)^2.
\end{aligned}
\end{align}
Note that the first term is nonpositive. Furthermore, the second term can be bounded as
\begin{align*}
    |\widehat V_{g,1}^{k}(s_1)-b -4\alpha H^3- 4\theta H^2 - 4\alpha H^3Y_k| 
    &\leq |\widehat V_{g,1}^{k}(s_1)-b| + 4\alpha H^3 + 4\theta H^2 + 4\alpha H^3Y_k\\
    &\leq 3H + 4H^2K^{-3/4} + 4H^2K^{-1} + 4H^2K^{-3/4}Y_k\\
    &\leq 3H + 4H^{1/2} + 4 + 4H^2K^{-3/4}Y_k\\
    &\leq 11H + 4H^2K^{-3/4}Y_k
\end{align*}
where the second inequality follows from $E_g$, the third inequality is because we assumed that $H^2 \leq K$. Thus, \eqref{eq:lem:proj gd:0} is bounded as
\begin{align*}
    \sum_{k=k_e}^{k_{e+1}-2} Y_k(b-\widehat V_{g,1}^k(s_1))
    &\leq \frac{\eta}{2} \sum_{k=k_e}^{k_{e+1}-2}\left(11H + 4H^2K^{-3/4}Y_k\right)^2\\
    &\leq \frac{\eta}{2}\sum_{k=k_e}^{k_{e+1}-2}2(121H^2 + 16H^4K^{-3/2}Y_k^2)
\end{align*}
where the second inequality is due to the Cauchy-Schwarz inequality. Then we have
\begin{align*}
    \sum_{k\in K_e} Y_k(b-\widehat V_{g,1}^k(s_1)) 
    &= \sum_{k=k_e}^{k_{e+1}-2} Y_k(b-\widehat V_{g,1}^k(s_1)) + Y_{k_{e+1}-1}(b-\widehat V_{g,1}^{k_{e+1}-1}(s_1))\\
    &\leq \frac{\eta}{2}\sum_{k=k_e}^{k_{e+1}-2}2(121H^2 + 16H^4K^{-3/2}Y_k^2) + 3HY_{k_{e+1}-1}\\
    &\leq 121\eta H^2 |K_e| + 16\eta H^4 K^{-3/2}\sum_{k=k_e}^{k_{e+1}-2}Y_k^2 + 3HY_{k_{e+1}-1}.
\end{align*}
By \Cref{lem:Yk bound}, we have $Y_k = \widetilde\calO(H^2/\gamma)$ for all $k\in [K]$. Furthermore, by \Cref{lem:cassel 8 num of epoch}, the number of epochs is at most $\widetilde\calO(dH)$. By taking $\sum_{e\in E}$, it follows that
\begin{align*}
    \sum_{k\in [K]} Y_k(b-\widehat V_{g,1}^k(s_1)) 
    &= \sum_{e\in E}\sum_{k\in K_e} Y_k(b-\widehat V_{g,1}^k(s_1))\\
    &=\widetilde\calO\left(K^{1/4} + \frac{dH^6}{\gamma^2}\right).
\end{align*}
\qed

\paragraph{Proof of \Cref{lem:omd+optimism}}
Note that
\begin{align*}
    &\sum_{k=1}^K \left(\widehat V_{f,1}^k(s_1) + Y_k \widehat V_{g,1}^k(s_1) - V_{f^k,1}^{\pi^*}(s_1) - Y_k V_{g,1}^{\pi^*}(s_1)\right) \\
    &\leq \sum_{e\in E}\sum_{k\in K_e} \left(\widehat V_{f,1}^k(s_1) + Y_k \widehat V_{g,1}^k(s_1) - \bar V_{f^k,1}^{\pi^*}(s_1) - Y_k \bar V_{g,1}^{\pi^*}(s_1)\right)
\end{align*}
where $\bar V_{f^k,1}^{\pi^*}, \bar V_{g,1}^{\pi^*}$ are the value functions with respect to a contracted MDP.
Furthermore, by \Cref{lem:extended value diff}, it follows that
    \begin{align*}
        &\sum_{e\in E}\sum_{k\in K_e} \left(\widehat V_{f,1}^k(s_1) + Y_k \widehat V_{g,1}^k(s_1) - \bar V_{f^k,1}^{\pi^*}(s_1) - Y_k \bar V_{g,1}^{\pi^*}(s_1)\right)\\
        &= \sum_{e\in E} \sum_{k\in K_e}\bbE_{\bar \bbP^{k_e}, \pi^*}\left[\sum_{h\in[H]}\sum_{a\in\calA} (\widehat Q_{f,h}^k(s_h^k,a) + Y_k\widehat Q_{g,h}^k(s_h^k,a))(\widehat\pi_h^k(a\mid s_h^k)-\pi_h^*(a\mid s_h^k))\right]\\
        &\quad+\sum_{e\in E} \sum_{k\in K_e} \bbE_{\bar \bbP^{k_e}, \pi^*}\left[
            \sum_{h\in[H]}\widehat Q_{f,h}^k(s_h^k,a_h^k) - \bar f_h^{k}(s_h^k,a_h^k) - \sum_{s'\in\calS}\bar \bbP_h^{k_e}(s'\mid s_h^k, a_h^k)\widehat V_{f,h+1}^k(s')
        \right]\\
        &\quad+\sum_{e\in E} \sum_{k\in K_e}\bbE_{\bar \bbP^{k_e}, \pi^*}\left[
            \sum_{h\in[H]} Y_k\left(\widehat Q_{g,h}^k(s_h^k,a_h^k) -  \bar g_h(s_h^k, a_h^k) - \sum_{s'\in\calS}\bar \bbP_h^{k_e}(s'\mid s_h^k, a_h^k) \widehat V_{g,h+1}^k(s')\right)
        \right].
    \end{align*}
    Note that by the definition of the contracted MDP, we have $\bar f_h^{k}(s,a) = \bar\phi_{h}^{k_e}(s,a)^\top \theta_{f,h}^k,\ \bar g_h(s,a) = \bar\phi_{h}^{k_e}(s,a)^\top \theta_{g, h}$, and $\bar \bbP_h^{k_e}(s'\mid s,a) = \bar\phi_h^{k_e}(s,a)^\top \psi_h(s')$. Then it follows that
    \begin{align*}
        &\bar f_h(s_h^k, a_h^k) + \sum_{s'\in\cal S} \bar \bbP_h^{k_e}(s'\mid s_h^k,a_h^k) \widehat V_{f,h+1}^k(s') = \bar \phi_h^{k_e}(s_h^k, a_h^k)^\top(\theta_{f,h}^k + \psi_h \widehat V_{f,h+1}^k),\\
        &\bar g_h(s_h^k, a_h^k) + \sum_{s'\in\cal S} \bar \bbP_h^{k_e}(s'\mid s_h^k,a_h^k) \widehat V_{g,h+1}^k(s') = \bar \phi_h^{k_e}(s_h^k, a_h^k)^\top(\theta_{g,h} + \psi_h \widehat V_{g,h+1}^k).
    \end{align*}
    We deduce that
    \begin{align*}
        &\sum_{k=1}^K \left(\widehat V_{f,1}^k(s_1) + Y_k \widehat V_{g,1}^k(s_1) - V_{f^k,1}^{\pi^*}(s_1) - Y_k V_{g,1}^{\pi^*}(s_1)\right)\\
        &\leq\sum_{e\in E}\sum_{k\in K_e} \left(\widehat V_{f,1}^k(s_1) + Y_k \widehat V_{g,1}^k(s_1) - \bar V_{f^k,1}^{\pi^*}(s_1) - Y_k \bar V_{g,1}^{\pi^*}(s_1)\right)\\
        &= \underbrace{\sum_{e\in E} \sum_{k\in K_e}\bbE_{\bar \bbP^{k_e}, \pi^*}\left[\sum_{h\in[H]}\sum_{a\in\calA} (\widehat Q_{f,h}^k(s_h^k,a) + Y_k\widehat Q_{g,h}^k(s_h^k,a))(\widehat\pi_h^k(a\mid s_h^k)-\pi_h^*(a\mid s_h^k))\right]}_{\text{(I)}}\\
        &\quad+\underbrace{\sum_{e\in E} \sum_{k\in K_e}\bbE_{\bar \bbP^{k_e}, \pi^*}\left[
            \sum_{h\in[H]}\widehat Q_{f,h}^k(s_h^k,a_h^k) - \bar \phi_h^{k_e}(s_h^k, a_h^k)^\top(\theta_{f,h}^k + \psi_h \widehat V_{f,h+1}^k)
        \right]}_{\text{(II)}}\\
        &\quad+\underbrace{\sum_{e\in E} \sum_{k\in K_e}\bbE_{\bar \bbP^{k_e}, \pi^*}\left[
            \sum_{h\in[H]} Y_k\left(\widehat Q_{g,h}^k(s_h^k,a_h^k) -  \bar \phi_h^{k_e}(s_h^k, a_h^k)^\top(\theta_{g,h} + \psi_h \widehat V_{g,h+1}^k)\right)
        \right]}_{\text{(III)}}.
    \end{align*}
    By \Cref{lem:omd with mixing}, 
    \begin{align*}
        \text{(I)}=\widetilde\calO \left(dH^3 K^{3/4} + \frac{H^6}{\gamma^2}K^{1/4} + \frac{dH^5}{\gamma}\right).
    \end{align*}
    Note that $Y_k\geq 0$ for all $k$. Then, by \Cref{lem:optimism}, for any $(s,a,h,k)\in \calS\times\calA\times[H]\times[K]$, we have
    \begin{align*}
        &\widehat Q_{f,h}^k(s,a) - \bar \phi_h^{k_e}(s, a)^\top(\theta_{f,h}^k + \psi_h \widehat V_{f,h+1}^k)\leq 0,\\
        &Y_k\left(\widehat Q_{g,h}^k(s,a) -  \bar \phi_h^{k_e}(s, a)^\top(\theta_{g,h} + \psi_h \widehat V_{g,h+1}^k)\right)\leq 0.
    \end{align*}
    This implies that
    \begin{align*}
        \text{(II), (III)} \leq 0.
    \end{align*}
    Finally, we have
    \begin{align*}
        \sum_{k=1}^K \left(\widehat V_{f,1}^k(s_1) + Y_k \widehat V_{g,1}^k(s_1) - V_{f^k,1}^{\pi^*}(s_1) - Y_k V_{g,1}^{\pi^*}(s_1)\right) = \widetilde\calO \left(dH^3 K^{3/4} + \frac{H^6}{\gamma^2}K^{1/4} + \frac{dH^5}{\gamma}\right).
    \end{align*}
\qed

\paragraph{Proof of \Cref{lem:violation optimism}}
    Note that the dual update is
    \begin{align*}
        Y_{k+1} = \begin{cases}
            0 & \text{if $k+1=k_e$},\\
            \left[(1-4\alpha\eta H^3)Y_k + \eta \left(\widehat V_{g,1}^k(s_1) - b - 4\alpha H^3 - 4\theta H^2\right)\right]_+ & \text{otherwise}.
        \end{cases}
    \end{align*}
    Then it follows that for any $e\in E$,
    \begin{align*}
        Y_{k_{e+1} -1} &= \left[(1-4\alpha\eta H^3)Y_{k_{e+1}-2} + \eta \left(\widehat V_{g,1}^{k_{e+1}-2}(s_1) - b - 4\alpha H^3  - 4\theta H^2 \right)\right]_+\\
        &\geq(1-4\alpha\eta H^3)Y_{k_{e+1}-2} + \eta \left(\widehat V_{g,1}^{k_{e+1}-2}(s_1) - b - 4\alpha H^3  - 4\theta H^2 \right)\\
        &=Y_{k_{e+1}-2} + \eta \left(\widehat V_{g,1}^{k_{e+1}-2}(s_1) - b - 4\alpha H^3(1+Y_{k_{e+1}-2})  - 4\theta H^2 \right)\\
        &\quad \vdots\\
        &\geq Y_{k_{e}} + \eta \sum_{k=k_e}^{k_{e+1}-2}\left(\widehat V_{g,1}^{k}(s_1) - b - 4\alpha H^3(1+Y_{k})  - 4\theta H^2\right).
    \end{align*}
    Note that $Y_{k_e}=0$. Then we have
    \begin{align*}
        \sum_{k=k_e}^{k_{e+1}-1}\left(\widehat V_{g,1}^k(s_1) - b\right) 
        &= \sum_{k=k_e}^{k_{e+1}-2}\left(\widehat V_{g,1}^k(s_1) - b\right) + \left(\widehat V_{g,1}^{k_{e+1}-1}(s_1) - b\right)\\
        &\leq \frac{Y_{k_{e+1}-1}}{\eta} + \sum_{k=k_e}^{k_{e+1}-2}\left(4\alpha H^3 (1+Y_k) + 4\theta H^2\right) + 3H.
    \end{align*}
    By \Cref{lem:Yk bound}, we have $Y_k = \widetilde\calO(H^2/\gamma)$ for all $k\in [K]$. Then it follows that
    \[
        4\alpha H^3 (1+Y_k) + 4\theta H^2 =4H^2K^{-3/4}(1+Y_k) + 4H^2K^{-1} = \widetilde\calO\left(\frac{H^4}{\gamma}K^{-3/4} + H^2K^{-1}\right).
    \]
    
    Furthermore, by \Cref{lem:cassel 8 num of epoch}, the number of epochs is at most $\widetilde\calO(dH)$. By taking $\sum_{e\in E}$, it follows that
    \begin{align*}
        \sum_{e\in E}\sum_{k\in K_e}\left(\widehat V_{g,1}^k(s_1) - b\right) =\widetilde\calO\left(\frac{dH^5}{\gamma}K^{3/4} + \frac{H^4}{\gamma}K^{1/4}\right).
    \end{align*}
    \qed

\paragraph{Proof of \Cref{thm:main}} 
If $K < \beta_w$, we cannot use \Cref{lem:good event}. Nevertheless, in this case, we have the following upper bounds for regret and violation.
\begin{align*}
    &\Regret \leq HK < H\beta_w = \widetilde\calO\left(dH^2 K^{1/4}\right),\\
    &\Violation \leq HK < H\beta_w = \widetilde\calO\left(dH^2 K^{1/4}\right).
\end{align*}
Otherwise, it is trivial that the conditions of \Cref{lem:good event} hold, i.e., $E_g$ holds with probability at least $1-\delta$. Furthermore, under $E_g$, we have the upper bound on $Y_k$ as in \Cref{lem:Yk bound}, i.e., 
\begin{equation}\label{eq:proof regret:Yk}
    Y_k = \widetilde\calO\left(\frac{H^2}{\gamma}\right)    
\end{equation}
Thus, with probability at least $1-\delta$, $E_g$ and \eqref{eq:proof regret:Yk} hold.

Now, we begin with the proof of the regret upper bound. Note that an optimal policy $\pi^*$ satisfies $V_{g,1}^{\pi^*}(s_1) \leq b$. Since $Y_k\geq 0$ for all $k$, it follows that
\begin{align*}
    &\Regret\\
    &= \sum_{k=1}^K \left(V_{f^k,1}^{\pi^k}(s_1) - \widehat V_{f,1}^k(s_1)\right)+\sum_{k=1}^K Y_k(b-\widehat V_{g,1}^k(s_1))+
    \sum_{k=1}^K \left(\widehat V_{f,1}^k(s_1) + Y_k \widehat V_{g,1}^k(s_1) - V_{f^k,1}^{\pi^*}(s_1) - Y_k b\right)\\
    &\leq \underbrace{\sum_{k=1}^K \left(V_{f^k,1}^{\pi^k}(s_1) - \widehat V_{f,1}^k(s_1)\right)}_{\text{(I)}}+\underbrace{\sum_{k=1}^K Y_k(b-\widehat V_{g,1}^k(s_1))}_{\text{(II)}}+\underbrace{\sum_{k=1}^K \left(\widehat V_{f,1}^k(s_1) + Y_k \widehat V_{g,1}^k(s_1) - V_{f^k,1}^{\pi^*}(s_1) - Y_k V_{g,1}^{\pi^*}(s_1)\right)}_{\text{(III)}}
\end{align*}
By \Cref{lem:cost of opt},
\begin{align*}
    \text{(I)} 
    &= \widetilde\calO\left(\sqrt{d^3H^4}K^{3/4} + d^3H^4 K^{1/2}\right).
\end{align*}
By \Cref{lem:proj gd},
\begin{align*}
    \text{(II)} =\widetilde\calO\left(K^{1/4} + \frac{dH^6}{\gamma^2}\right).
\end{align*}
By \Cref{lem:omd+optimism},
\begin{align*}
    \text{(III)} &=\widetilde\calO \left(dH^3 K^{3/4} + \frac{H^6}{\gamma^2}K^{1/4} + \frac{dH^5}{\gamma}\right).
\end{align*}
Thus, we have the following regret upper bound.
\begin{align*}
    \Regret &= \widetilde\calO\left(\sqrt{d^3 H^4}K^{3/4} + dH^3 K^{3/4} + d^3 H^4 K^{1/2} + \frac{H^6}{\gamma^2} K^{1/4} + \frac{dH^6}{\gamma^2}\right).
\end{align*} 
Next, we show the violation upper bound. If constraint violation is $0$, the statement is trivial. Otherwise, we decompose it as
\begin{align*}
    \Violation 
    &\leq \underbrace{\sum_{k=1}^K \left( V_{g,1}^{\pi^k}(s_1)-\widehat V_{g,1}^{k}(s_1)\right)}_{\text{(IV)}}+\underbrace{\sum_{k=1}^K \left(\widehat V_{g,1}^k(s_1) - b\right)}_{\text{(V)}}.
\end{align*}
By \Cref{lem:cost of opt},
\begin{align*}
    \text{(IV)} 
    &= \widetilde\calO\left(\sqrt{d^3H^4}K^{3/4} + d^3H^4 K^{1/2}\right).
\end{align*}
By \Cref{lem:violation optimism},
\begin{align*}
    \text{(V)}=\widetilde\calO\left(\frac{dH^5}{\gamma}K^{3/4} + \frac{H^4}{\gamma}K^{1/4}\right).
\end{align*}
Thus, we have the following violation upper bound.
\begin{align*}
    \Violation = \widetilde\calO\left(\frac{dH^5}{\gamma}K^{3/4} + \sqrt{d^3 H^4}K^{3/4} +  d^3 H^4 K^{1/2}\right).
\end{align*}
\qed

\section{Auxiliary Lemmas}

\begin{lemma}\label{lem:lip log}
    Let $c>0$. For $x\in [c,\infty)^d$, let $\log x = (\log x_1, \ldots, \log x_d)^\top$. For any $x,y \in [c,\infty)^d$, we have
    \begin{align*}
        \|\log x - \log y\|_\infty \leq \frac{1}{c}\|x-y\|_1
    \end{align*}
\end{lemma}
\begin{proof}
    Fix $x,y\in[c,\infty)^d$ and $i\in\{1,\dots,d\}$. Consider the scalar function $p_i:[0,1]\to\mathbb{R}$ defined as
    \[
        p_i(t):=\log\bigl(y_i+t(x_i-y_i)\bigr).
    \]
    Since $[c,\infty)^d$ is convex, we have $y_i+t(x_i-y_i)\ge \min\{x_i,y_i\}\geq c$ for all $t\in[0,1]$, so $p_i$ is continuously differentiable and
    \[
    p_i'(t) \;=\; \frac{x_i-y_i}{\,y_i+t(x_i-y_i)\,}.
    \]
    Hence, for all $t\in[0,1]$,
    \[
    \bigl|p_i'(t)\bigr| \;=\; \frac{|x_i-y_i|}{\,y_i+t(x_i-y_i)\,} \;\le\; \frac{|x_i-y_i|}{c}.
    \]
    By the fundamental theorem of calculus,
    \[
    \bigl|\log x_i - \log y_i\bigr|
    \;=\; \bigl|p_i(1)-p_i(0)\bigr|
    \;=\; \Bigl|\int_0^1 p_i'(t)\,dt\Bigr|
    \;\le\; \int_0^1 \bigl|p_i'(t)\bigr|\,dt
    \;\le\; \frac{|x_i-y_i|}{c}.
    \]
    Taking the maximum over $i$ and using $\|z\|_\infty \le \|z\|_1$ for all $z\in\mathbb{R}^d$ yields
    \[
    \|\log x-\log y\|_\infty
    = \max_{1\le i\le d} |\log x_i - \log y_i|
    \le \frac{1}{c}\,\max_{1\le i\le d} |x_i-y_i|
    \le \frac{1}{c}\,\sum_{i=1}^d |x_i-y_i|
    = \frac{1}{c}\,\|x-y\|_1.
    \]
\end{proof}

\begin{lemma}\label{lem:policy inequality}
    Let $\widehat\pi_h^k:\calS \to \Delta(\calA)$ be any policies. For $\theta \in [0,1]$, let $\widetilde\pi_h^k(\cdot|s)=(1-\theta)\widehat\pi_h^k(\cdot|s) + \theta\piunif(\cdot|s)$. For $\widehat Q_{f,h}^k, \widehat Q_{g,h}^k:\calS \times\calA \to [-2H,2H]$, $Y_k \in \bbR_+$, and $\alpha >0$, let $\widehat\pi^{k+1}(\cdot\mid s) \propto \widetilde\pi^k(\cdot\mid s)\exp(-\alpha (\widehat Q_{f,h}^k(s,\cdot) + Y_k \widehat Q_{g,h}^k(s,\cdot))$. For any $s\in\calS$, we have
    \begin{enumerate}
        \item $\|\widehat\pi_h^{k+1}(\cdot\mid s) - \widetilde\pi_h^k(\cdot\mid s)\|_1 \leq 2\alpha H(1+Y_k),$
        \item $\left|\langle \widehat\pi_h^{k+1}(\cdot|s) - \widehat\pi_h^{k}(\cdot|s), \widehat Q_{g,h}^k(s,\cdot) \rangle\right| \leq 4\alpha H^2(1+Y_k) + 4\theta H.$
        \item $\langle \widehat\pi_h^{k}(\cdot|s) - \widehat\pi_h^{k+1}(\cdot|s), \widehat Q_{f,h}^{k}(s,\cdot) \rangle - \frac{1}{\alpha} D(\widehat\pi_h^{k+1}(\cdot|s)|| \widetilde\pi_h^{k}(\cdot|s)) \leq 2\alpha H^2 + 4H\theta.$
    \end{enumerate}
\end{lemma}
\begin{proof}
    \textbf{(Proof of the first statement)} 
    We show the first statement. Given $s\in\calS$, we omit $(\cdot\mid s)$ in notation for simplicity. Note that $\widehat\pi_h^{k+1}$ can be viewed as an optimal solution for $\min_{\pi} \ \langle \pi, \widehat Q_{f,h}^k + Y_k \widehat Q_{g,h}^k\rangle + (1/\alpha) D(\pi || \widetilde \pi_h^k)$. Due to the pushback lemma (\Cref{lem:pushback}), by taking $z=\widetilde \pi^k$,
    \begin{align*}
        \langle \widehat\pi_h^{k+1}, \widehat Q_{f,h}^k + Y_k \widehat Q_{g,h}^k \rangle + \frac{1}{\alpha}D(\widehat\pi_h^{k+1}||\widetilde \pi_h^k) \leq \langle \widetilde \pi_h^k, \widehat Q_{f,h}^k + Y_k \widehat Q_{g,h}^k \rangle + \frac{1}{\alpha}D(\widetilde\pi_h^k || \widetilde\pi_h^k) - \frac{1}{\alpha}D(\widetilde \pi_h^k || \widehat\pi_h^{k+1}).
    \end{align*}
    Note that $D(\widetilde\pi_h^k || \widetilde\pi_h^k) =0$. This can be rewritten as
    \begin{align*}
        \frac{1}{\alpha}D(\widehat\pi_h^{k+1}||\widetilde \pi_h^k) + \frac{1}{\alpha}D(\widetilde \pi_h^k||\widehat\pi_h^{k+1}) \leq \langle \widetilde\pi_h^{k} - \widehat\pi_h^{k+1}, \widehat Q_{f,h}^k + Y_k \widehat Q_{g,h}^k \rangle.
    \end{align*}
    To lower bound the left-hand side, Pinsker's inequality implies that 
    \begin{align*}
        \frac{1}{2}\|\widehat\pi_h^{k+1} - \widetilde\pi_h^{k}\|_1^2\leq D(\widehat\pi_h^{k+1}||\widetilde\pi_h^k), \quad \frac{1}{2}\|\widehat\pi_h^{k+1} - \widetilde\pi_h^{k}\|_1^2\leq D(\widetilde\pi_h^k||\widehat\pi_h^{k+1}).
    \end{align*}
    To upper bound the right-hand side, H\"older's inequality implies that
    \begin{align*}
        \langle \widetilde\pi_h^{k} - \widehat\pi_h^{k+1}, \widehat Q_{f,h}^k + Y_k \widehat Q_{g,h}^k \rangle \leq \|\widetilde\pi_h^k - \widehat\pi_h^{k+1}\|_1\|\widehat Q_{f,h}^k + Y_k \widehat Q_{g,h}^k\|_\infty.
    \end{align*}
    As a result, we deduce that
    \begin{align*}
        \frac{1}{\alpha}\|\widehat\pi_h^{k+1} - \widetilde\pi_h^{k}\|_1^2 \leq \|\widetilde\pi_h^k - \widehat\pi_h^{k+1}\|_1\|\widehat Q_{f,h}^k + Y_k \widehat Q_{g,h}^k\|_\infty.
    \end{align*}
    If $\|\widehat\pi_h^{k+1} - \widetilde\pi_h^{k}\|_1=0$, then the statement is trivial. Otherwise, it follows that
    \begin{align*}
        \|\widehat\pi_h^{k+1} - \widetilde\pi_h^{k}\|_1 \leq \alpha \|\widehat Q_{f,h}^k + Y_k \widehat Q_{g,h}^k\|_\infty.
    \end{align*}
    Since $\|\widehat Q_{f,h}^k\|_\infty,\|\widehat Q_{g,h}^k\|_\infty \leq 2H$ and $Y_k \geq 0$, we have
    \begin{align*}
        \|\widehat\pi_h^{k+1} - \widetilde\pi_h^{k}\|_1 \leq 2\alpha H(1+Y_k).
    \end{align*}
    
    \textbf{(Proof of the second statement)} 
    Now, we show the second statement. By H\"older's inequality and the triangle inequality,
    \begin{align*}
        \left|\langle \widehat\pi_h^{k+1} - \widehat\pi_h^{k}, \widehat Q_{g,h}^k\rangle\right| \leq \|\widehat\pi_h^{k+1} - \widehat\pi_h^{k}\|_1 \|\widehat Q_{g,h}^k\|_\infty \leq \|\widehat\pi_h^{k+1} - \widetilde\pi_h^{k}\|_1 \|\widehat Q_{g,h}^k\|_\infty + \|\widetilde\pi_h^{k}-\widehat\pi_h^{k}\|_1 \|\widehat Q_{g,h}^k\|_\infty.
    \end{align*}
    By the first statement,
    \begin{align*}
        &\|\widehat\pi_h^{k+1} - \widetilde\pi_h^{k}\|_1 \leq \alpha \| \widehat Q_{f,h}^k(s,\cdot) + Y_k \widehat Q_{g,h}^k(s,\cdot)\|_\infty.
    \end{align*}
    Furthermore, we have
    \begin{align*}
        \|\widetilde\pi_h^{k}-\widehat\pi_h^{k}\|_1 = \|(1-\theta)\widehat\pi_h^{k} + \theta\piunif-\widehat\pi_h^{k}\|_1 = \theta\|-\widehat\pi_h^k + \piunif\|_1 \leq 2\theta.
    \end{align*}
    Finally, we have
    \begin{align*}
        \left|\langle \widehat\pi_h^{k+1} - \widehat\pi_h^{k}, \widehat Q_{g,h}^k\rangle\right| \leq \alpha \| \widehat Q_{f,h}^k(s,\cdot) + Y_k \widehat Q_{g,h}^k(s,\cdot)\|_\infty\|\widehat Q_{g,h}^k\|_\infty + 2\theta\|\widehat Q_{g,h}^k\|_\infty.
    \end{align*}
    Since $\|\widehat Q_{f,h}^k\|_\infty,\|\widehat Q_{g,h}^k\|_\infty \leq 2H$ and $Y_k\geq 0$,
    \begin{align*}
        \left|\langle \widehat\pi_h^{k+1} - \widehat\pi_h^{k}, \widehat Q_{g,h}^k\rangle\right| \leq 4\alpha H^2(1+Y_k) + 4\theta H.
    \end{align*}
    
    \textbf{(Proof of the third statement)} 
    Note that 
    \begin{align*}
        \langle \widehat\pi_h^{k} - \widehat\pi_h^{k+1}, \widehat Q_{f,h}^{k} \rangle - \frac{1}{\alpha} D(\widehat\pi_h^{k+1}|| \widetilde\pi_h^{k})
        &\leq \|\widehat \pi_h^{k+1} - \widehat \pi_h^k\|_1 \|\widehat Q_{f,h}^k\|_\infty - \frac{1}{2\alpha} \|\widehat\pi_h^{k+1}-\widetilde\pi_h^k\|_1^2\\
        &\leq 2H\|\widehat \pi_h^{k+1} - \widehat \pi_h^k\|_1  - \frac{1}{2\alpha} \|\widehat\pi_h^{k+1}-\widetilde\pi_h^k\|_1^2\\
        &\leq 2H\|\widehat \pi_h^{k+1} - \widetilde \pi_h^k\|_1  - \frac{1}{2\alpha}\|\widehat\pi_h^{k+1}-\widetilde\pi_h^k\|_1^2 + 2H\|\widetilde \pi_h^k -\widehat\pi_h^k\|_1\\
        &\leq 2\alpha H^2 + 2H\|\widehat\pi_h^k - \widetilde \pi_h^k \|_1
    \end{align*}
    where the first inequality is due to H\"older's inequality and Pinsker's inequality, the third inequality is due to the triangle inequality, and the last inequality follows from the fact that $-ax^2 + bx \leq b^2 / (4a)$ for $a,b,x > 0$, i.e., $a = 1/(2\alpha),\ b= 2H,\ x=\|\widehat \pi_h^{k+1}-\widetilde\pi_h^k\|_1$. The following is true.
    \begin{align*}
        \|\widehat \pi_h^k - \widetilde\pi_h^k\|_1 \leq \theta\|\widehat\pi_h^k - \piunif\|_1 \leq \theta(\|\widehat\pi_h^k\|_1 + \|\piunif\|_1) \leq 2\theta.
    \end{align*}
    Finally, we have
    \begin{align*}
        \langle \widehat\pi_h^{k} - \widehat\pi_h^{k+1}, \widehat Q_{f,h}^{k} \rangle - \frac{1}{\alpha} D(\widehat\pi_h^{k+1}|| \widetilde\pi_h^{k})  \leq 2\alpha H^2 + 4H\theta
    \end{align*}
    as desired.
\end{proof}

\begin{lemma}[Lemma 1 of \citet{wei2020online}, Lemma C.3 of \citet{qiu2020upper}
]\label{lem:pushback}
    Let $\Delta,\ \interior(\Delta)$ be the probability simplex and its interior, respectively, and let $f:\calC \rightarrow \bbR$ be a convex function. Fix $\alpha >0,\ y\in \interior(\Delta)$. Suppose $x^* \in \argmin_{x\in \Delta} f(x) + (1/\alpha) D(x||y)$ and $x^*\in \interior(\Delta)$, then, for any $z\in \Delta$,
    \[
        f(x^*) + \frac{1}{\alpha} D(x^* ||y) \leq f(z) + \frac{1}{\alpha} D(z||y) - \frac{1}{\alpha} D(z||x^*).
    \]
\end{lemma}

\begin{lemma}[Lemma 33 of \citet{kitamura2025provably}]\label{lem:softmax lipschitz}
Let $Q_1, Q_2:\calA \rightarrow \bbR$ be two functions. For $\alpha > 0$, let $\pi_1 \propto \exp(\alpha Q_1),\ \pi_2 \propto \exp(\alpha Q_2)$. Then we have
\[
    \|\pi_1 - \pi_2\|_1 \leq 8\alpha\|Q_1 - Q_2\|_\infty.
\]
\end{lemma}


\begin{lemma}[Lemma 31 of \citet{wei2020online}]\label{lem:KL mixing}
Let $\pi_1, \pi_2$ be two probability distributions in $\Delta(\mathcal{A})$.  
Let $\tilde{\pi}_2 = (1-\theta)\pi_2 + \theta/|\mathcal{A}|$ where $\theta \in (0,1)$.  
Then,
\begin{align*}
    D(\pi_1 \| \tilde{\pi}_2) - D(\pi_1 \| \pi_2) \leq \theta \log |\mathcal{A}|,\quad D(\pi_1 \| \tilde{\pi}_2) \leq \log(|\mathcal{A}|/\theta).
\end{align*}
\end{lemma}

\begin{lemma}[Lemma 24 of \citet{cassel2024warm}]\label{lem:cassel 24}
Let \( \mathcal{V} = \{ V(\cdot; \theta) : \|\theta\| \leq W \} \) denote a class of functions \( V : \mathcal{S} \to \mathbb{R} \). Suppose that any \( V \in \mathcal{V} \) is \( L \)-Lipschitz with respect to \( \theta \) and the supremum distance, i.e.,
\[
\| V(\cdot; \theta_1) - V(\cdot; \theta_2) \|_{\infty} \leq L \| \theta_1 - \theta_2 \|_1, \quad \|\theta_1\|_2, \|\theta_2\|_2 \leq W.
\]
Let \( \mathcal{N}_\epsilon \) be the \( \epsilon \)-covering number of \( \mathcal{V} \) with respect to the supremum distance. Then
\[
\log \mathcal{N}_\epsilon \leq d \log(1 + 2WL/\epsilon).
\]
\end{lemma}

\begin{lemma}[Lemma 18 of \citet{cassel2024warm}]\label{lem:cassel 18}
For any \( K \geq 1 \), \( \beta > 0 \), we have that
\[
\max_{y \geq 0} \left[ y \cdot \sigma(-\beta y + \log K) \right] \leq \frac{2 \log K}{\beta}.
\]
\end{lemma}

\begin{lemma}[Lemma D.4 of \citet{rosenberg2020near}]\label{lem:rosen d4}
Let \( \{X_t\}_{t \geq 1} \) be a sequence of random variables with expectation adapted to a filtration \( \mathcal{F}_t \). Suppose that \( 0 \leq X_t \leq C\) almost surely. Then with probability at least \( 1 - \delta \),
\[
\sum_{t=1}^T \mathbb{E}[X_t \mid \mathcal{F}_{t-1}] \leq 2 \sum_{t=1}^T X_t + 4 C\log \frac{2T}{\delta}.
\]
\end{lemma}

\begin{lemma}[Lemma 21 of \citet{cassel2024warm}]\label{lem:cassel 21}
Let \( \widehat{\theta}_{g,h}^k \) be as in line 14 of \Cref{alg:main-linear}. With probability at least \( 1 - \delta \), for all \( k \geq 1 \), \( h \in [H] \),
\[
\| \theta_{g,h} - \widehat{\theta}_{g,h}^k \|_{\Lambda_h^k} \leq 2 \sqrt{2d \log(2KH/\delta)}.
\]
\end{lemma}

\begin{lemma}[Lemma 22 of \citet{cassel2024warm}]\label{lem:cassel 22}
Let \( \widehat{\psi}_h^k : \mathbb{R}^{|\calS|} \to \mathbb{R}^d \) be the linear operator defined in line 16 of \Cref{alg:main-linear}. For all \( h \in [H] \), let \( \widehat \calV_h \subset \mathbb{R}^{|\calS|} \) be a set of mappings \( \widehat V : \calS \to \mathbb{R} \) such that \( \|\widehat V\|_{\infty} \leq \beta_Q \) and \( \beta_Q \geq 1 \). With probability at least \( 1 - \delta \), for all \( h \in [H] \), \( \widehat V \in \widehat\calV_{h+1} \), and \( k \geq 1 \),
\[
\| (\psi_h - \widehat{\psi}_h^k)\widehat V \|_{\Lambda_h^k} \leq 4 \beta_Q \sqrt{d \log(K+1) + 2 \log(H \calN_\epsilon/\delta)},
\]
where \( \epsilon \leq \beta_Q \sqrt{d}/(2K) \), \( \mathcal{N}_{\epsilon} = \sum_{h \in [H]} \mathcal{N}_{\epsilon}(\widehat\calV_{h}) \), and \( \mathcal{N}_{\epsilon}(\widehat\calV_{h}) \) is the \( \epsilon \)-covering number of \( \widehat\calV_h \) with respect to the $\ell_\infty$-norm.
\end{lemma}

\begin{lemma}[Lemma 17 of \citet{cassel2024warm}]\label{lem:cassel 17}
    For any \( \lambda > 0 \) and matrices \( \Lambda, \Lambda' \in \mathbb{R}^{d \times d} \) satisfying \( \Lambda, \Lambda' \succeq \lambda I \), we have that
    \[
    \|\Lambda^{1/2} - (\Lambda')^{1/2}\|_2 \leq \frac{1}{2\sqrt{\lambda}} \|\Lambda - \Lambda'\|_2.
    \]
\end{lemma}

\begin{lemma}[Lemma 3 of \citet{cassel2024warm}]\label{lem:cassel 3}
For any \( e \in [E] \) and \( v \in \mathbb{R}^d \), we have that
\[
\left( \phi(s_h, a_h) - \bar{\phi}_h^{k_e}(s_h, a_h) \right)^\top v 
\leq \left( 4 \beta_w^2 \| \phi(s_h, a_h) \|^2_{(\Lambda_h^k)^{-1}} + 2K^{-1} \right) 
\left| \phi(s_h, a_h)^\top v \right|.
\]
\end{lemma}

\begin{lemma}[Lemma 8 of \citet{cassel2024warm}]\label{lem:cassel 8 num of epoch}
    The number of epochs $|E|$ is bounded by $(3/2)dH \log(2K)$.
\end{lemma}

\begin{lemma}[Lemma 12 of \citet{abbasi2011improved}]\label{lem:abbasi 12}
Let \( A, B, C \) be positive semi-definite matrices such that \( A = B + C \). Then, we have that
\[
\sup_{x \neq 0} \frac{x^\top A x}{x^\top B x} \leq \frac{\det(A)}{\det(B)}.
\]
\end{lemma}
\begin{lemma}[Lemma D.2 of \citet{jin2020provably}]\label{lem: jin d2}
Let \( \{\phi_t\}_{t \geq 0} \) be a bounded sequence in \( \mathbb{R}^d \) satisfying \( \sup_{t \geq 0} \|\phi_t\| \leq 1 \). Let \( \Lambda_0 \in \mathbb{R}^{d \times d} \) be a positive definite matrix. For any \( t \geq 0 \), define
\[
\Lambda_t = \Lambda_0 + \sum_{j=1}^t \phi_j \phi_j^\top.
\]
Then, if the smallest eigenvalue of \( \Lambda_0 \) satisfies \( \lambda_{\min}(\Lambda_0) \geq 1 \), we have
\[
\log \left[ \frac{\det(\Lambda_t)}{\det(\Lambda_0)} \right]
\leq \sum_{j=1}^t \phi_j^\top \Lambda_{j-1}^{-1} \phi_j
\leq 2 \log \left[ \frac{\det(\Lambda_t)}{\det(\Lambda_0)} \right].
\]
\end{lemma}


\section{Numerical Experiment}
\label{sec:exp}
We evaluate \Cref{alg:main-linear} on a finite-horizon job-scheduling CMDP closely following the setup of \citet{ghosh2022provably} with modifications to incorporate adversarial losses. The number of episodes and the horizon are set to $K=100{,}000$ and $H=10$, respectively, and the state space is $\calS = \{0,1,\ldots,9\}$. At step $h$, $s_h$ denotes the number of remaining jobs in the stack, and each episode begins with the initial state $s_1=9$. The agent chooses $a_h\in\calA = \{0,1\}$, where $a_h=1$ corresponds to processing the current job and $a_h=0$ corresponds to idling. Specifically, if $a_h=1$, then $s_{h+1}=\max\{s_h-2,0\}$ with probability $0.8$, $s_{h+1}=\max\{s_h-1,0\}$ with probability $0.1$, and $s_{h+1}=s_h$ otherwise. If $a_h=0$, then $s_{h+1}=s_h$.

The loss and cost functions are defined as follows. To simulate an adversarial setting, in each episode $k$, the loss is chosen between two functions, $f^{(1)}$ and $f^{(2)}$, with probabilities $0.9-0.9(k-1)/(K-1)$ and $0.1+0.9(k-1)/(K-1)$, respectively. These functions are defined as
\begin{align*}
    f^{(1)}_h(a_h)=
    \begin{cases}
    1 & a_h=0,\\
    0.55 & a_h=1 \text{ and } h\in\{3,4,5,6\},\\
    0.2 & a_h=1 \text{ and } h\notin\{3,4,5,6\},
    \end{cases}
    \quad f^{(2)}_h(a_h)=
    \begin{cases}
    1 & a_h=0,\\
    0.6 & a_h=1 \text{ and } h\in\{4,5,6\},\\
    0.2 & a_h=1 \text{ and } h\notin\{4,5,6\}.
    \end{cases}
\end{align*}
The cost is defined as $g_h(s_h,a_h,s_{h+1}) = 1-(s_h-s_{h+1})/2$ for all $h$, and the cost budget is set to $b=5.6$. 

Figure~\ref{fig:main} summarizes the results of running \Cref{alg:main-linear} for $K=100{,}000$ episodes. To promote learning, we set the parameters as $\alpha = 0.1,\ \beta_b = K^{1/4},$ and $\beta_w = \beta_b\log K$, while keeping the other parameters same as in the setup of \Cref{alg:main-linear}. As shown in \Cref{fig:regret}, the regret grows sublinearly in $K$, despite the fact that the losses are not sampled from a fixed distribution. Furthermore, \Cref{fig:violation} shows that while the constraint violation grows rapidly in the early phase, it eventually converges to $0$ approximately after episode 45,000. These results support our main claim that both regret and constraint violation are bounded by sublinear terms.

\section{The Use of Large Language Models}
Portions of the text were polished using ChatGPT-5, which was employed for grammar checking and sentence refinement. 

\end{document}